\documentclass{article}

\usepackage[final,nonatbib]{neurips_2022}

\usepackage{amsmath}  
\usepackage{amssymb}
\usepackage{mathtools}
\usepackage{amsthm}
\usepackage{bigints}
\usepackage[utf8]{inputenc} % allow utf-8 input
\usepackage[T1]{fontenc}    % use 8-bit T1 fonts
\usepackage{hyperref}       % hyperlinks
\usepackage{url}            % simple URL typesetting
\usepackage{booktabs}       % professional-quality tables
\usepackage{amsfonts}       % blackboard math symbols
\usepackage{nicefrac}       % compact symbols for 1/2, etc.
\usepackage{microtype}      % microtypography
\usepackage{xcolor}         % colors
\usepackage{anyfontsize}
\usepackage{subcaption}
\usepackage{caption}
\usepackage{microtype}
\usepackage{graphicx}

\theoremstyle{plain}
\newtheorem{theorem}{Theorem}[section]
\newtheorem{proposition}[theorem]{Proposition}
\newtheorem{lemma}[theorem]{Lemma}
\newtheorem{corollary}[theorem]{Corollary}
\theoremstyle{definition}
\newtheorem{definition}[theorem]{Definition}

\newtheorem{example}[theorem]{Example}
\theoremstyle{remark}

\newtheorem{innercustomgeneric}{\customgenericname}
\providecommand{\customgenericname}{}
\newcommand{\newcustomtheorem}[2]{%
  \newenvironment{#1}[1]
  {%
   \renewcommand\customgenericname{#2}%
   \renewcommand\theinnercustomgeneric{##1}%
   \innercustomgeneric
  }
  {\endinnercustomgeneric}
}

\DeclareMathOperator*{\argmax}{arg\,max}

\newcustomtheorem{customthm}{Theorem}
\newcustomtheorem{customlemma}{Lemma}
\newcustomtheorem{customprop}{Proposition}

\newcommand{\trace}{\ensuremath{tr}}

\long\def\comment#1{}
\renewcommand\vec[1]{\ensuremath\boldsymbol{#1}}

\renewcommand{\trace}{\ensuremath{\operatorname{tr}}}
\newcommand{\Ecal}{\ensuremath{\mathcal{E}}}
\newcommand{\Tcal}{\ensuremath{\mathcal{T}}}
\newcommand{\Ocal}{\ensuremath{\mathcal{O}}}

\newcommand{\Pcal}{\ensuremath{\mathcal{P}}}
\newcommand{\Rspace}{\ensuremath{\mathbb{R}}}
\newcommand{\Pbb}{\ensuremath{\mathbb{P}}}

\newcommand{\sumfun}{\ensuremath{h_{0}}}

\newcommand{\overlapset}{\ensuremath{\mathcal{A}}}
\newcommand{\mydefn}{\ensuremath{: =}}
\newcommand{\gaussden}{\ensuremath{f}}
\newcommand{\extraset}{\ensuremath{\mathcal{B}}}
\newcommand{\extraind}{\ensuremath{\mathcal{I}}}
\newcommand{\sumind}{\ensuremath{\mathcal{S}}}

\newcommand{\smallO}{\ensuremath{o}}
\newcommand{\diff}{\ensuremath{\mathcal{T}}}

\newcommand{\brackets}[1]{\left[ #1 \right]}
\newcommand{\parenth}[1]{\left( #1 \right)}

\newcommand{\abss}[1]{\left| #1 \right |}

\newcommand{\vecnorm}[2]{\left\| #1\right\|_{#2}}
\newcommand{\enorm}[1]{\vecnorm{#1}{2}} % euclidean norm

\newcommand{\norm}[1]{\left\|#1\right\|}

\title{Beyond Black Box Densities: Parameter Learning for the Deviated Components}

\author{%
  Dat Do$^{\star}$\\
  Department of Statistics\\
  University of Michigan at Ann Arbor\\
  Ann Arbor, MI 48109 \\
  \texttt{dodat@umich.edu} \\
   \And
   Nhat Ho$^{\star}$ \\
   Department of Statistics and Data Sciences \\
   University of Texas at Austin \\
   Austin, TX 78712 \\
   \texttt{minhnhat@utexas.edu} \\
   \AND
   XuanLong Nguyen \\
   Department of Statistics\\
  University of Michigan at Ann Arbor\\
  Ann Arbor, MI 48109 \\
  \texttt{xuanlong@umich.edu} \\
}

\begin{document}

\maketitle

\begin{abstract}
As we collect additional samples from a data population for which a known density function estimate may have been previously obtained by a black box method, the increased complexity of the data set may result in the true density being deviated from the known estimate by a mixture distribution. To learn about this phenomenon, we consider the \emph{deviating mixture model} $(1-\lambda^{*})h_0 + \lambda^{*} (\sum_{i = 1}^{k} p_{i}^{*} f(x|\theta_{i}^{*}))$, where $h_0$ is a known density function, while the deviated proportion $\lambda^{*}$ and latent mixing measure $G_{*} = \sum_{i = 1}^{k} p_{i}^{*} \delta_{\theta_i^{*}}$ associated with the mixture distribution are unknown. Using a novel notion of distinguishability between the known density $h_{0}$ and the deviated mixture distribution, we establish rates of convergence for the maximum likelihood estimates of $\lambda^{*}$ and $G^{*}$ under Wasserstein metrics. Simulation studies are carried out to illustrate the theory. 
\end{abstract}
\let\thefootnote\relax\footnotetext{$^{\star}$: Dat Do and Nhat Ho contributed equally to this work.}

\vspace{-0.5 em}
\section{Introduction}
\vspace{-0.5 em}
Most data-driven learning processes typically consist of an iteration of steps that involve model training and fine-tuning, with more data in-take leading to further model re-training and refinement. As more samples become available and exhibit more complex patterns, the initial model may be obsolete, risks being discarded, or absorbed into a richer class of models that adapt better to the increased complexity. It takes considerable resources to train complex models on a rich data population. Moreover, many successful models in modern real-world applications have become so complex that make them hard to properly evaluate and interpret; aside from the predictive performance they may as well be considered as black boxes. Nonetheless, as data populations evolve and so must the learning models, several desiderata remain worthy: the ability to adapt to new complexity while retaining aspects of old "wise" model, and the ability to interpret the changes. 

In this paper we will investigate a class of complex models for density estimation that are receptive to \emph{adaptation}, \emph{reuse} and \emph{interpretablity}: we posit that there is an existing distribution $h_0$ which may have been obtained a priori by some means for the data population of interest, e.g., via kernel density estimation (KDE)~\cite{parzen1962estimation} or mixture models~\cite{mclachlan200001} or some modern black box methods, such as generative adversarial networks (GANs)~\cite{goodfellow2014generative, arjovsky2017wasserstein} or normalizing flows~\cite{dinh2016density}. Nonetheless, as more samples become available and/or as the data population changes, it is possible that the true density may deviate from $h_0$. While $h_0$ is potentially difficult to explicate, it is the deviation from the known $h_0$ that we wish to learn and interpret. We will use a mixture distribution to represent this deviation, leading to what we call a \emph{deviating mixture model} for the underlying data population:
\begin{align} 
p_{\lambda^{*} G_{*}}(x) := (1-\lambda^{*}) \sumfun(x) +\lambda^{*} F(x, G_{*}), \label{eq:true_model}
\end{align} 
for $x \in \mathbb{R}^d$, where $F(x, G_{*}) : = \sum_{i = 1}^{k_{*}} p_{i}^{*} f(x| \theta_{i}^{*}))$ represents a mixture distribution for the density components deviating from $h_0$. Such deviating components are from a known family of density function $f$. The unknown parameters for this model are the mixing proportion $\lambda^* \in [0,1]$, and the mixing measure $G_{*} = \sum_{i = 1}^{k_{*}} p_{i}^{*} \delta_{\theta_{i}^{*}} $, where $k_{*} \geq 1$ number of \emph{deviated} components. The choice of mixture distribution $F(x,G_*)$ allows us to express complex deviation from $h_0$, yet the overall model remains amenable to the interpretation of its parameters: $\lambda_*$ represents the amount of deviation from the existing candidate $h_0$, while the mixing measure $G_*$ represents heterogeneous patterns of the deviation. {  Because $h_0$ might be complex and trained with great computational resource to estimate the density of prior data population, it is reasonable to assume $h_0$ be known in the model~\eqref{eq:true_model}.}
The primary contribution of this paper is a rigorous investigation into the rather challenging questions of identifiability and parameter learning rates that arise from a standard maximum likelihood estimation procedure. 

\vspace{-0.3 em}
\textbf{Relations to existing works.} 
This modeling framework owes its roots to several significant bodies of work in both statistics and machine learning literature. In classical statistics, a dominant approach to address the increased complexity of data populations is via hypothesis testing: one can test an alternative (possibly composite) hypothesis represented by a class of distributions against the null hypothesis represented by $h_0$. Due to the constraint for obtaining simple and theoretically valid test statistics in order to accept or reject the null hypothesis, the testing approaches were mostly restricted to simple choices of distribution for the null and alternative hypotheses \cite{Chen-2003, Donoho-2004, Chen-2009, Tony_Cai_2011, Chen-2012}. 
More similar to~\eqref{eq:true_model} is the class of \emph{contaminated mixture models} for density estimation: in this framework, the data are assumed to be sampled from a mixture of $P_{0}$ and $Q$ where either $P_{0}$ or $Q$ can be an unknown distribution that needs to be estimated. While this approach offers more flexibility in terms of modeling, it does not always guarantee the identifiability of the mixing weight or mixture components $P_{0}, Q$~\cite{Scott_2015, patra2016estimation, Scott-2019}. Without identifiability, it is virtually impossible to interpret the model parameters for the data domains. To avoid the identifiability issue, several researchers added the semi-parametric or parametric structures on $P_{0}$ and $Q$, such as $P_{0}$ and $Q$ are mixture distributions~\cite{Bordes_2006, gadat2020parameter}. However, to the best of our knowledge, the convergence theory of these models remains poorly understood, except for some simple settings (see also \cite{butucea2014semiparametric,cai2007estimation}).
The main distinction between our modeling framework of deviating mixture models and the existing research on contaminated mixtures lies in our assumption that one of the mixture components, namely $h_0$ is known, allowing us to focus on the inference of the deviation from $h_0$, for which a considerable learning theory for the parameters of interest can be established and will be presented in this paper. Finally, estimating parameters of mixture distributions is an essential problem in mixture models. The convergence properties have been studied using identifiability notions and Wasserstein distances \cite{Nguyen-13,Ho-Nguyen-Ann-17,heinrich2018strong}. Our technical approach requires a generalization of the identifiability notion to take account of structural property of the existing component $h_0$, which helps to shed light on a considerably more complex convergence behavior of the deviated components.

\vspace{-0.3 em}
\textbf{Contributions.} The primary contribution of this paper is a rich theory of \emph{identifiability} and rates of convergence for \emph{parameters} and density estimation that arise in the \emph{deviating mixture model}~\eqref{eq:true_model}, under various settings of the existing component $h_0$, and that of the deviating components (via $f$ and $G_*$). {  Because the convergence of density estimation in Hellinger distance under the MLE procedure is well studied in \cite{geer2000empirical}, the bulk of our technical innovation lies in establishing a collection of \emph{inverse bounds} which relate the Hellinger distance of densities in model~\eqref{eq:true_model} in terms of that of their parameters. To do that, we introduce a novel notion of \emph{distinguishability} between $h_0$ and family of density $f$. The inverse bounds will be characterized under such distinguishability conditions (or the lack thereof). Our proof technique allows us to characterize different convergence rates of parameters in the deviating mixture model under distinguishable settings. It also gives rise to several new types of inverse bounds in partially distinguishable settings, where we may not have identifiability in our model. To the best of our knowledge, this is the first work in which such bounds are obtained in mixture modeling literature. Moreover, we will provide many examples to demonstrate the broad applicability of our theory, including cases where the existing component $h_0$ is obtained by a black box method (e.g., deep learning model) and a more traditional method (e.g., via KDE's or mixture models). By doing so, we are able to push the boundary of identifiability and learning theory of mixture models toward a larger class of modern machine learning models.}

\vspace{-0.3 em}
\textbf{Organization.} The remainder part of this paper is structured 
 as follows. In Section 2, we review the MLE method and the identifiability conditions, where the notion of distinguishability is presented. In Section 3, the main results of inverse bounds and convergence rates for parameters estimation of model~\eqref{eq:true_model} are shown. In Section 4, multiple simulation experiments are carried out to support the theory. Finally, Section 5 is used to discuss and conclude. Proofs of all the results in the main text are deferred to the Supplementary Material.

\vspace{-0.3 em}
\textbf{Notation.} We denote by $\Ecal_k(\Theta)=\{\sum_{i=1}^{k}p_i f(x|\theta_i) : \sum_{i=1}^{k}p_i = 1, p_i > 0, \theta_i \in \Theta \mbox{ } \forall 1\leq i \leq k \}$ the family of mixtures with exactly $k$ components and $\Ocal_K(\Theta)=\{\sum_{i=1}^{K}p_i f(x|\theta_i) : \sum_{i=1}^{K}p_i = 1, p_i \geq 0, \theta_i \in \Theta \mbox{ } \forall 1\leq i \leq K \}$ the family of mixtures with no more than $K$ components. $\Ecal_{k, c_0}(\Theta)= \{\sum_{i=1}^{k}p_i f(x|\theta_i) : \sum_{i=1}^{k}p_i = 1, p_i \geq c_0, \theta_i \in \Theta \mbox{ } \forall 1\leq i \leq k, k\leq K\}$ is the family of mixtures with exactly $K$ components and mixing proportions being bounded below by $c_0$, and $\Ocal_{K, c_0}(\Theta)= \{\sum_{i=1}^{k'}p_i f(x|\theta_i) : \sum_{i=1}^{k'}p_i = 1, p_i \geq c_0, \theta_i \in \Theta \mbox{ } \forall 1\leq i \leq k', k'\leq K\}$. $\enorm{\cdot}$ is the usual $l^2$ norm for vectors in $\mathbb{R}^d$ and matrices in $\mathbb{R}^{d\times d}$. We write $g(x) \gtrsim h(x)$ if $g(x) > c h(x)$ for all $x$, where $c$ is a constant does not depend on $x$ (similar for $g(x) \lesssim h(x)$). For any $\lambda \in \mathbb{R}$ and $B \subset \mathbb{R}$, denote by $1_{\{\lambda\in B \}}$ the function that takes value 1 if $\lambda\in B$, and 0 otherwise. For any two densities $p$ and $q$, we denote $h(p, q)$ by the Hellinger distance and $V(p, q)$ by the Total Variation distance between them.

\vspace{-0.5 em}
\section{Identifiability and distinguishability theory}
\vspace{-0.5 em}
\label{sec:distinguish}
The principal goal of the paper is to establish the efficiency of parameter learning for the deviating mixture model~\eqref{eq:true_model} via the standard maximum likelihood estimation (MLE) method. To achieve this goal, the parameters have to be identifiable to begin with. Thus, our theory builds on and extends a standard notion of identifiability of families of density $\{f(x|\theta): \theta\in \Theta\}$ that has been considered in previous work~\cite{Nguyen-13, heinrich2018strong}.
\begin{definition} \label{definition:strong_order_identifiability}
The family $\left\{f(x|\theta), \theta \in \Theta \right\}$ (or in short, $f$) is
\underline{identifiable in the order $r$}, for some $r \geq 1$, if $f(x|\theta)$ is differentiable up to the order $r$ in $\theta$
and the following holds:
\vspace{-0.2cm}
\begin{itemize}
\item[A1.] For any $k \geq 1$, given $k$ different elements
$\theta_{1}, \ldots, \theta_{k} \in \Theta$, if we have $\alpha_{\eta}^{(i)}$ such that for almost all $x$
\begin{eqnarray}
\sum \limits_{l=0}^{r}{\sum \limits_{|\eta|=l}{\sum \limits_{i=1}^{k}{\alpha_{\eta}^{(i)}\dfrac{\partial^{|\eta|}{f}}{\partial{\theta^{\eta}}}(x|\theta_{i})}}}=0 \nonumber
\end{eqnarray}
then $\alpha_{\eta}^{(i)}=0$ for all $1 \leq i \leq k$ and $|\eta| \leq r$.
\end{itemize}
\end{definition}
Many commonly used families $f$ for mixture modeling satisfy the first order identifiability condition, including location-scale Gaussian distributions, e.g., $f(x|\theta) = N(x|\mu,\sigma^2)$ where $\mu$ and $\sigma^2$ represent { the mean (location) and variance (scale) parameters}, and location-scale Student's t-distributions. 
In model~\eqref{eq:true_model}, however, due to the presence of the existing component $\sumfun$, the deviated mixture components need to be \emph{distinguishable} from $\sumfun$. This motivates a more general notion of identifiability, namely, \emph{distinguishability} that we now define. This condition specifies a property jointly for both the existing component $\sumfun$ and the family of density functions $f$ that make up the deviated components. 

\begin{definition} 
\label{definition:distinguishable} 
For any natural numbers $k, r\geq 1$, we say that the family of density functions $\{f(\cdot| \theta), \theta \in \Theta \}$ with complexity level $k$ (or in short, $(f, k)$) is \underline{distinguishable} \underline{up to the order $r$ from $h_{0}$} if the following holds:
\vspace{-0.2cm}
\begin{itemize}
\item[A2.] For any $k$ distinct components $\theta_{1}, \ldots, \theta_{k}$, if we have real coefficients $\alpha_{\eta}^{(i)}$ for $0 \leq i \leq k$ such that
\begin{eqnarray}
\alpha^{(0)} h_0(x) + \sum \limits_{l=0}^{r}{\sum \limits_{|\eta|=l}{\sum \limits_{i=1}^{k}{\alpha_{\eta}^{(i)}\dfrac{\partial^{|\eta|}{f}}{\partial{\theta^{\eta}}}(x|\theta_{i})}}}=0, \nonumber
\end{eqnarray}
for almost surely $x \in \mathcal{X}$, then $\alpha^{(0)} = \alpha_{\eta}^{(i)} = 0$ for $ 1 \leq i \leq k$ and $|\eta|\leq r$.
\end{itemize}
\end{definition}
{  We observe that the identifiable condition is a direct consequence of the corresponding distinguishable condition.} A simple but non-trivial example of the distinguishability condition can be derived directly from the definitions.
\begin{example} \label{example:distinguishable_condition} 
(a) When $\sumfun(x) = \sum_{i = 1}^{k_{0}} p_{i}^{0} f(x| \theta_{i}^{0})$ for some given weights $(p_{1}^{0}, \ldots, p_{k_{0}}^{0})$ and parameters $(\theta_{1}^{0}, \ldots, \theta_{k_{0}}^{0})$ where $k_{0} \geq 1$, then $(f, k)$ is distinguishable in the order $r$ from $h_{0}$ as long as $k <  k_{0}$ and the family of density $f$ is identifiable in the order $r$.\\
(b) Given the choice of $\sumfun$ in (a), $(f, k)$ is not distinguishable in the order $r$ from $h_{0}$ when $k \geq k_{0}$.
\end{example}
\vspace{-0.5 em}
More significantly, we can establish a broad class of $h_0$ and families $f$ for which distinguishability holds. This is exemplified by the following theorem, where $f$ represents a family of location or location-scale Gaussian kernels, and $h_0$ is subject to a relatively weak condition.

%Usually, if $h_0$ is not a mixture of densities in family $f$ itself, then we can have distinguishability between $h_0$ and mixtures of densities in $f$. Some examples with mixtures of Gaussians can be seen in the following result.
\vspace{-0.3 em}
\begin{theorem}\label{theorem:example-distinguishable-Gaussian} 
(a) Suppose that $-\log h_0(x) \gtrsim \enorm{x}^{\beta_1}$ or $-\log h_0(x) \lesssim \enorm{x}^{\beta_2}$ for all $\enorm{x} > x_0$, for some $x_0 > 0$, $\beta_1 > 2$, and $\beta_2 < 2$. Then, for $f$ being family of location-scale Gaussian and any $k > 0$, $(f, k)$ is distinguishable from $h_0$ up to the first order, where the derivatives in Assumption A2 are taken with respect to both location and scale parameters, and $(f, k)$ is also distinguishable from $h_0$ up to any order, where the derivatives in Assumption A2 are taken only with location parameters.

(b) Suppose that $h_0$ is the pdf of a pushforward measure of $N(0, I_d)$ by a piecewise linear function with a finite and positive number of breakpoints. Then, the same conclusions as in part (a) hold.
\end{theorem}
The proof of Theorem~\ref{theorem:example-distinguishable-Gaussian} is in Appendix~\ref{subsec:proof:theorem:example-distinguishable-Gaussian}, {  where the main proof technique is carefully examining the tail densities of $h_0$ and $f$ at infinity}. 
Note that in part (a), $h_0$ can be a pdf of any distribution possessing a lighter or heavier tail than Gaussian distributions, and in part (b), $h_0$ represents the pushforward of a Gaussian distribution by any piecewise linear function (recall that family of piecewise linear functions is dense in the Banach space of continuous functions with compact support). 
In the sequel we shall demonstrate several examples of interest that are applicable to Theorem~\ref{theorem:example-distinguishable-Gaussian} where $h_0$ may have been estimated by some popular "black box" methods.

% In the proposition above, the first result say that any distribution with lighter or heavier tail than Gaussian distribution will be distinguishable with mixture of Gaussian distributions. The second result contains the case we describe in Section 1, where the distribution $h_0$ is already estimated using black-box density estimation method like neural networks with ReLU activation functions \cite{arjovsky2017wasserstein, kingma2013auto}, we still have distinguishability between $h_0$ and mixtures of Gaussian distributions.

\vspace{-0.3 em}
\textbf{Kernel based representation.} Suppose that $h_0$ was obtained from a $m$- sample $Y_1, \dots, Y_m \in \mathbb{R}^d$ by a classical kernel density estimation (KDE) method \cite{parzen1962estimation} or a RKHS-based method \cite{Scholkopf02}, so that
\begin{align}\label{eq:h0-KDE}
    h_0(x) = \dfrac{1}{m} \sum_{j=1}^{m} k_{\sigma}(x, Y_j) \quad \forall x\in \mathbb{R}^d, 
\end{align}
where $k_{\sigma}$ is a kernel function with bandwidth $\sigma$. Popular choices of kernels include the Gaussian kernel $
    k_{\sigma}(x, x') = \left(\dfrac{1}{\sqrt{2\pi}\sigma}\right)^{d} \exp\left(-\dfrac{\enorm{x-x'}^2}{2\sigma^2}\right)$
and the multivariate Student's kernel
$k_{\sigma}(x, x') = \left(\frac{1}{\sqrt{\pi}\sigma}\right)^{d} \frac{\Gamma((\nu + d)/2)}{\Gamma(\nu/2)} \left(1 +\frac{\enorm{x-x'}^2}{\nu \sigma^2} \right)^{-\frac{\nu+d}{2}}$.
The corresponding distinguishability guarantee is as follows.
\begin{corollary}
    \label{prop:KDE-distinguish}
    Suppose $h_0$ is defined by Eq.~\eqref{eq:h0-KDE}, where $k_{\sigma}$ is Gaussian kernel and $m > K$, or $k_{\sigma}$ is the multivariate Student's kernel. Then, for $f$ being family of location-scale Gaussian, $(f, K)$ is distinguishable from $h_0$ up to the first order, where the derivatives in Assumption A2 are taken with respect to both location and scale parameters, and $(f, K)$ is also distinguishable from $h_0$ up to any order, where the derivatives in Assumption A2 are taken only with location parameters.
\end{corollary}
In application, it is common that the condition $m > K$ is satisfied. It is also matches with the scenario that we consider in the paper, where $h_0$ is already trained using a big data set, and there is a small number of deviated components. 

\vspace{-0.3 em}
\textbf{Neural networks.} Deep neural networks represent a powerful, albeit black box, approximation device for constructing rich classes of distribution for generative models \cite{goodfellow2014generative, arjovsky2017wasserstein}.  Accordingly, $h_0$ is the pdf function of a Gaussian distribution being push-forwarded by a map $T$, which is represented by a neural network (NN). Suppose that the NN representing $T$ has a positive and finite number of layers $L$, and so
\begin{equation}\label{eq:h0-NN}
    T(x) = a(W_{L} a(W_{L-1}( \dots a(W_{1} x + b_{1})) + b_{L-1}) + b_{L}),
\end{equation}
where $W_1, \dots, W_{L}\in \mathbb{R}^{d\times d}$ are the weights and $b_1, \dots, b_L\in \mathbb{R}^{d}$ are the biases. The activation function $a$ is chosen to be rectified linear unit (ReLU) function defined by $a(x) = \max\{x, 0\}$, and is applied elementwise to any vector in $\mathbb{R}^d$.
The corresponding guarantee on the distinguishability condition is as follows.
%This type of model is widely applicable in real-life \cite{goodfellow2014generative, arjovsky2017wasserstein}.

\begin{corollary}
\label{prop:neural_networks}
    Suppose that $h_0$ is the pdf a pushforward measure of $N(0, I_d)$ by a map $T$ defined by Eq.~\eqref{eq:h0-NN}. Then, for $f$ being family of location-scale Gaussian and any $k > 0$, $(f, k)$ is distinguishable from $h_0$ up to the first order, where the derivatives in Assumption A2 are taken with respect to both location and scale parameters, and $(f, k)$ is also distinguishable from $h_0$ up to any order, where the derivatives in Assumption A2 are taken only with location parameters.
\end{corollary}

% The proof of Cor.~\ref{prop:neural_networks} is in Appendix~\ref{subsec:proof:prop:neural_networks}.
\vspace{-0.5 em}
\section{Convergence rates of density estimation}
\vspace{-0.5 em}
\label{sec:conv-rate-density-estimation}
In this section, we first establish the rate of density estimation for the deviating mixture models in Section~\ref{sec:density_estimation_MLE}. We then describe a general procedure to obtain the convergence rate of parameter estimation based on that of density estimation via inverse bounds in Section~\ref{sec:point_wise_bound}. Finally, we provide comprehensive inverse bounds under several settings of the deviating mixture models in Section~\ref{subsec:distinguish_point_wise}.
\vspace{-0.3 em}
\subsection{MLE for deviating mixture model}
\vspace{-0.3 em}
\label{sec:density_estimation_MLE}
Given $n$ i.i.d. sample $X_1, X_2, \dots, X_n$ from $p_{\lambda^* G_*}$ as in model~\eqref{eq:true_model}, where $G_*$ has $k_*$ components, we want to estimate $\lambda^*$ and $G_*$ from the data. We refer to the problem as in \textit{exact-fitted setting} if $k_*$ is known, and we refer to it as in \textit{over-fitted setting} if $k_*$ is unknown but is known to be bounded by some number $K$. We denote the MLE for exact-fitted setting by
\begin{equation*}
    \widehat{\lambda}_n, \widehat{G}_n \in \argmax_{\lambda\in [0,1], G \in \mathcal{E}_{k_*}(\Theta)} \sum_{i=1}^{n} \log(p_{\lambda G}(X_i)),
\end{equation*}
and for the over-fitted setting, we replace $\mathcal{E}_{k_*}(\Theta)$ in the equation above by $\mathcal{O}_K(\Theta)$, where $K\geq k_*$. 

In order to state a rate of convergence
for the density estimators
$p_{\hat{G}_n}$ 
under the Hellinger distance $h$ \cite{geer2000empirical},  
we need a condition on the complexity
of the function class 
\begin{equation}
\overline{\mathcal{P}}^{1/2}_k(\Theta,\epsilon)  
 = \left\{\bar p_{\lambda G}^{1/2}  : G \in \mathcal{O}_k(\Theta), ~ h(\bar p_{\lambda G}, p_{\lambda^* G_*}) \leq \epsilon\right\},
\end{equation}
where for any $G \in \mathcal{O}_K(\Theta)$, 
we write $\bar{p}_{\lambda G} = (p_{\lambda G} + p_{\lambda^*G_*})/2$. 
The definition of $\overline{\mathcal{P}}_k(\Theta,\epsilon)$
originates from~\cite{geer2000empirical}. We measure the complexity of this class
through the bracketing entropy integral
\begin{equation}   
\label{eq:bracketing_integral} 
\mathcal{J}_B(\epsilon, \overline{\mathcal{P}}^{1/2}_k(\Theta,\epsilon), \nu)
= \int_{\epsilon^2/2^{13}}^{\epsilon} \sqrt{\log N_B(u, \overline{\mathcal{P}}_k^{1/2}(\Theta,\epsilon), \nu)}du \vee \epsilon,
 \end{equation}
where $N_B(\epsilon, X, \eta)$ denotes
the $\epsilon$-bracketing number
of a metric space $(X,\eta)$ and $\nu$ is the Lebesgue measure. 
We require the following assumption. 
\vspace{-0.2cm}
\begin{enumerate}   
\item[A3.]
Given a universal constant $J > 0$,
there exists  $N > 0$, possibly
depending on $\Theta$ and $k$, such that
for all $n \geq N$ and all $\epsilon > (\log n/n)^{1/2}$,
$$\mathcal{J}_{B}(\epsilon, \overline{\mathcal{P}}_k^{1/2}(\Theta,\epsilon), \nu) \leq J \sqrt n \epsilon^2.$$
\end{enumerate}
\begin{theorem}
\label{thm:density_estimation_rate}
Assume that Assumption A3 holds, and let $k \geq 1$. There exists a  constant $C > 0$ depending
only on $\Theta, k$ such that for all $n \geq 1$, 
$$
\sup_{G_{*} \in \mathcal{O}_k(\Theta), \lambda^*\in [0, 1]} \mathbb{E}_{\lambda^*, G_*} h(p_{\widehat{\lambda}_n\widehat{G}_n}, p_{\lambda^*G_*})  
\leq  C\sqrt{\log n/n}.$$
\end{theorem}
\vspace{-0.5 em}
Therefore, in order to get convergence rate for density functions based on MLE procedure, we only need to check assumption A3. This assumption holds true for a wide range class of parametric model \cite{geer2000empirical}. For our model, we give an example that it holds when $h_0$ has an exponential tail (satisfied for KDE's and Neural networks above) and $f$ is location-scale Gaussian distribution.
\begin{proposition}\label{prop:entropy-number-calculation}
    Suppose $f$ is location-scale Gaussian family and $\Theta = [-a, a]^d \times \Omega$, where $\Omega$ is a subset of $S_d^{++}$ whose eigenvalues are bounded in $[\underline{\lambda}, \overline{\lambda}]$, $a, \underline{\lambda}, \overline{\lambda}> 0$, and $h_0$ is bounded with tail $-\log h_0(x) \gtrsim \enorm{x}^{\beta}$ for some $\beta > 0$. Then, the family of densities $\{p_{\lambda G}: \lambda\in [0,1], G\in \Ocal_{k}(\Theta)\}$ satisfies assumption A3.
\end{proposition}

\vspace{-0.5 em}
\subsection{Parameter learning rates of deviated components}
\vspace{-0.3 em}
\label{sec:point_wise_bound}

The core of this paper lies in establishing a collection of \emph{inverse bounds}, provided that some distinguishability condition developed in Section~\ref{sec:distinguish} holds. The inverse bounds basically say that a small distance between $p_{\lambda G}$ and $p_{\lambda^{*} G_{*}}$ under the total variation distance entails that $(\lambda, G)$ and $(\lambda^{*}, G_{*})$ are similar under appropriate distances, where $(\lambda^{*}, G_{*})$ is fixed. To this end, we employ Wasserstein metrics \cite{Villani-2009} and their extensions.

\vspace{-0.3 em}
\textbf{Wasserstein  distances.} Wasserstein distances are natural and useful for assessing the convergence of latent mixing measures in mixture models \cite{Nguyen-13, Ho-Nguyen-Ann-16, guha2021posterior}. Given two measures $G = \sum_{i=1}^{k} p_i \delta_{\theta_i}$ and $G' = \sum_{j=1}^{k'} p_j' \delta_{\theta_j'}$ on a space $\Theta$ endowed with a metric $\rho$, the Wasserstein metric of order $r \geq 1$ is:
\begin{equation*}
    W_{r}(G, G') = [\inf_{q} \sum_{i, j} q_{ij} \rho^{r}(\theta_i, \theta_j')]^{1/r},
\end{equation*}
where the infimum is taken over all joint distribution on $[1, \dots, k]\times [1, \dots, k']$ such that $\sum_{i} q_{ij} = p_j', \sum_{j} q_{ij} = p_i$. Note that if $G_n$ is a sequence of discrete measures that converges to $G$ in a Wasserstein distance, then for every atom of $G$, there is a subset of atoms of $G_n$ converges to it. Therefore, the convergence in Wasserstein metrics implies convergence of parameters in mixture models. In this paper, space $\Theta$ is often chosen to be a compact subset of $\mathbb{R}^d$ and $\rho$ is the usual $l^2$ distance. In the case of location-scale Gaussian mixtures, space $\Theta$ is a compact subset of $\mathbb{R}^d \times S_d^{++}$, where $S_d^{++}$ is the set of positive definite and symmetric matrices in $\mathbb{R}^{d\times d}$, and for every $(\mu, \Sigma), (\mu', \Sigma')\in \Theta$, the distance $\rho$ is defined by 
$\rho((\mu, \Sigma), (\mu', \Sigma')) = \enorm{\mu - \mu'} + \enorm{\Sigma - \Sigma'}.$

\vspace{-0.3 em}
\textbf{From inverse bounds to parameter learning rates.} Suppose that some distinguishablity condition is satisfied, then we will establish an inverse bound providing a guarantee that a small distance between $p_{\lambda^*G_*}$ and $p_{\lambda G}$ entails a small distance between $\lambda$ and $\lambda^*$ and between $G$ and $G_*$. More concretely, define a divergence between two measures $\lambda G$ and  $\lambda^* G_*$ via
\[\overline{W}_{r}(\lambda G, \lambda^* G_*):= |\lambda - \lambda^*| + (\lambda + \lambda^*)W_r^r(G, G_*).\]
for all $r\geq 1$, and the inverse bounds will have the form that
$V(p_{\lambda G}, p_{\lambda^*, G_*}) \gtrsim \overline{W}_{r}(\lambda G, \lambda^* G_*)$,
for some $r$ that depends on the level of distinguishable level of the model. Since total variational distance is upper bounded by Hellinger distance, if Assumption A3. holds, then combining the aforementioned inverse bound with  Theorem~\ref{thm:density_estimation_rate} we immediately obtain
\begin{align*}
    \mathbb{E}_{\lambda^*, G_*} \overline{W}_r(\widehat{\lambda}_n \widehat{G}_n, \lambda^* G_*) \leq C \sqrt{\dfrac{\log n}{n}}.
\end{align*}
This further implies that the convergence rate of $\hat{\lambda}_n$ to $\lambda^*$ is of order $(\log(n)/n)^{1/2}$ and the convergence rate of $W_r(\hat{G}_n, G_*)$ to 0 is of order $(\log(n)/n)^{1/2r}$.

%%%%%%%%%%%%%%%%%%%%%%%%%%%%%%%%%%%%%%%%%%%%%%%%%%%%%%%%%%%%%%%%%%%%%%%%%%%%%%%
\vspace{-0.3 em}
\subsection{Inverse bounds in distinguishable setting}
\vspace{-0.3 em}
\label{subsec:distinguish_point_wise}
We shall establish inverse bounds provided a distinguishability condition for model~\eqref{eq:true_model} holds under either exact-fitted and over-fitted settings regarding the true number of components $k_*$.
\begin{theorem}
\label{theorem:distinguish_exact_specified_point_wise}
Assume that $k_{*}$ is known and $(f, k_{*})$ is distinguishable in the first order from $h_{0}$. Then, for any $G \in \Ecal_{k_{*}}( \Theta)$, there exist positive constant $C_{1}$ and $C_{2}$ depending only on $\lambda^{*}, G_{*}, h_{0}, \Theta$ such that the following holds:

(a) When $\lambda^{*} = 0$, then
	$V(p_{\lambda^{*} G_{*}}, p_{\lambda G}) \geq C_{1} \lambda$.

(b) When $\lambda^{*} \in (0, 1]$, then
$V(p_{\lambda^{*} G_{*}}, p_{\lambda G}) \geq C_{2} \overline{W}_{1}(\lambda G, \lambda^{*} G_{*}).$
\end{theorem}
\vspace{-0.15cm}
{  We now present a proof sketch for Theorem~\ref{theorem:distinguish_exact_specified_point_wise}. It is a combination of the Taylor expansion around the true parameters and the Fatou's lemma; the proof technique for the remaining results also shares similar spirit as that of Theorem~\ref{theorem:distinguish_exact_specified_point_wise}. Detailed proof of Theorem~\ref{theorem:distinguish_exact_specified_point_wise} is deferred to the Appendix.

\textbf{Proof sketch for part (b):} Suppose that the bound is not correct, so there exists a sequence $\lambda_n\in (0,1]$ and $G_n \in \mathcal{E}_{k_*}(\Theta)$ such that $V(p_{\lambda^{*} G_{*}}, p_{\lambda_n G_n})/ \overline{W}_{1}(\lambda^{*} G_{*}, \lambda_n G_n)\to 0.$ Because of the compactness of the parameter space, by extracting a subsequence if necessary, we can assume $\lambda_n\to \lambda', G_n\xrightarrow{W_1} G'$. If $(\lambda', G')\neq (\lambda^*, G_*)$, we have $\overline{W}_{1}(\lambda^{*} G_{*}, \lambda_n G_n) \to \overline{W}_{1}(\lambda^{*} G_{*}, \lambda' G')\neq 0$. It indicates that $V(p_{\lambda^{*} G_{*}}, p_{\lambda_n G_n}) \to 0$, which leads to $p_{\lambda^* G_*} = p_{\lambda' G'}$. It contradicts to the distinguishable condition when $(\lambda', G')\neq (\lambda^*, G_*)$).

Otherwise, we have $\lambda_n \to \lambda^*, G_n \to G_*$, and can present $G_n = \sum_{i=1}^{k_*} p_i^n \delta_{\theta_i^n}$ and $G_* = \sum_{i=1}^{k_*} p_i^* \delta_{\theta_i^*}$ such that $p_i^n\to p^*, \theta_i^n\to \theta_i^*$. Because of these limits and by Taylor expansion, we can arrange the difference $(p_{\lambda_n G_n}(x) - p_{\lambda^* G_*}(x))/\overline{W}_1(\lambda_n G_n, \lambda^* G_*)$ 
in terms of a linear combination of $h_0(x), f(x|\theta_i^*), \frac{\partial }{\partial \theta}f(x|\theta_i^*)$ such that at least one coefficient is different from 0. By Fatou's lemma, 
$0 = \dfrac{\liminf V(p_{\lambda_n G_n}, p_{\lambda^* G_*})}{\overline{W}_1(\lambda_n G_n, \lambda^* G_*)} dx\geq \bigintsss\left|\liminf \dfrac{p_{\lambda_n G_n}(x) - p_{\lambda^* G_*}(x))}{\overline{W}_1(\lambda_n G_n, \lambda^* G_*)}\right|dx$, which equals to the absolute integral of the linear combination above. Hence, there exists a non-trivial linear combination of $h_0(x), f(x|\theta_i^*), \frac{\partial }{\partial \theta}f(x|\theta_i^*)$ that equals 0, which contradict to the distinguishability condition. Therefore, we complete the proof.

}

\vspace{-0.3 em}
In application, the true number of components $k_*$ might not be known and we often fit the model~\eqref{eq:true_model} with $G \in \mathcal{O}_K(\Theta)$ for some { large} $K \geq k_*$. The next result shows that similar bounds can also be established in this case, where we require distinguishability of $f$ and $h_0$ in a higher order. 
\begin{theorem}
\label{theorem:distinguish_over_specified_point_wise}
Assume that $k_{*}$ is unknown and strictly upper bounded by a given $K$. Assume additionally that $(f, K)$ is distinguishable in second order from $h_{0}$. Then, for any $G \in \Ocal_{K}( \Theta)$, there exist positive constant $C_{1}$ and $C_{2}$ depending only on $\lambda^{*}, G_{*}, h_{0}, \Theta$ such that the following holds:

(a) When $\lambda^{*} = 0$, then
	$V(p_{\lambda^{*} G_{*}}, p_{\lambda G}) \geq C_{1} \lambda$.
	
(b) When $\lambda^{*} \in (0, 1]$, then
$V(p_{\lambda^{*} G_{*}}, p_{\lambda G}) \geq C_{2} \overline{W}_{2}(\lambda G, \lambda^{*} G_{*}).$
\end{theorem}
\vspace{-0.3 em}
Thanks to the distinguishability up to second order, no matter how large the number of over-fitted components $K$ is, we always get the $\overline{W}_{2}$ lower bound for the total variation distances. {  Proof of this theorem shares the same spirit with what of Theorem~\ref{theorem:distinguish_exact_specified_point_wise}. The difference here is when we overfit $G_*$ with some $\hat{G}$, there are some atoms of $\hat{G}$ that converges to the same atom of $G_*$, which requires us to do Taylor expansion up to second order and explain the higher order of Wasserstein distance here.} 
Next, we relax the assumption of Theorem~\ref{theorem:distinguish_over_specified_point_wise} by working on the setting where $f$ is not second order identifiable. This is an instance of the so-called \emph{weakly identifiable} setting ---
%where the order of Wasserstein distances in the lower bound might get worse if the number of over-fitted components increase. 
One popular example of weakly identifiable $f$ is location-scale Gaussian distribution, which admits the partial differential equation (PDE) structure $\dfrac{\partial^2{\gaussden}}{\partial{\mu^2}} (x| \mu, \Sigma) = 2 \dfrac{\partial{\gaussden}}{\partial{\Sigma}} (x| \mu, \Sigma),$
for all $x \in \Rspace^{d}$ where $\gaussden(x| \mu, \Sigma)$ stands for location-scale Gaussian density function with location $\mu$ and covariance $\Sigma$. In order to illustrate the result of our bound for that weak identifiability setting of $f$, we specifically consider $f$ to be location-scale Gaussian distribution. In this case, the parameter space $\Theta$ is a compact subset of $\mathbb{R}^d \times S_d^{++}$, where $S_d^{++}$ is the set of positive definite and symmetric matrices in $\mathbb{R}^{d\times d}$ equipped with the usual Frobenius norm. To put our result in context, we shall adopt a notion used in analyzing the convergence rate of parameter estimation in location-scale Gaussian mixtures in~\cite{Ho-Nguyen-Ann-16}. For any $k \geq 1$, let $\overline{r}(k)$ be the minimum value of $r$ such that the following system of polynomial equations:
\vspace{-0.1cm}
\begin{eqnarray}
\sum \limits_{j=1}^{k+1} \sum \limits_{n_{1}, n_{2}} \dfrac{c_{j}^{2}a_{j}^{n_{1}}b_{j}^{n_{2}}}{n_{1}!n_{2}!} = 0 \ \text{for each} \ \alpha=1,\ldots,r, \label{eqn:system_polynomial_Gaussian_first}
\end{eqnarray}
\vspace{-0.05cm}
does not have any nontrivial solution for the unknown variables $(a_{j},b_{j},c_{j})_{j=1}^{k+1}$, where the ranges of $n_{1}$ and $n_{2}$ in the second sum consist of all natural pairs satisfying the equation $n_{1}+2n_{2}=\alpha$. A solution to the above system is considered \textit{nontrivial} if all of variables $c_{j}$ are non-zeroes, while at least one of the $a_{j}$ is non-zero. Some examples of known values of $\overline{r}$ are $\overline{r}(1) = 4$ and $\overline{r}(2) = 6$, and $\overline{r}(k) \geq 7$ for all $k\geq 3$. Using this notion, we can characterize the convergence of parameters of model \eqref{eq:true_model} for the location-scale Gaussian family via the following theorem for inverse bounds.
\begin{theorem}\label{theorem:distinguish_over_specified_weakly_ident_point_wise}
    Assume that $G^* \in \mathcal{E}_{k^*, c_0}(\Theta)$, and $k_{*}$ is unknown and strictly upper bounded by a given $K$. In addition, $f$ is location-scale Gaussian distribution and $(f, K)$ with varied location, fixed variance parameters is distinguishable in any order from $h_{0}$. Then, for any $G \in \Ocal_{K, c_0}( \Theta)$, there exist positive constant $C_{1}$ and $C_{2}$ depending only on $\lambda^{*}, G_{*}, h_{0}, \Theta$ such that the following holds:
    
(a) When $\lambda^{*} = 0$, then
	$V(p_{\lambda^{*} G_{*}}, p_{\lambda G}) \geq C_{1} \lambda$.

(b) When $\lambda^{*} \in (0, 1]$, then
$V(p_{\lambda^{*} G_{*}}, p_{\lambda G}) \geq C_{2} \overline{W}_{\overline{r}(K-k_*)}(\lambda G, \lambda^{*} G_{*}).$
\end{theorem}
{  The proof technique of this result involves doing Taylor expansion of both location and scale parameter up to order $\overline{r}$, then utilize the heat equation $\dfrac{\partial{\gaussden}}{\partial{\Sigma}} (x| \mu, \Sigma) = \dfrac{1}{2}\dfrac{\partial^2{\gaussden}}{\partial{\mu^2}} (x| \mu, \Sigma) $ to compress this expression into linear combination of $h_0$ and derivatives of $f(x|\mu, \Sigma)$ with respect to $\mu$ only. This allows us to use the condition in this theorem to imply a contradiction, and gives rise to Eq.~\eqref{eqn:system_polynomial_Gaussian_first}. }
\vspace{-0.4 em}
\subsection{Inverse Bounds in Partially Distinguishable Setting}
\vspace{-0.3 em}
\label{subsec:pardistinguish_point_wise}
%%%%%%%%%%%%%%%%%%%%%%%%%%%%%%%%%%%%%%%%%%%%%%%%%%%%%%%%%%%%%%%%%%%%%%%%%%%%%%%
What happens if the distinguishability condition required by Def.~\ref{definition:distinguishable} no longer holds generally? Recall in Example \ref{example:distinguishable_condition} (b) that this situation is not uncommon, specifically when
\vspace{-0.5 em}
\begin{align}
	h_{0}(x) = f(x; G_{0}) = \sum_{i = 1}^{k_{0}} p_{i}^{0} f(x| \theta_{i}^{0}), \label{eq:par_distin_h0}
\end{align}
\vspace{-0.2 em}
where $G_{0} \mydefn \sum_{i = 1}^{k_{0}} p_{i}^{0} \delta_{\theta_{i}^{0}}$. In some specific cases of this setting, in fact, we fail to attain distinguishability, and the model may not even be identifiable in the classical sense, i.e. $p_{\lambda G} = p_{\lambda^* G_*}$ does not guarantee to have $\lambda G = \lambda^* G_*$. Since $h_0$ is the pdf of a mixture distribution --- a popular choice for modeling complex forms of probability densities given its amenability to interpretation compared to black box type models --- it is of interest to study the implication of parameter estimation for the deviated components in this setting, provided that the distinguisability condition may be at least partially achieved in some suitable sense. As we shall see, our theory demands a more refined analysis. To facilitate the presentation, denote $\overlapset \mydefn \bigr\{1 \leq i \leq k_{*}: \theta_{i}^{*} \in \{\theta_{1}^{0}, \ldots, \theta_{k_{0}}^{0} \} \bigr\}$. Also, set $\bar{k} \mydefn | \overlapset|$, which stands for the cardinal of the set $\overlapset$. Our results will be divided into three separate regimes of $\bar{k}$ and $\lambda^{*}$: (i) $\lambda^{*} = 0$, (ii) $\bar{k} < k_{0}$ and $\lambda^{*} \in (0, 1]$, and (iii) $\bar{k} = k_{0}$ and $\lambda^{*} \in (0, 1]$. We only choose to present results of the second regime (ii) in the main text because of { limited space} and because of its representativeness as it shows all the intriguing behaviours of the model in this partially distinguishable setting.
The first and third regime  are deferred to Appendix~\ref{sec:additional_results}.

\vspace{-0.3 em}
\subsubsection{Regime B: $\bar{k} < k_{0}$ and $\lambda^{*} \in (0, 1]$}
\vspace{-0.3 em}
\label{subsec:pardistinguish_point_wise_unequal}
First, we consider the exactly-specified setting of model~\eqref{eq:true_model}, namely, $k_{*}$ is known. When $\bar{k} < k_{0}$, we can check that we still have dishtinguishability of $h_0$ and linear combinations of $\{f(x|\theta_i^*)\}_{i=1}^{k_*}$ and its derivatives. Therefore, as long as $f$ is first order identifiable, one can invoke the proof of Theorem~\ref{theorem:distinguish_exact_specified_point_wise} to establish the same lower bound  $V(p_{\lambda G},p_{\lambda^{*} G_{*}})$ in terms of $\overline{W}_r(\lambda G,\lambda^{*}G_{*})$ for some $r\geq 1$. Thus, our focus in this subsection is the settings when $k_{*}$ is unknown.

\textbf{Over-fitted setting with strongly identifiable $f$.}
Moving to the over-fitted settings of model setup~\eqref{eq:true_model}, i.e., $k_{*}$ is unknown and strictly upper bounded by a given $K$, as long as $K \geq k_{0}$, $(f, K)$ is not distinguishable from $\sumfun$. Therefore, the results of Theorem~\ref{theorem:distinguish_exact_specified_point_wise} are not always applicable to the setting when $K \geq k_{0}$. Besides, in the over-fitted setting, the identifiability of model~\eqref{eq:true_model} no longer holds. Indeed, for any $\lambda > \lambda^*$, if we take
\begin{eqnarray}\label{eq:pathology-case}
\overline{G}_{*}(\lambda) = \left(1-\lambda^{*}/\lambda \right)G_0 + (\lambda^{*}/\lambda) G_*,
\end{eqnarray}
then $p_{\lambda^* G_*} = p_{\lambda \overline{G}_{*}(\lambda)}$. 
We present this pathological behavior in the following result.
\begin{theorem}
\label{theorem:pardistinguish_point_wise_overspec_strong_iden}
Assume that $\sumfun$ takes the form~\eqref{eq:par_distin_h0} and $\bar{k} < k_{0}$. Besides that, $K \geq k_{0}$ and $f$ is second order identifiable. Then, for any $G \in \Ocal_{K}( \Theta)$, there exist positive constants $C_{1}$ and $C_{2}$ depending only on $\lambda^{*}, G_{*}, \sumfun, \Theta$ such that the following hold:

(a) If $K \leq k_{*} + k_{0} - \bar{k} - 1$, then $V(p_{\lambda^{*}, G_{*}}, p_{\lambda, G}) \geq C_{1} \overline{W}_{2}(\lambda G, \lambda^{*} G_{*})$,

(b) If $K \geq k_{*} + k_{0} - \bar{k}$, then
\begin{align*}
 V(p_{\lambda^{*}, G_{*}}, p_{\lambda, G}) \geq C_{2} \left( 1_{\{\lambda \leq \lambda^{*}\}} \overline{W}_{2}(\lambda G, \lambda^{*} G_{*}) +  1_{\{\lambda > \lambda^{*}\}} W_{2}^2(G, \overline{G}_{*}(\lambda))\right).
\end{align*}

(c) As a special case, if $K = k_{*} + k_{0} - \bar{k}$, we have $$V(p_{\lambda^{*}, G_{*}}, p_{\lambda, G})  \geq C_{3}  1_{\{\lambda > \lambda^{*} + \delta\}}  W_{1}(G, \overline{G}_{*}(\lambda)),$$ for all $\delta > 0$, where $C_3$ depends on $\lambda^{*}, G_{*}, \sumfun, \Theta, \delta$.
\end{theorem}
As we can see, the magnitude of $\lambda$ compared to $\lambda^*$ will decide the solution of $(\lambda, G)$ to the identifiable equation $p_{\lambda G} = p_{\lambda^* G^*}$, therefore lead to different lower bounds such in part (b) of the theorem. In particular, if $\lambda \leq \lambda^*$, the solution is $(\lambda, G) = (\lambda^*, G_*)$, and for any $\lambda > \lambda^*$, the solution is $G = \overline{G}_{*}(\lambda)$ given in Eq. \eqref{eq:pathology-case}. Specifically, when $\lambda$ is strictly larger than $\lambda^*$ by some amount $\delta > 0$, then the latter case is well separated from the former, and we have an exact-fitted result when $K = k_0 + k_* - \overline{k}$. 

\vspace{-0.3 em}

\section{Experiments}\label{sec:experiments}
We now would like to  demonstrate the convergence rates in Section~\ref{sec:conv-rate-density-estimation} via two synthetic experiments: one for distinguishable setting and one for partially distinguishable setting. For the partially distinguishable one, the experiments are in Appendix~\ref{sec:add-exp}. 

\textbf{Distinguishable setting.} We conduct an experiment where the original data distribution comes from an uniform distribution on a curve (half circle) in $\mathbb{R}^2$ convoluted with Gaussian noises (red curve and blue points in Fig. \ref{fig:distinguishable-strong-ident-result}(a)), and train a Normalizing Flow  neural network \cite{nflows} (Masked Autoregressive architecture) with 5 layers to get a good density estimation $h_0$ for this dataset. Then we assume that there are new data coming in, and the original distribution $h_0$ is deviated by a mixture of distributions in the location Gaussian family $f(x|\theta)$. So the true generating density now is
\vspace{-0.3 em}
\begin{equation}\label{eq:simulation-true-model}
    p_{\lambda^* G_*}(x) = (1-\lambda^*) h_0(x) + \lambda^* \sum_{i=1}^{3} p_i^* f(x|\theta_i^*),
\end{equation}
where $\lambda^* = 0.5, G_* = \sum_
{i=1}^{3} p_i^* \delta_{\theta_i^*}$, where $p_1^* = 0.3, p_2^* = 0.3, p_3^* = 0.4, \theta_1^* = (-0.7, 1.5), \theta_2^* = (0.1, 2.0), \theta_3^* = (1.0, 1.5)$. Samples from the deviated component are green points in Fig. \ref{fig:distinguishable-strong-ident-result}(a). It can be seen from Proposition~\ref{theorem:example-distinguishable-Gaussian}(a) that $h_0$ is distinguishable with family $f$. For each $n$, we simulate $n$ data points from true model~\eqref{eq:simulation-true-model}, estimate $\hat{\lambda}_n, \hat{G}_n$ by the EM algorithm (it is possible because Normalizing Flows provides exact density computation), and measure its convergence to the true $\lambda^{*}, G_{*}$. 
We conduct 16 replications for each sample size. The average error estimations with a 75\% error bar can be seen in Fig.~\ref{fig:distinguishable-strong-ident-result}. The $W_1$ error in the exact-fitted case is of order $(\log(n)/n)^{1/2}$ and $W_2$ error in the over-fitted case is of order $(\log(n)/n)^{1/4}$. Meanwhile, thanks to the distinguishability, the estimation errors in both cases of $\lambda$ are all of the order $(\log(n)/n)^{1/2}$. These simulation results are matched with the theoretical results found in Theorem~\ref{theorem:distinguish_exact_specified_point_wise} and Theorem~\ref{theorem:distinguish_over_specified_point_wise}. From the result, we see that the deviating mixture model successfully learns the deviated components and reuses the pre-trained black box model $h_0$, which helps to reduce computational costs.
    
\begin{figure}[ht]
      \centering
      \subcaptionbox*{\scriptsize (a) Synthetic data set \par}{\includegraphics[width = 0.32\textwidth]{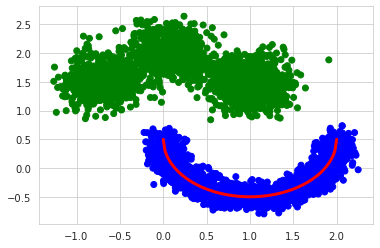}}
      \subcaptionbox*{\scriptsize (b) Convergence rates of $\hat{\lambda}_n$ \par}{\includegraphics[width = 0.32\textwidth]{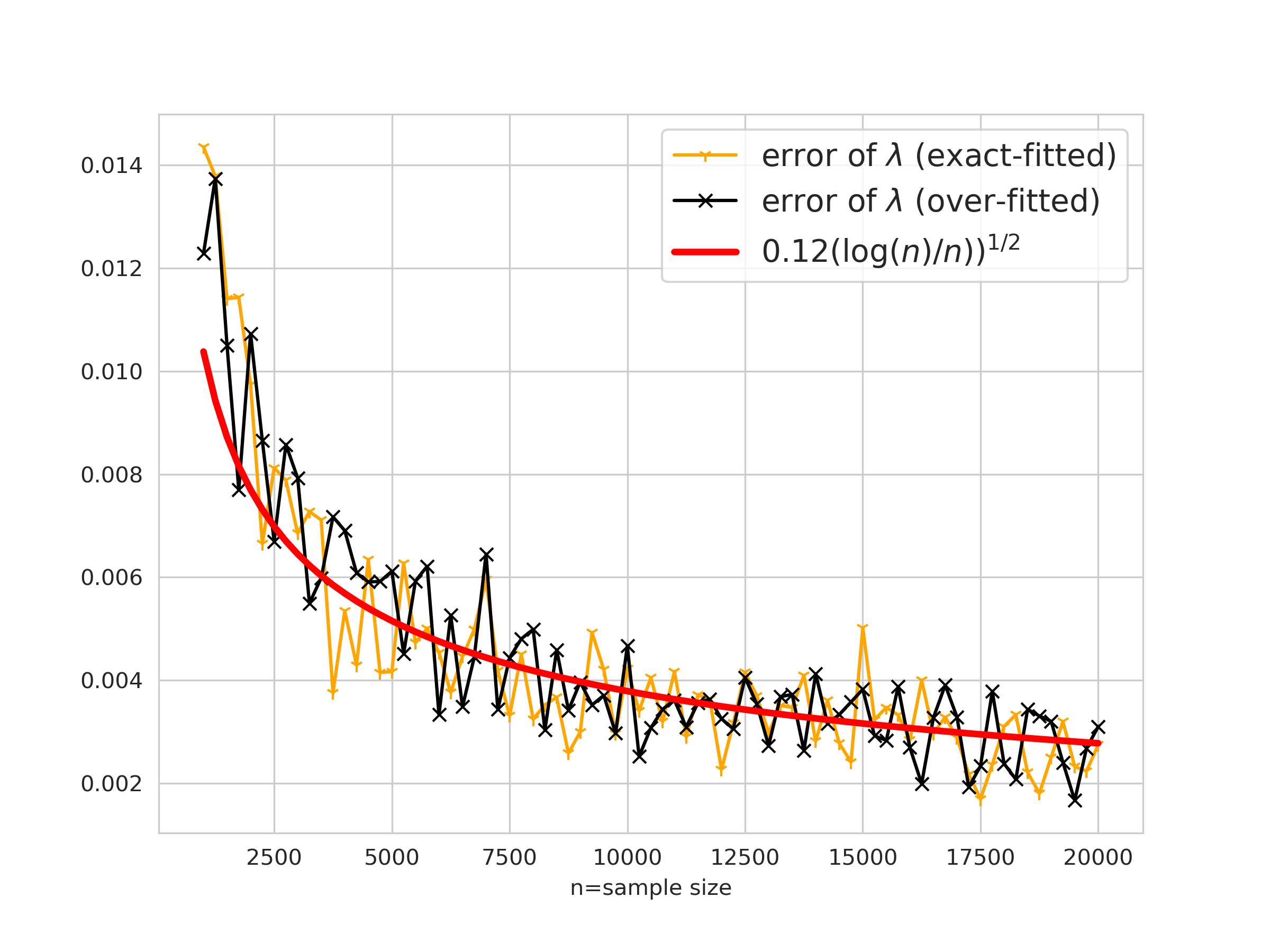}}
      \subcaptionbox*{\scriptsize (c) Convergence rates of $W(\hat{G}_n, G_{*})$ \par}{\includegraphics[width = 0.32\textwidth]{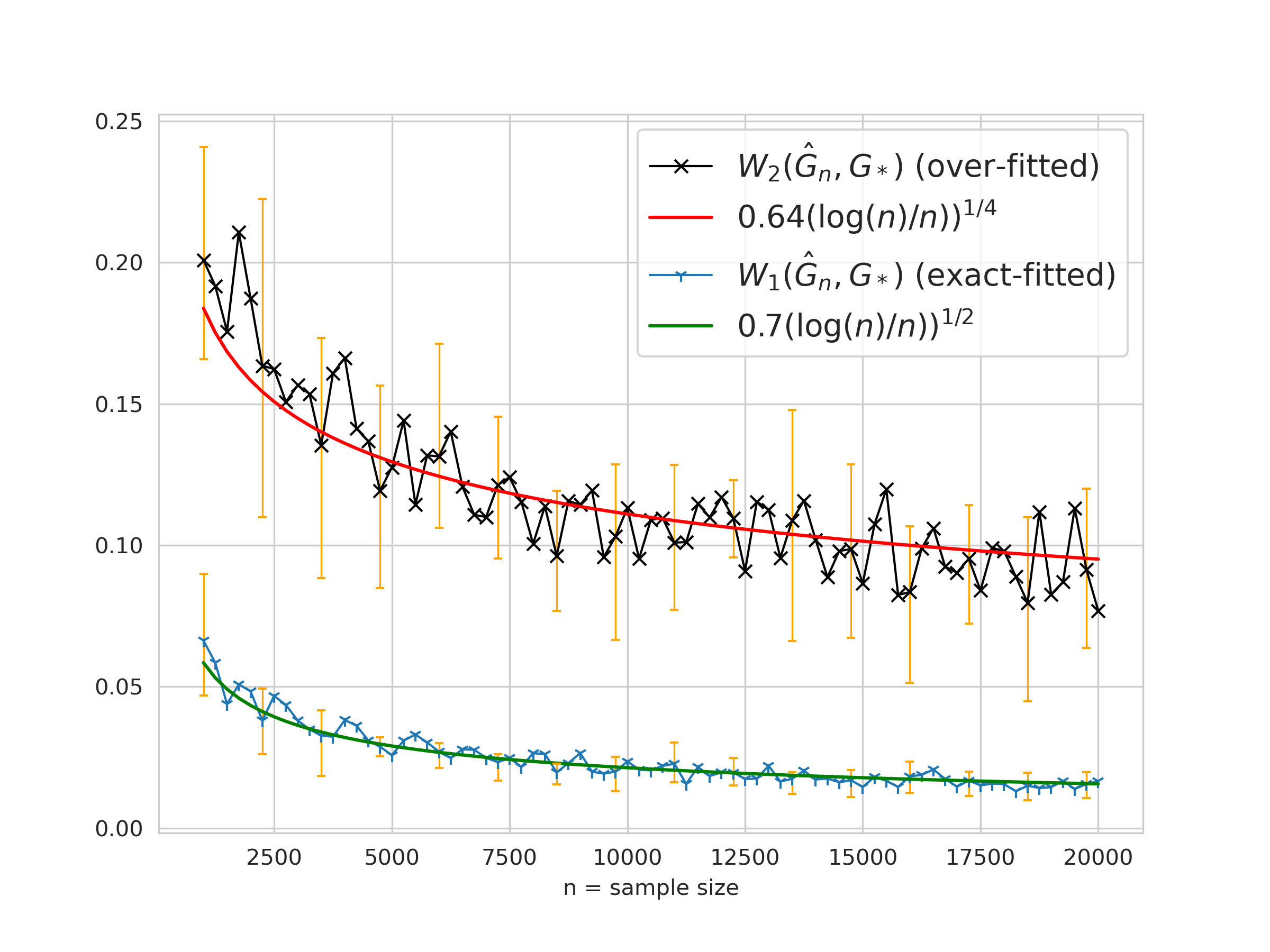}}
      \caption{Convergence rates for parameter estimation in the distinguishable case.}\label{fig:distinguishable-strong-ident-result}
    \end{figure}

\section{Discussion}
In this work, we have presented the deviating mixture model and studied its parameter learning rates under MLE procedure. With a novel notion of distinguishability between distributions, we are able to prove inverse bounds for our model under several distinguishability settings, which allow us to deduce the parameter learning rates from the convergence rate of density functions. The distinguishability condition is shown to be satisfied for multiple families of distributions including those that come from black box models. 

We now discuss practical implication of the theory. The deviating mixture model is designed to capture the deviated mixture components, and learning its parameters can reveal meaningful information about subpopulations in the data. When there is distinguishability in the model, our theory implies that we can learn the deviated proportion with the parametric rate and deviated components with a rate depending on the identifiablity of $f$. However, our theory does not support employing the deviating mixture model when the existing distribution $h_0$ itself is a mixture distributions in family $f$ and possesses parameters similar to deviated part, as the learning rate can be slow, and the deviated proportion estimator may not converge to the true value. {  Asymptotically, when $h_0$ is estimated using a very complex model (eg. a wide and deep neural network) and somehow approximates a mixture of $f$, and/or the signal from deviating components is low, then the provided learning rates in the paper, while still the same with respect to sample size $n$, may deteriorate from a large multiplicative constant that depends on $h_0, \lambda^*$, and $G_*$.}

We believe that this work is the first attempt in the effort of understanding a broader class of mixture models combining with black box models, and interpreting the learned model parameters. There is room for future work going forward. From a theoretical viewpoint, one may be interested in establishing minimax lower bounds for the learning behavior of the deviating mixture model, or show uniform inverse bounds for the model when $\lambda^*$ and $G_*$ are considered as signals that will change with samples. From a modeling viewpoint, it is worthwhile to explore mixtures of black box models and develop a  suitable notion of identifiability and inverse bounds so that the learning process is efficient.         

\section*{Acknowledgements}
Nhat Ho
acknowledges support from the NSF IFML 2019844 and the NSF AI Institute for Foundations of Machine Learning.
Long Nguyen is partially supported by NSF grant DMS-2015361.

\bibliography{Dat.bib}

\begin{thebibliography}{10}

\bibitem{arjovsky2017wasserstein}
Martin Arjovsky, Soumith Chintala, and L{\'e}on Bottou.
\newblock Wasserstein generative adversarial networks.
\newblock In {\em International conference on machine learning}, pages
  214--223. PMLR, 2017.

\bibitem{Bordes_2006}
L.~Bordes, S.~Mottelet, and P.~Vandekerkhove.
\newblock Semiparametric estimation of a two-component mixture model.
\newblock {\em Annals of Statistics}, 34, 2006.

\bibitem{butucea2014semiparametric}
Cristina Butucea and Pierre Vandekerkhove.
\newblock Semiparametric mixtures of symmetric distributions.
\newblock {\em Scandinavian Journal of Statistics}, 41(1):227--239, 2014.

\bibitem{Tony_Cai_2011}
T.~Cai, X.~J. Jeng, and J.~Jin.
\newblock Optimal detection of heterogeneous and heteroscedastic mixtures.
\newblock {\em Journal of the Royal Statistical Society: Series B (Statistical
  Methodology)}, 73, 2011.

\bibitem{cai2007estimation}
T~Tony Cai, Jiashun Jin, and Mark~G Low.
\newblock Estimation and confidence sets for sparse normal mixtures.
\newblock {\em The Annals of Statistics}, 35(6):2421--2449, 2007.

\bibitem{Chen-2003}
H.~Chen and J.~Chen.
\newblock Tests for homogeneity in normal mixtures in the presence of a
  structural parameter.
\newblock {\em Statistica Sinica}, 13:351--365, 2003.

\bibitem{Chen-2009}
J.~Chen and P.~Li.
\newblock Hypothesis test for normal mixture models: the em approach.
\newblock {\em Annals of Statistics}, 37:2523--2542, 2009.

\bibitem{Chen-2012}
J.~Chen, P.~Li, and Y.~Fu.
\newblock Inference on the order of a normal mixture.
\newblock {\em Journal of the American Statistical Association},
  107:1096--1105, 2012.

\bibitem{dinh2016density}
Laurent Dinh, Jascha Sohl{-}Dickstein, and Samy Bengio.
\newblock Density estimation using real {NVP}.
\newblock In {\em 5th International Conference on Learning Representations,
  {ICLR} 2017, Toulon, France, April 24-26, 2017, Conference Track
  Proceedings}, 2017.

\bibitem{Donoho-2004}
D.~Donoho and J.~Jin.
\newblock Higher criticism for detecting sparse heterogeneous mixtures.
\newblock {\em Annals of Statistics}, 32, 2004.

\bibitem{nflows}
Conor Durkan, Artur Bekasov, Iain Murray, and George Papamakarios.
\newblock {nflows}: normalizing flows in {PyTorch}, November 2020.

\bibitem{gadat2020parameter}
S{\'e}bastien Gadat, Jonas Kahn, Cl{\'e}ment Marteau, and Cathy
  Maugis-Rabusseau.
\newblock Parameter recovery in two-component contamination mixtures: The $
  l^{2}$ strategy.
\newblock In {\em Annales de l'Institut Henri Poincar{\'e}, Probabilit{\'e}s et
  Statistiques}, volume~56, pages 1391--1418. Institut Henri Poincar{\'e},
  2020.

\bibitem{goodfellow2014generative}
Ian Goodfellow, Jean Pouget-Abadie, Mehdi Mirza, Bing Xu, David Warde-Farley,
  Sherjil Ozair, Aaron Courville, and Yoshua Bengio.
\newblock Generative adversarial nets.
\newblock {\em Advances in neural information processing systems}, 27, 2014.

\bibitem{guha2021posterior}
Aritra Guha, Nhat Ho, and XuanLong Nguyen.
\newblock On posterior contraction of parameters and interpretability in
  bayesian mixture modeling.
\newblock {\em Bernoulli}, 27(4):2159--2188, 2021.

\bibitem{heinrich2018strong}
Philippe Heinrich and Jonas Kahn.
\newblock Strong identifiability and optimal minimax rates for finite mixture
  estimation.
\newblock {\em Annals of Statistics}, 46(6A):2844--2870, 2018.

\bibitem{Ho-Nguyen-Ann-16}
N.~Ho and X.~Nguyen.
\newblock Convergence rates of parameter estimation for some weakly
  identifiable finite mixtures.
\newblock {\em Annals of Statistics}, 44:2726--2755, 2016.

\bibitem{Ho-Nguyen-EJS-16}
N.~Ho and X.~Nguyen.
\newblock On strong identifiability and convergence rates of parameter
  estimation in finite mixtures.
\newblock {\em Electronic Journal of Statistics}, 10:271--307, 2016.

\bibitem{Ho-Nguyen-Ann-17}
Nhat Ho and Long Nguyen.
\newblock Singularity structures and impacts on parameter estimation in finite
  mixtures of distributions.
\newblock {\em SIAM Journal on Mathematics of Data Science}, 1(4):730--758,
  2019.

\bibitem{Scott-2019}
J.~Katz-Samuels, G.~Blanchard, and C.~Scott.
\newblock Decontamination of mutual contamination models.
\newblock {\em Journal of Machine Learning Research}, 20, 2019.

\bibitem{mclachlan200001}
Geoffrey~J. McLachlan and David Peel.
\newblock {\em Finite mixture models}, volume 299 of {\em Probability and
  Statistics -- Applied Probability and Statistics Section}.
\newblock Wiley, New York, 2000.

\bibitem{Nguyen-13}
X.~Nguyen.
\newblock Convergence of latent mixing measures in finite and infinite mixture
  models.
\newblock {\em Annals of Statistics}, 4(1):370--400, 2013.

\bibitem{parzen1962estimation}
Emanuel Parzen.
\newblock On estimation of a probability density function and mode.
\newblock {\em The annals of mathematical statistics}, 33(3):1065--1076, 1962.

\bibitem{patra2016estimation}
Rohit~Kumar Patra and Bodhisattva Sen.
\newblock Estimation of a two-component mixture model with applications to
  multiple testing.
\newblock {\em Journal of the Royal Statistical Society: Series B (Statistical
  Methodology)}, 78(4):869--893, 2016.

\bibitem{Scholkopf02}
B.~Sch\"olkopf and A.~Smola.
\newblock {\em Learning with Kernels}.
\newblock MIT Press, Cambridge, MA, 2002.

\bibitem{Scott_2015}
C.~Scott.
\newblock A rate of convergence for mixture proportion estimation, with
  application to learning from noisy labels.
\newblock In {\em AISTATS}, 2015.

\bibitem{geer2000empirical}
Sara van~de Geer.
\newblock {\em Empirical Processes in M-estimation}, volume~6.
\newblock Cambridge university press, 2000.

\bibitem{Villani-2009}
C.~Villani.
\newblock {\em Optimal Transport: Old and New. Grundlehren der Mathematischen
  Wissenschaften [Fundamental Principles of Mathemtical Sciences]}.
\newblock Springer, Berlin, 2009.

\bibitem{yakowitz1968identifiability}
Sidney~J Yakowitz and John~D Spragins.
\newblock On the identifiability of finite mixtures.
\newblock {\em The Annals of Mathematical Statistics}, 39(1):209--214, 1968.

\end{thebibliography}
\bibliographystyle{plain}

\section*{Checklist}

\begin{enumerate}
\item For all authors... 
\begin{enumerate}
\item Do the main claims made in the abstract and introduction accurately
reflect the paper's contributions and scope? \answerYes{See Section 1} 
\item Did you describe the limitations of your work? \answerYes{See Section 5} 
\item Did you discuss any potential negative societal impacts of your work?
\answerNo{} 
\item Have you read the ethics review guidelines and ensured that your paper
conforms to them? \answerYes{} 
\end{enumerate}
\item If you are including theoretical results... 
\begin{enumerate}
\item Did you state the full set of assumptions of all theoretical results?
\answerYes{} 
\item Did you include complete proofs of all theoretical results? \answerYes{} 
\end{enumerate}
\item If you ran experiments... 
\begin{enumerate}
\item Did you include the code, data, and instructions needed to reproduce
the main experimental results (either in the supplemental material
or as a URL)? \answerYes{} 
\item Did you specify all the training details (e.g., data splits, hyperparameters,
how they were chosen)? \answerYes{It can be seen in the source code.} 
\item Did you report error bars (e.g., with respect to the random seed after
running experiments multiple times)? \answerYes{See Figure 1, 2, 3.} 
\item Did you include the total amount of compute and the type of resources
used (e.g., type of GPUs, internal cluster, or cloud provider)? \answerNo{The experiments are run on CPU's only.} 
\end{enumerate}
\item If you are using existing assets (e.g., code, data, models) or curating/releasing
new assets... 
\begin{enumerate}
\item If your work uses existing assets, did you cite the creators? \answerNo{We do not use any existing assets.} 
\item Did you mention the license of the assets? \answerNo{} 
\item Did you include any new assets either in the supplemental material
or as a URL? \answerNo{} 
\item Did you discuss whether and how consent was obtained from people whose
data you're using/curating? \answerNo{} 
\item Did you discuss whether the data you are using/curating contains personally
identifiable information or offensive content? \answerNo{} 
\end{enumerate}
\item If you used crowdsourcing or conducted research with human subjects... 
\begin{enumerate}
\item Did you include the full text of instructions given to participants
and screenshots, if applicable? \answerNo{We do not use crowdsourcing/conducted research with human subjects} 
\item Did you describe any potential participant risks, with links to Institutional
Review Board (IRB) approvals, if applicable? \answerNo{} 
\item Did you include the estimated hourly wage paid to participants and
the total amount spent on participant compensation? \answerNo{} 
\end{enumerate}
\end{enumerate}

\newpage
\appendix

\appendix
\begin{center}
{\bf Supplement for "Beyond Black Box Densities: Parameter Learning for the Deviated Components"}
\end{center}
In the supplementary material, we collect proofs and results deferred from the main text. Section \ref{sec:additional_results} provides remaining results for the partially distinguishable case. Section \ref{sec:add-exp} presents the simulation studies that demonstrates the results in the partially distinguishable case. 
Section \ref{sec:section_2} contains proofs of results in Section 2, and Section \ref{sec:section_3} contains proofs of Section~\ref{sec:conv-rate-density-estimation}.
%%%%%%%%%%%%%%%%%%%%%%%%%%%%%%%%%%%%%%%%%%%%%%%%%%%%%%%%%%%%%%%%%%%%%%%%%%%%%%%%%%
\section{Additional results}
\label{sec:additional_results}
In this appendix,  we provide theory for the inverse bounds in partially distinguishable setting when $\bar{k} = k_{0}$ and $\lambda^{*} \in (0, 1]$.

\subsection{Regime A: $\lambda^{*} = 0$.}
\label{subsec:pardistinguish_point_wise_zerolambda}
\begin{theorem}
\label{theorem:pardistinguish_point_wise_zerolambda}
Assume that $\sumfun$ takes the form~\eqref{eq:par_distin_h0} and $\lambda^{*} = 0$. Then, there exist positive constants $C_{1}$ and $C_{2}$ depending only on $\sumfun, \Theta$ such that the following holds:

(a) (exact-fitted) If $f$ is first order identifiable, then for any $G\in \Ecal_{k_0}(\Theta)$
\begin{align*}
V(p_{\lambda^{*}, G_{*}}, p_{\lambda, G}) \geq C_{1} \lambda W_{1}( G, G_{0}).
\end{align*}

(b) (over-fitted) If $f$ is second order identifiable, then for any $G\in \Ocal_{K}(\Theta)$ that $K > k_0$
\begin{align*}
V(p_{\lambda^{*}, G_{*}}, p_{\lambda, G}) \geq C_{2} \lambda W_{2}^2(G, G_{0}).
\end{align*}

(c) (over-fitted and weakly identifiable) If $f$ is location-scale Gaussian distribution and we further assume that $G_* \in \Ecal_{k_*, c_0}(\Theta)$, then for any $G\in \Ocal_{K, c_0}(\Theta)$ that $K > k_0$, there exists $C_3$ depends on $h_0, \Theta_0, c_0$ such that
\begin{align*}
V(p_{\lambda^{*}, G_{*}}, p_{\lambda, G}) \geq C_{3} \lambda W_{\overline{r}(K-k_*)}^{\overline{r}(K-k_*)}(G, G_{0}).
\end{align*}
\end{theorem}
We may also "underfit" the deviated components by imposing $G\in \Ocal_{K}(\Theta)$ such that $K < k_0$. In that case, because of having less atoms, $p_{\lambda G}$ is $K-$distinguishable with $h_0$ and the result in Theorem~\ref{theorem:distinguish_exact_specified_point_wise} applies. 

\subsection{Regime B: $\bar{k} < k_{0}$ and $\lambda^{*} \in (0, 1]$}
We recall Theorem~\ref{theorem:pardistinguish_point_wise_overspec_strong_iden} in the main text, together with a similar theorem on weak identifiable family (Theorem~\ref{theorem:pardistinguish_point_wise_overspec_weak_iden}), and then provide some additional comments on the results.
\begin{theorem}
Assume that $\sumfun$ takes the form~\eqref{eq:par_distin_h0} and $\bar{k} < k_{0}$. Besides that, $K \geq k_{0}$ and $f$ is second order identifiable. Then, for any $G \in \Ocal_{K}( \Theta)$, there exist positive constants $C_{1}$ and $C_{2}$ depending only on $\lambda^{*}, G_{*}, \sumfun, \Theta$ such that the following hold:

(a) If $K \leq k_{*} + k_{0} - \bar{k} - 1$, then $V(p_{\lambda^{*}, G_{*}}, p_{\lambda, G}) \geq C_{1} \overline{W}_{2}(\lambda G, \lambda^{*} G_{*})$,

(b) If $K \geq k_{*} + k_{0} - \bar{k}$, then
\begin{align*}
 V(p_{\lambda^{*}, G_{*}}, p_{\lambda, G}) \geq C_{2} \left( 1_{\{\lambda \leq \lambda^{*}\}} \overline{W}_{2}(\lambda G, \lambda^{*} G_{*}) +  1_{\{\lambda > \lambda^{*}\}} W_{2}^2(G, \overline{G}_{*}(\lambda))\right).
\end{align*}

(c) As a special case, if $K = k_{*} + k_{0} - \bar{k}$, we have $$V(p_{\lambda^{*}, G_{*}}, p_{\lambda, G})  \geq C_{3}  1_{\{\lambda > \lambda^{*} + \delta\}}  W_{1}(G, \overline{G}_{*}(\lambda)),$$ for all $\delta > 0$, where $C_3$ depends on $\lambda^{*}, G_{*}, \sumfun, \Theta, \delta$.
\end{theorem}

We can view $p_{\lambda G}$ as a mixture distributions with latent mixing measures $\widehat{G} = (1 - \lambda) \sum_{i = 1}^{k_{0}} p_{i}^{0} \delta_{\theta_{i}^{0}} + \sum_{i = 1}^{K} p_{i} \delta_{\theta_{i}}$ having at most $K + k_{0}$ elements, while $p_{\lambda^* G_*}$ as a mixture with latent measure $\widehat{G}_{*} = \sum_{i = 1}^{\bar{k}}\biggr[(1 - \lambda^{*}) p_{i}^{0} + \lambda^{*} p_{i}^{*} \biggr] \delta_{\theta_{i}^{0}} +  \sum_{i = \bar{k} + 1}^{k_0}(1-\lambda^*) p_{i}^{0} \delta_{\theta_{i}^{0}} +  \sum_{i = \bar{k} + 1}^{k_*}\lambda^* p_{i}^{*} \delta_{\theta_{i}^{*}}$ having exactly $k_0 + k_{*} - \bar{k}$ elements. Because $k_0 + k_{*} - \bar{k} < K + k_0$, a direct application of Theorem 3.2 in ~\cite{Ho-Nguyen-EJS-16} gives us 
$V(p_{\lambda^{*}, G_{*}}, p_{\lambda, G}) \gtrsim W_{2}^2(\widehat{G}_{*}, \widehat{G})$.
But this bound is not as tight as what in Theorem \ref{theorem:pardistinguish_point_wise_overspec_strong_iden}(c), since $W_1 \gtrsim W_2^2$. The bounds established in the theorem are possible as we carefully explore the structure of $\widehat{G}_{*}$ and $\widehat{G}$.

\textbf{Over-fitted setting with weakly identifiable $f$.} Similar to Theorem~\ref{theorem:distinguish_over_specified_weakly_ident_point_wise}, when $f$ is the location-scale Gaussian, the weak identifiability can worsen the power of the bound in the over-fitted case. 

\begin{theorem}
\label{theorem:pardistinguish_point_wise_overspec_weak_iden} 
Assume that $\sumfun$ takes the form~\eqref{eq:par_distin_h0}. Besides that, $K \geq k_{0}$ and $f$ is location-scale Gaussian distribution. Then, for any $\lambda \in [0, 1]$ and $G \in \Ocal_{K, c_{0}}( \Theta)$ for some $c_{0} > 0$, there exist positive constants $C_{1}, C_{2}, C_3, C_4$ depending only on $\lambda^{*}, G_{*}, G_{0}, \Theta$ ($C_3$ and $C_4$ also depend on $\delta$) such that the following holds:

(a) When $K \leq k_{*} + k_{0} - \bar{k} - 1$, then $V(p_{\lambda^{*}, G_{*}}, p_{\lambda, G}) \geq  C_1 \overline{W}_{\overline{r}(K - k_{*} )}(\lambda G, \lambda^{*} G_{*})$.

(b) When $K \geq k_{*} + k_{0} - \bar{k} $, then
\begin{align*}
V(p_{\lambda^{*}, G_{*}}, p_{\lambda, G}) \geq C_{2} \biggr( 1_{\{\lambda \leq \lambda^{*}\}} \overline{W}_{\overline{r}(K - k_{*} )}(\lambda G, \lambda^{*} G_{*}) + 1_{\{\lambda > \lambda^{*}\}}  W_{\overline{r}(K - k_{*} )}^{\overline{r}(K - k_{*})}(G, \overline{G}_{*}(\lambda))\biggr).
\end{align*}
\vspace{-0.5 em}
(c) For $\delta > 0$, when $K = k_{*} + k_{0} - \bar{k}$, we have 
\begin{align*}
 V(p_{\lambda^{*}, G_{*}}, p_{\lambda, G}) \geq C_{3}  1_{\{\lambda > \lambda^{*} + \delta\}}  W_{1}(G, \overline{G}_{*}(\lambda)),
\end{align*}
and when $K > k_{*} + k_{0} - \bar{k}$, we have 
\begin{align*}
 V(p_{\lambda^{*}, G_{*}}, p_{\lambda, G}) \geq C_{4}  1_{\{\lambda > \lambda^{*} + \delta\}}  W_{\overline{r}(K - k_{0} - k_{*} + \bar{k} )}^{\overline{r}(K - k_{0}\ - k_{*} + \bar{k} )}(G, \overline{G}_{*}(\lambda)).
\end{align*}
\end{theorem}
\vspace{-0.3 em}
In this theorem, we once again observe the pathological behavior of the lower bound by Wasserstein distances caused by the unidentifiability of the model~\eqref{eq:true_model}. In part (c), when there is a well separation between two region of solutions of equation $p_{\lambda G} = p_{\lambda^*G_*}$, we can improve the order of Wasserstein distances for both exact-fitted case and over-fitted case.  In application, if $\hat{\lambda}_n$ and $\hat{G}_n$ are the MLE of model~\eqref{eq:true_model} estimated by $n$ i.i.d. data, then the convergence of $(\hat{\lambda}_n, \hat{G}_n)$ depends on the limit of $\hat{\lambda}_n$ (or its subsequence) comparing to $\lambda^*$. If $K = k_0 + k_* - \overline{k}$, any subsequence of $(\hat{\lambda}_n)$ having limit greater than $\lambda^*$ can achieve $W_1$ convergence rate of the distance between $\hat{G}_n$ and $\overline{G}_{*}(\hat{\lambda}_n)$. If $K > k_0 + k_* - \overline{k}$, any subsequence of $(\hat{\lambda}_n)$ having limit greater than $\lambda^*$ can achieve $W_{\overline{r}(K - k_{0} - k_{*} + \bar{k} )}^{\overline{r}(K - k_{0} - k_{*} + \bar{k} )}$ convergence rate of the distance between $\hat{G}_n$ and $\overline{G}_{*}(\hat{\lambda}_n)$, where $\overline{r}(K - k_{0} - k_{*} + \bar{k}) $ is smaller than $\overline{r}(K-k_*)$ in part (b).

\subsection{Regime C: $\bar{k} = k_{0}$ and $\lambda^{*} \in (0, 1]$.}
\label{subsec:pardistinguish_point_wise_equal}
When $\bar{k} = k_{0}$, $(f, k_{*})$ and $(f, K)$ are not distinguishable from $h_{0}$. It indicates that the results of Theorem~\ref{theorem:distinguish_exact_specified_point_wise} are no longer applicable to this setting. If $G^* = G_0$, the setting goes back to the case $\lambda^* = 0$ and it is already considered, so from this section, we assume that $G_* \neq G_0$. To streamline the argument, we further denote a few more notations. As $\bar{k} = k_{0}$, we can rewrite $G_{*}$ as follows:
\begin{align}
	G_{*} = \sum_{i = 1}^{k_{0}} p_{i}^{*} \delta_{\theta_{i}^{0}} + \sum_{i = k_{0} + 1}^{k_{*}} p_{i}^{*} \delta_{\theta_{i}^{*}}. \label{eq:refor_G*}
\end{align}
Because of the non-identifiability, the lower bound of $V(p_{\lambda G}, p_{\lambda^* G_*})$ must be inspected carefully based on the magnitude of mixing proportions of $p_{\lambda G}$ compared to what of $p_{\lambda^* G_*}$. To serve this purpose, we denote
\begin{align*}
\extraset & : = \{ \lambda \in [0, 1]: (\lambda^{*} - \lambda) p_{i}^{0} \leq \lambda^{*} p_{i}^{*} \ \forall  \ 1 \leq i \leq k_{0} \}, \\ \nonumber
\extraind( \lambda) & : = \{1 \leq i \leq k_{0}: (\lambda^{*} - \lambda) p_{i}^{0} > \lambda^{*} p_{i}^{*} \}.
\end{align*}
For any $\lambda \in [0, 1]$, we say that the set $\extraind( \lambda)$ is \textit{ratio-independent} if and only if $|\extraind( \lambda)| = 1$ or $p_{i}/ p_{i}^{*} = p_{j}/ p_{j}^{*}$ for all $i, j \in \extraind( \lambda)$ when $|\extraind( \lambda)| \geq 2$. Moreover, we define 
\begin{eqnarray}\label{eq:G-tilde-star-lambda}
    \widetilde{G}_{*} (\lambda) : = \frac{1}{\sumind( \mathcal{I} (\lambda))} \biggr( \sum_{i \in \extraind( \lambda)^{c}} \biggr[ p_{i}^{*} \lambda^{*} + \left( \lambda - \lambda^{*} \right) p_{i}^{0} \biggr] \delta_{\theta_{i}^{0}}\nonumber \\ + \lambda^{*} \sum_{i = k_{0} + 1}^{k_*} p_{i}^{*} \delta_{\theta_{i}^{*}} \biggr),
\end{eqnarray}
where $\sumind(\extraind (\lambda)) : = \sum_{i \in \extraind( \lambda)^{c}} \biggr[ p_{i}^{*} \lambda^{*} + \left( \lambda - \lambda^{*} \right)p_{i}^{0} \biggr] + \lambda^{*} \sum_{i = k_{0} + 1}^{k} p_{i}^{*}$. In the case $\mathcal{I}(\lambda)$ is ratio-independent, the identifiable equation $p_{\lambda G} = p_{\lambda^* G_*}$ attains a solution $G = \widetilde{G}_{*} (\lambda)$ as in equation~\eqref{eq:G-tilde-star-lambda}. Hence, in the following, we need to divide $\lambda$ into several regimes to specify the lower bound for $V(p_{\lambda G}, p_{\lambda^* G_*})$ based on appropriate distances of $(\lambda, G)$ and $(\lambda^*, G_*)$.
\paragraph{Setting with second order identifiable $f$:} We first consider the setting when $f$ is second order identifiable and the model setup~\eqref{eq:true_model} is  over-fitted. The following result demonstrates that under different settings of $\lambda$ and $\extraind( \lambda)$, the lower bound of $V(p_{\lambda G}, p_{\lambda^{*} G_{*}})$ in terms of its corresponding parameters $(\lambda, G)$ and $(\lambda^{*}, G_{*})$ can be very different.  
\begin{theorem}
\label{theorem:pardistinguish_point_wise_equal_exactspec}
Assume that $\sumfun$ takes the form~\eqref{eq:par_distin_h0} and $\bar{k} = k_{0}$. Besides that, $f$ is second order identifiable. Then, for any $\lambda \in [0, 1]$ and $G \in \Ocal_{K}( \Theta)$ that $K \geq k_*$, there exist positive constants $C_{1}$ and $C_{2}$ depending only on $\lambda^{*}, G_{*}, G_{0}, \Theta$ such that the following holds:
\begin{itemize}
\item[(a)] If $\extraind( \lambda)$ is not ratio-independent, then
\begin{align}
V(p_{\lambda^{*} G_{*}}, p_{\lambda G}) & \geq C_{1} \biggr[ 1_{\{\lambda \in \extraset^{c}\}} + 1_{\{\lambda\in \extraset\}} W_{2}^2(G, \overline{G}_{*} ( \lambda))\biggr]. \label{eq:pardistinguish_point_wise_equal_exactspec_ratdep}
\end{align}
\item[(b)] If $\extraind( \lambda)$ is ratio-independent, then
\begin{align}
V(p_{\lambda^{*}, G_{*}}, p_{\lambda, G}) & \geq C_{2} \biggr[ 1_{\{\lambda \in\extraset^{c}\}} \biggr(  \sum_{i \in \extraind( \lambda)} \biggr[ (\lambda^{*} - \lambda) p_{i}^{0}  \nonumber \\
& - \lambda^{*}p_{i}^{*} \biggr]+ \sumind(\mathcal{I} (\lambda)) W_{2}^2(G, \widetilde{G}_{*} ( \lambda)) \biggr) \nonumber \\ 
& + 1_{\{\lambda\in\extraset\}}W_{2}^2(G, \overline{G}_{*} ( \lambda))\biggr]. \label{eq:pardistinguish_point_wise_equal_exactspec_ratindep}
\end{align}
\end{itemize}
\end{theorem}
\noindent 
We can see that when $\lambda \in \extraset^{c}$ and $\extraind( \lambda)$ is not ratio-independent, the bound in equation~\eqref{eq:pardistinguish_point_wise_equal_exactspec_ratdep} shows that $V(p_{\lambda^{*} G_{*}}, p_{\lambda G}) \geq C_{1}$. It is due to the fact that $(\lambda^{*} - \lambda) p_{i}^{0} - \lambda^{*} p_{i}^{*}$ cannot be simultaneously arbitrarily small as $i \in \extraind( \lambda)$. On the other hand, these terms can become very small at the same time when $\extraind( \lambda)$ is ratio-independent. It implies that $V(p_{\lambda^{*} G_{*}}, p_{\lambda G})$ can become arbitrarily close to 0 under this setting of $\extraind( \lambda)$. It explains the difference of bounds between two settings of $\extraind( \lambda)$. 
\paragraph{Setting with weakly identifiable $f$:} Finally, we consider the settings of model setup~\eqref{eq:true_model} when $f$ is weakly identifiable. We specifically choose $f$ to be location-scale Gaussian distribution and study the lower bounds of $V(p_{\lambda G}, p_{\lambda^{*} G_{*}})$ in terms of their parameters.
% \begin{theorem}
% \label{theorem:pardistinguish_point_wise_equal_exactspec_weakind}
% Assume that $\sumfun$ takes the form~\eqref{eq:par_distin_h0} and $\bar{k} = k_{0}$. Besides that, $f$ is location-scale Gaussian distribution. Then, for any $\lambda \in [0, 1]$ and $G \in \Ecal_{k_{*}, c_{0}}( \Theta)$ for some $c_{0} > 0$, there exist positive constants $C_{1}$ and $C_{2}$ depending only on $\lambda^{*}, G_{*}, G_{0}, \Theta$ such that
% on $\lambda^{*}, G_{*}, G_{0}, \Theta$ such that
% \begin{itemize}
% \item[(a)] If $\extraind( \lambda)$ is not ratio-independent, then
% \begin{align}
% V(p_{\lambda^{*} G_{*}}, p_{\lambda G}) & \geq C_{1} \biggr[ 1_{\extraset^{c}} + 1_{\extraset} (\lambda + \lambda^{*}) W_{1}(G, \bar{G}_{*}( \lambda) ) \biggr]. \label{eq:pardistinguish_equal_exacspec_weakind_ratdep}
% \end{align}
% \item[(b)] If $\extraind( \lambda)$ is ratio-independent, then
% \begin{align}
% V(p_{\lambda^{*}, G_{*}}, p_{\lambda, G}) & \geq C_{2} \biggr[ 1_{\extraset^{c}} \biggr(  \sum_{i \in \extraind( \lambda)} \biggr[ (\lambda^{*} - \lambda) p_{i}^{0} - \lambda^{*}p_{i}^{*} \biggr] \nonumber \\
% & + \sumind(\mathcal{I} (\lambda)) W_{\overline{r}(k_{0} - |\extraind ( \lambda)^{c}| + 1)}^{\overline{r}(k_{0} - |\extraind ( \lambda)^{c}| + 1)}(G, \widetilde{G}_{*} ( \lambda)) \biggr) \nonumber \\ 
% & + 1_{\extraset} (\lambda + \lambda^{*}) W_{1}(G, \bar{G}_{*}( \lambda)) \biggr]. \label{eq:pardistinguish_equal_exacspec_weakind_ratindep}
% \end{align}
% \end{itemize}
% \end{theorem}

\begin{theorem}
\label{theorem:pardistinguish_point_wise_equal_overspec_weakind}
Assume that $\sumfun$ takes the form~\eqref{eq:par_distin_h0} and $\bar{k} = k_{0}$. Besides that, $f$ is location-scale Gaussian distribution. Then, for $\tilde{k}:=\max\{k_* - k_0, 1\}$, and for any $\lambda \in [0, 1]$ and $G \in \Ocal_{K, c_{0}}( \Theta)$ for some $K \geq k_*$ and $c_{0} > 0$, there exist positive constants $C_{1}$ and $C_{2}$ depending only on $\lambda^{*}, G_{*}, G_{0}, \Theta$ such that
on $\lambda^{*}, G_{*}, G_{0}, \Theta$ such that
\begin{itemize}
\item[(a)] If $\extraind( \lambda)$ is not ratio-independent, then
\begin{align}
V(p_{\lambda^{*} G_{*}}, p_{\lambda G}) & \geq C_{1} \biggr[ 1_{\{\lambda \in\extraset^{c}\}} \nonumber \\
& + 1_{\{\lambda \in\extraset\}}   W_{\overline{r}(K - \tilde{k} )}^{\overline{r}(K - \tilde{k})}(G, \bar{G}_{*}( \lambda) ) \biggr]. \label{eq:pardistinguish_equal_overspec_weakind_ratdep}
\end{align}
\item[(b)] If $\extraind( \lambda)$ is ratio-independent, then
\begin{align}
V(p_{\lambda^{*}, G_{*}}, p_{\lambda, G}) & \geq C_{2} \biggr[ 1_{\{\lambda \in\extraset^{c}\}} \biggr(  \sum_{i \in \extraind( \lambda)} \biggr[ (\lambda^{*} - \lambda) p_{i}^{0}  \nonumber \\
- \lambda^{*}p_{i}^{*} \biggr]& + \sumind(\mathcal{I} (\lambda)) W_{\overline{r}(K - \tilde{k})}^{\overline{r}(K - \tilde{k})}(G, \widetilde{G}_{*} ( \lambda)) \biggr) \nonumber \\ 
& + 1_{\{\lambda \in\extraset\}}   W_{\overline{r}(K - \tilde{k} )}^{\overline{r}(K - \tilde{k})}(G, \bar{G}_{*}( \lambda)) \biggr]. \label{eq:pardistinguish_equal_overspec_weakind_ratindep}
\end{align}
\end{itemize}
\end{theorem}

\section{Additional Experiment}\label{sec:add-exp}
We provide a simulation experiment with partially distinguishable setting in this section to demonstrate the theoretical results in Section~\ref{subsec:pardistinguish_point_wise}.

\textbf{Partially distinguishable setting.} Consider the partial distinguishable case as in Theorem~\ref{theorem:pardistinguish_point_wise_overspec_weak_iden} with weakly identifiable $f$, we will conduct an experiment to distinguish two regimes in part (b) and (c) of the theorem, which are $\lambda > \lambda^*$ and $\lambda \leq \lambda^*$. We simulate $n$ data from the true data generating model~\eqref{eq:true_model}, where $p_1^0 = 0.4, p_2^0 = 0.6$, $p_1^* = 1,\lambda^* = 0.3$,
$\mu_1^0 = \mu_1^* = (-2, 3), \Sigma_1^0 = \Sigma_1^* = \begin{pmatrix} 3 & -1 \\ -1 & 2 \end{pmatrix}$, $\mu_2^0 = (1, -4), \Sigma_2^0 = \begin{pmatrix} 1 & 0 \\ 0 & 4 \end{pmatrix}$. In this case, $k_* = 1, k_0 = 2, \bar{k} = 1$, $k_* + k_0 - \overline{k} = 2$ and we will fit the data with model $p_{\lambda G}$, where $G$ has 3 atoms. The MLE $(\hat{\lambda}_n, \hat{G}_n)$ is found by the EM algorithm. In the regime $\hat{\lambda}_n < \lambda^*$, we see that $\hat{\lambda}_n\to \lambda^*$ in the parametric rate and the convergence of $\hat{G}_n$ to $G_*$ is of order $(\log(n)/n)^{2\overline{r}(K-k_*)} = (\log(n)/n)^{12}$ (Fig.~\ref{fig:partial-distinguishable-weak-ident}). When $\hat{\lambda}_n > \lambda^*$, because of the indistinguishability of the model, we do not expect $\hat{\lambda}_n\to \lambda^*$ but the Wasserstein distance between $\hat{G}_n$ and $\overline{G}_*(\hat{\lambda}_n)$ converges to 0 with the rate $(\log(n)/n)^{2\overline{r}(2)}=(\log(n)/n)^{1/8}$. The simulation study matches with this result, where $\hat{\lambda}_n$ converges to some number greater than $\lambda^*$, and the rate that $W_4(G, \overline{G}_*(\hat{\lambda}_n))$ converges to 0 is of order $(\log(n)/n)^{1/8}$ (Fig. \ref{fig:partial-distinguishable-weak-ident-large-lambda}).

    \begin{figure}[ht]
      \centering
      \subcaptionbox*{\scriptsize (a) Convergence rates of $W_6(\hat{G}_n, G_{*})$ \par}{\includegraphics[width = 0.45\textwidth]{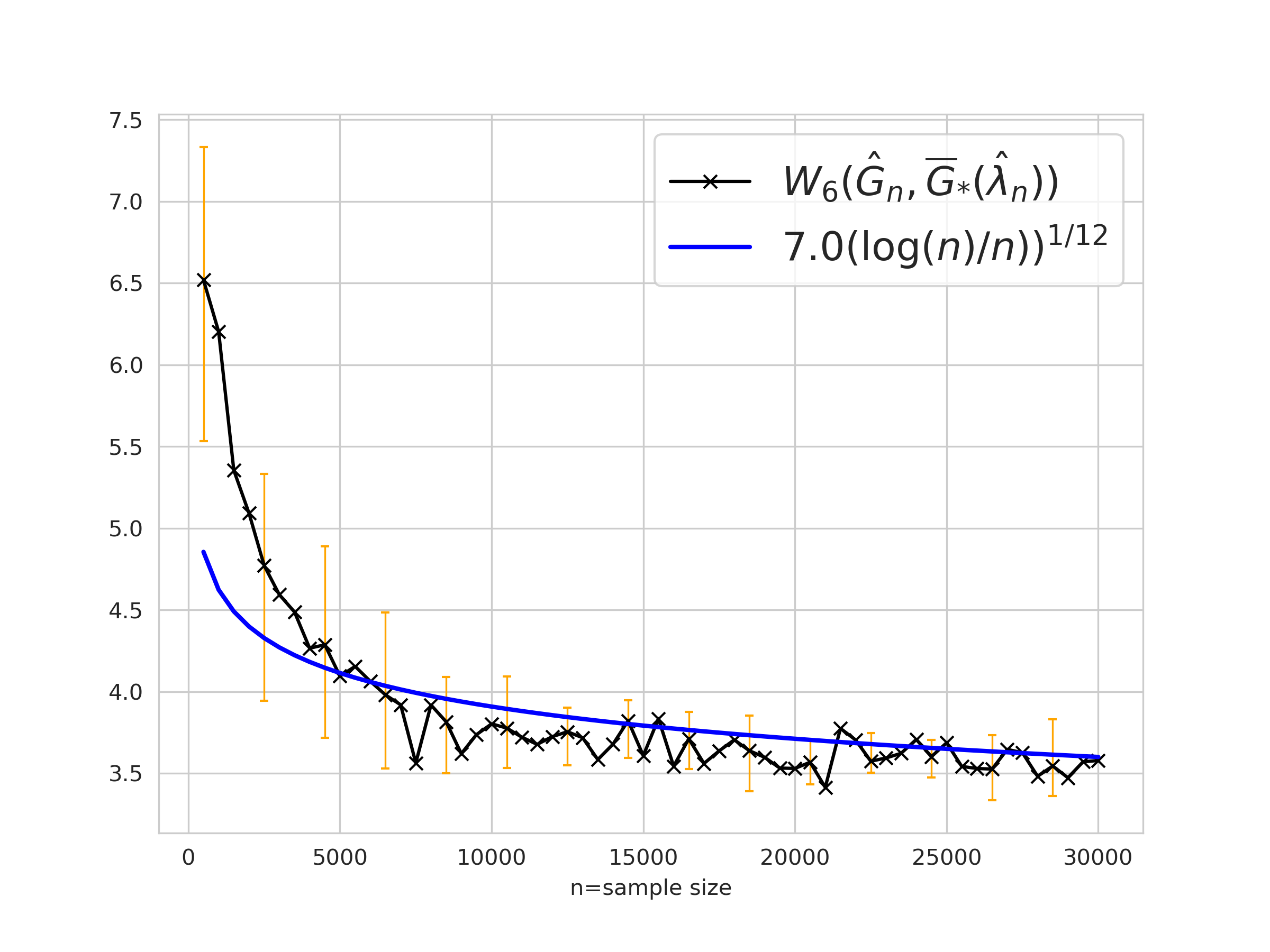}}
      \subcaptionbox*{\scriptsize(b) Convergence rates of $|\hat{\lambda}_n - \lambda^*|$ \par}{\includegraphics[width = 0.45\textwidth]{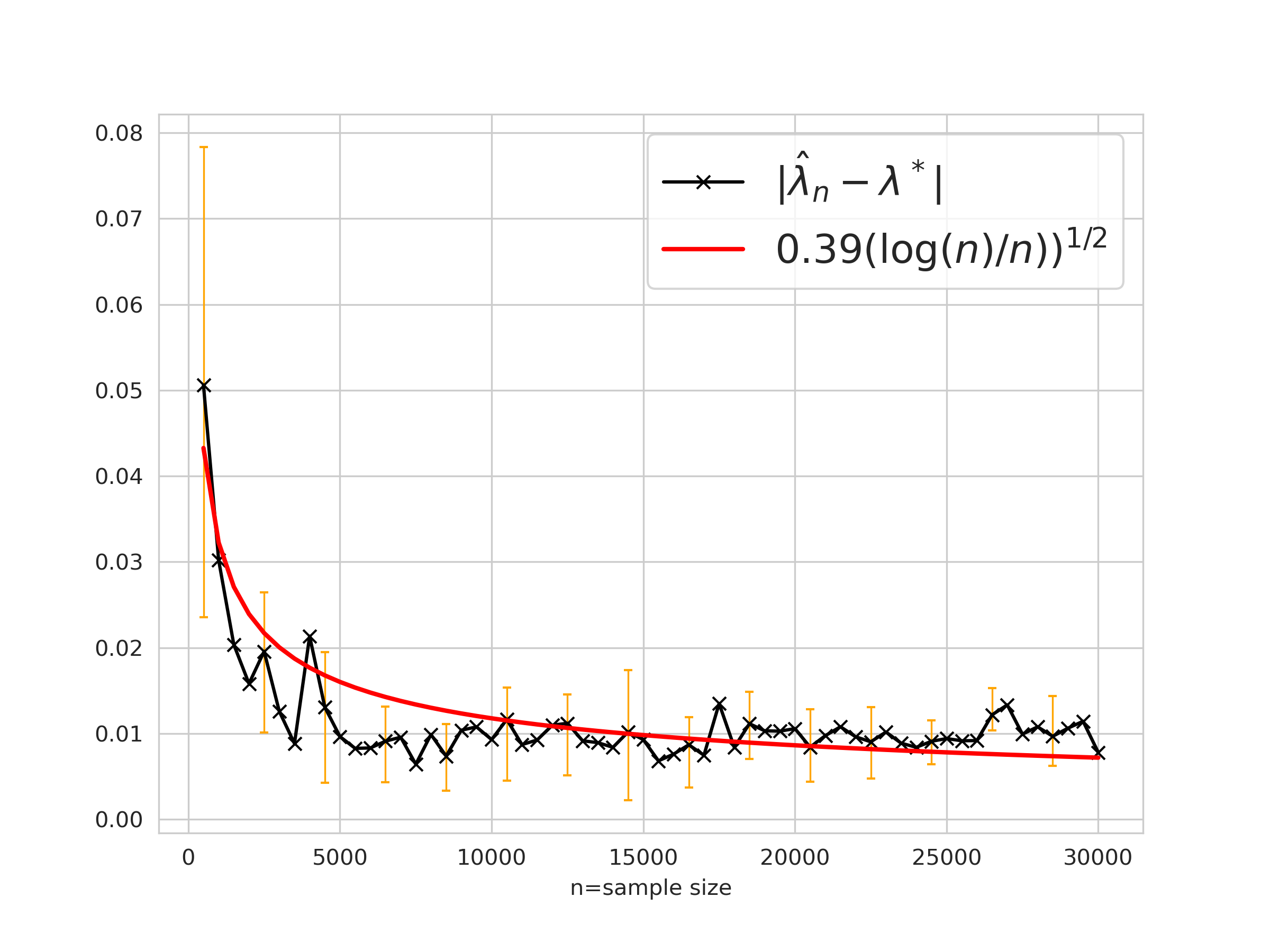}}
      \caption{Parameter learning rates in regime $\lambda \leq \lambda^*$.}\label{fig:partial-distinguishable-weak-ident}
    \end{figure}

\begin{figure}[ht]
      \centering
      \subcaptionbox*{\scriptsize (a) Convergence rates of $W_4(\hat{G}_n, G_{*})$ \par}{\includegraphics[width = 0.45\textwidth]{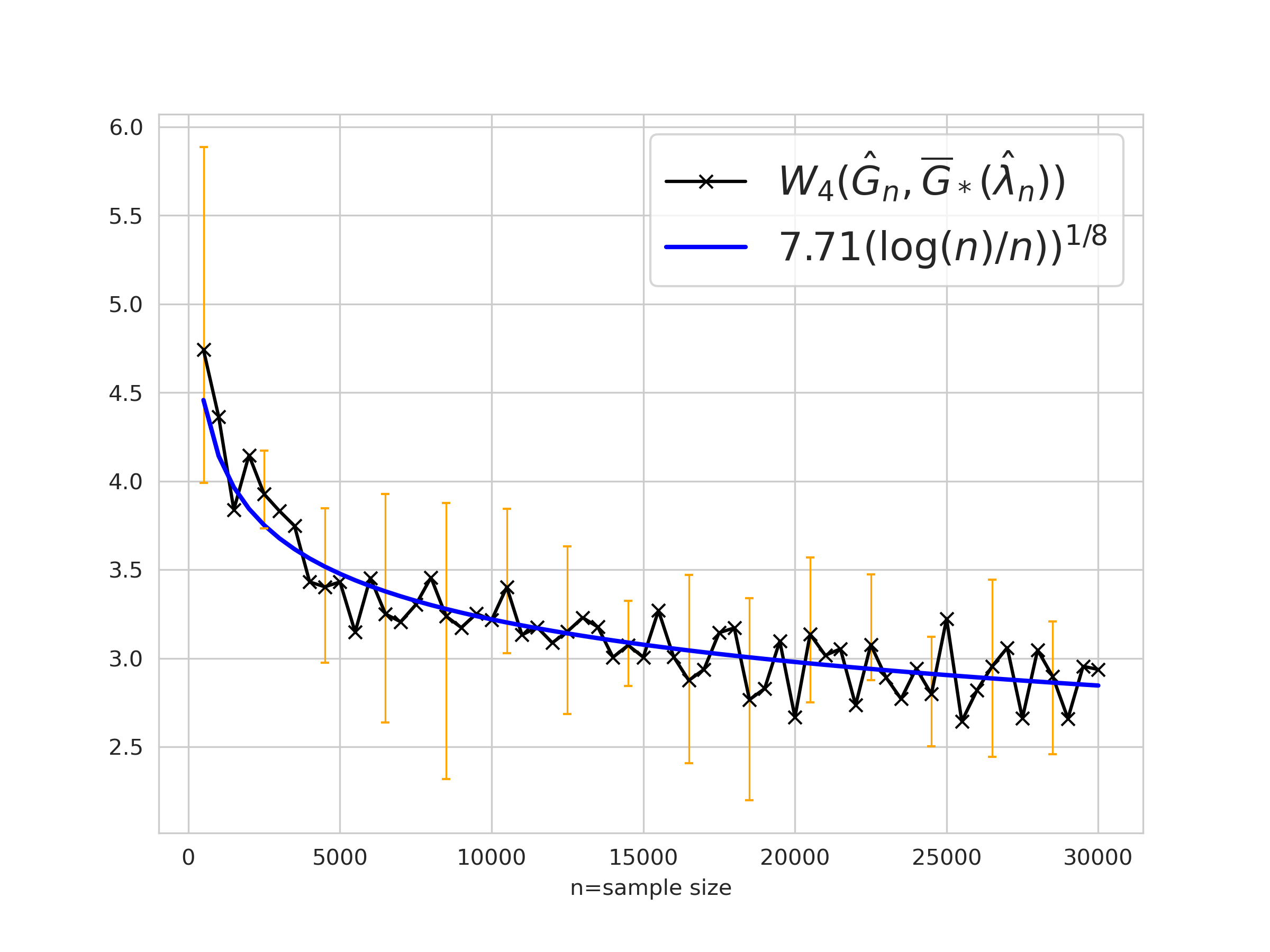}}
      \subcaptionbox*{\scriptsize (b) Limit of $\hat{\lambda}_n$ \par}{\includegraphics[width = 0.45\textwidth]{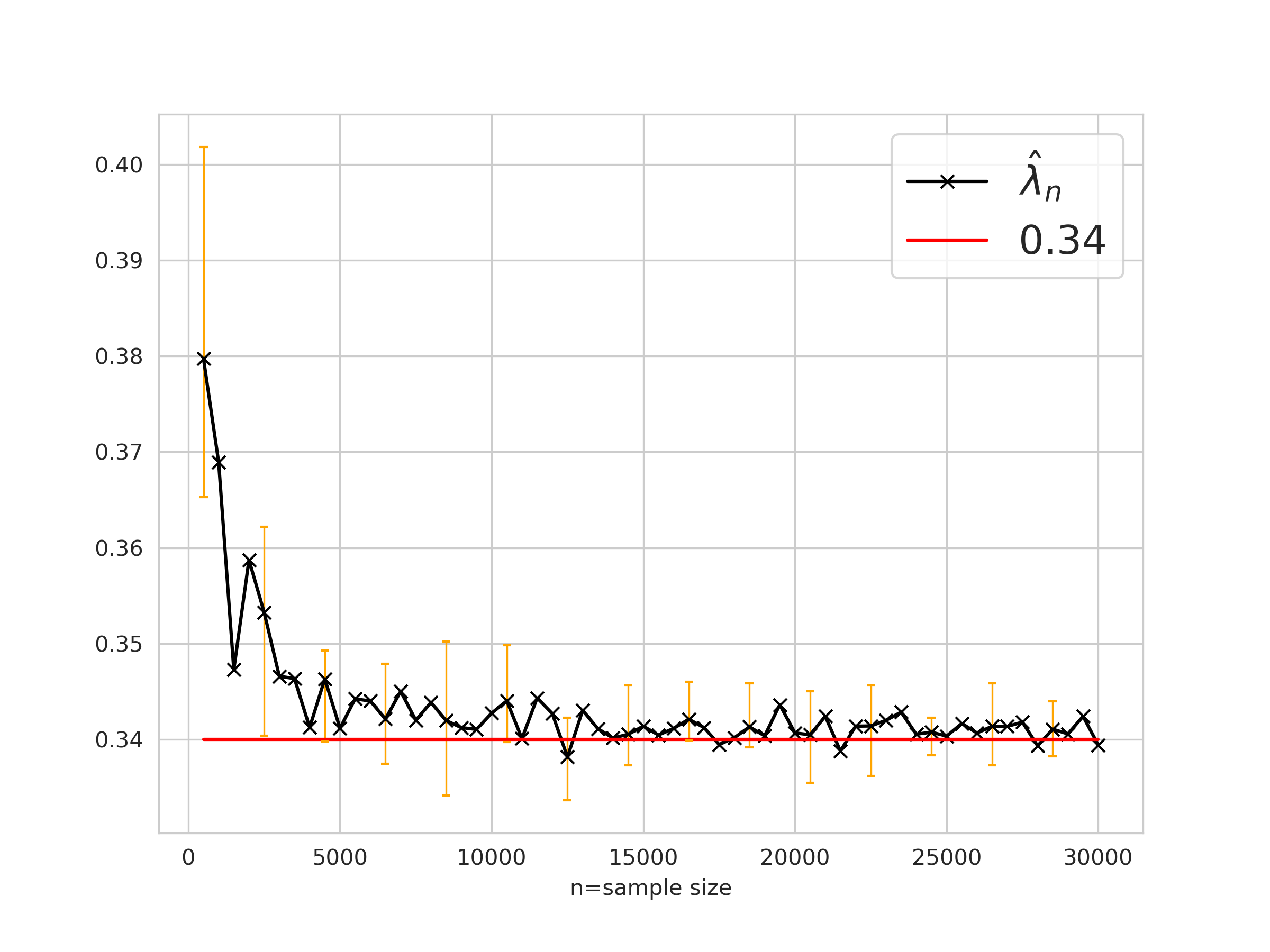}}
      \caption{Parameter learning rates in regime $\lambda >  \lambda^*$.}\label{fig:partial-distinguishable-weak-ident-large-lambda}
    \end{figure}

\section{Proofs of Section 2}
\label{sec:section_2}
\subsection{Proof of Theorem~\ref{theorem:example-distinguishable-Gaussian}}
\label{subsec:proof:theorem:example-distinguishable-Gaussian}
%\begin{customprop}{\ref{theorem:example-distinguishable-Gaussian}}
%\begin{enumerate}
%    \item[(a)] Suppose that $-\log h_0(x) \gtrsim \enorm{x}^{\beta_1}$ or $-\log h_0(x) \lesssim \enorm{x}^{\beta_2}$ for all $\enorm{x} > x_0$, for some $x_0 > 0$, $\beta_1 > 2$, and $\beta_2 < 2$. Then, $h_0$ is distinguishable with $(f, k)$ up to the first order, for $f$ being family of location-scale Gaussian and any $k > 0$. Besides, $h_0$ is distinguishable with $(f, k)$ up to the any order, for $f$ being family of location Gaussian and any $k > 0$.
%    \item[(b)] Suppose that $h_0$ is density of $N(0, I_d)$ being push-forwarded by a piecewise linear (non-linear) function with finite number of breakpoints. Then, $h_0$ is distinguishable with $(f, k)$ up to the first order, for $f$ being family of location-scale Gaussian and any $k > 0$. Besides, $h_0$ is distinguishable with $(f, k)$ up to the any order, for $f$ being family of location Gaussian and any $k > 0$. 
%\end{enumerate}
%\end{customprop}
(a) We first prove that $h_0$ is distinguishable with $(f, k)$ up to first order with any $k$ and $f$ being location-scale Gaussian family, i.e., if there exists $\lambda, \alpha_j\in \mathbb{R}, \beta_j \in \mathbb{R}^{d}, $ symmetric matrices $\gamma_i\in \mathbb{R}^{d\times d}$, $\theta_j\in \mathbb{R}^d$, and positive definite symmetric $\Sigma_j \in \mathbb{R}^{d\times d}$  for $j = 1,\dots, k$ such that
    \begin{equation*}
        \lambda h_0(x) + \sum_{j=1}^{k} \alpha_j f(x|\theta_j, \Sigma_j) + \sum_{j=1}^{k} \beta_j^T \dfrac{\partial f}{\partial \theta}(x|\theta_j, \Sigma_j) + \trace\left(\dfrac{\partial f}{\partial \Sigma}(x|\theta_j, \Sigma_j)^T \gamma_j \right) = 0,
    \end{equation*}
    then $\lambda = \alpha_j = \beta_j =  \gamma_j = 0$ for all $j = 1,\dots, k$, where $f(x|\theta, \Sigma)$ is the density evaluated at $x$ of Gaussian distribution with mean $\theta$ and covariance $\Sigma$ and $(\theta_j, \Sigma_j)_{j=1}^{k}$ are pairwise different. Suppose there exists such $(\lambda, \alpha_j, \beta_j, \gamma_j)_{j=1}^{k}$. We borrow a technique from \cite{Ho-Nguyen-EJS-16, yakowitz1968identifiability}, where we find a one-dimensional space to project $x \in \mathbb{R}^d$ onto and work with the order of means and variances in that space to show that the solution must be trivial. Calculating the first derivatives of $f$ gives
    \begin{equation}\label{eq:ident-location-scale}
        \lambda h_0(x) + \sum_{j=1}^{k} \left(\alpha_j' + (\beta_j')^T(x-\theta_j) + (x - \theta_j)^T \gamma_j^{-1} (x - \theta_j) \right) e^{-\dfrac{1}{2} (x-\theta_j)^T \Sigma_j^{-1} (x-\theta_j)} = 0,
    \end{equation}
    where
    \begin{equation*}
        \alpha_j' = \dfrac{2\alpha_j -\trace(\Sigma_j^{-1} \gamma_j)}{2\pi^{d/2} |\Sigma_j|^{1/2}}, \quad
        \beta_j' = \dfrac{2}{\pi^{d/2} |\Sigma_j|^{1/2}} \Sigma_j^{-1} \beta_j, \quad
        \gamma_j' = \dfrac{1}{\pi^{d/2} |\Sigma_j|^{1/2}} \Sigma_j^{-1}\gamma_j \Sigma_j^{-1},
    \end{equation*}
    for all $j = 1,\dots, k$. If all the covariance matrices are equal, i.e., $\Sigma_1 = \dots = \Sigma_k$, then $(\theta_j)_{j=1}^{k}$ are pairwise different. Denote by $\delta_{ij} = \theta_i - \theta_j$, then for any $x' \not \in \cup_{1\leq i\leq j\leq k} \{ u\in\mathbb{R}^d : \delta_{ij}^T u = 0 \}$, we have $(x')^T \theta_1, \dots, (x')^T \theta_k$ are distinct. Otherwise, if (without loss of generality) there are $\Sigma_1, \dots, \Sigma_m$ different matrices among $\Sigma_1, \dots, \Sigma_k$, then for every $x' \not \in \cup_{1\leq i\leq j\leq m} \{ u\in\mathbb{R}^d : u^T (\Sigma_i-\Sigma_j) u = 0 \}$, we have $(x')^T \Sigma_1 (x'), \dots, (x')^T \Sigma_m (x')$ are distinct. In both cases, we find a finite collection of hyperplanes and cones such that for every $x'$ not belongs to any set of this collection, we have $((x')^T \theta_1, (x')^T \Sigma_1(x')), \dots, ((x')^T \theta_k, (x')^T \Sigma_k(x'))$ are pairwise different. Note that because the union of these collection of $(d-1)$ dimensional manifolds can not be $\mathbb{R}^d$, such a non-zero $x'$ exists. Now we only consider $x$ belongs to the one-dimensional linear space spanned by this $x'$, i.e., $x = y (x')$, where $y\in \mathbb{R}$. Denote by 
    \begin{equation*}
        a_j = (x')^T \gamma_j' x',\,\, b_j = [(\beta_j')^T - 2 \theta_j^T\gamma_j'] x',\,\, c_j = \theta_j^T \gamma_j' \gamma_j - (\beta_j')^T \theta_j,
    \end{equation*}
    \begin{equation*}
        d_j = (x')^T \Sigma_j^{-1} x',\,\, e_j = (x')^T \Sigma_j^{-1} \theta_j',\,\, f_i = \theta_j^T \Sigma_j^{-1} \theta_j,
    \end{equation*}
    for $j = 1,\dots, k$, we proved that $((d_j, e_j))_{j=1}^{k}$ are distinct. Equation~\eqref{eq:ident-location-scale} implies that 
    \begin{equation}\label{eq:ident-location-scale-1d}
        \lambda h_0(yx') + \sum_{j=1}^{k}( \alpha_j' + a_j y^2 + b_j y + c_j) \exp(d_j y^2 + e_j y + f_j) = 0.
    \end{equation}
    \paragraph{Case 1.} If $-\log h_0(x) \gtrsim \enorm{x}^{\beta_1}$ for some $\beta_1 > 2$ and for all $\enorm{x} > x_0$, we have $h_0(x) \lesssim \exp^{-\enorm{x}^{\beta_1}}$. Choose $d_{i_1} = \max_{1\leq i \leq k} d_k$ and $e_{i_2} = \max \{e_{j}: d_j = d_{j_1}\}$. Because $h_0$ has a lighter tail than Gaussian and 
    \begin{equation*}
        d_j y^2 + e_j y + f_j < d_{i_2} y^2 + e_{i_2} y + f_{i_2}, \quad\forall j\neq i_2,
    \end{equation*}
    for all $y$ large enough, divide both sides of~\eqref{eq:ident-location-scale-1d} by $\exp(d_{i_2} y^2 + e_{i_2} y + f_{i_2})$ and let $y \to \infty$, we have $a_{i_2} = b_{i_2} = 0$. It implies that $(x')^T \gamma_{i_2}' x' = [(\beta_{i_2}')^T - 2 \theta_{i_2}^T\gamma_{i_2}'] x' = 0 $. If $\gamma_{i_2}' \neq 0$ then we can further choose $x'$ outside a cone such that $(x')^T \gamma_{i_2}' x'\neq 0$. Hence, $\gamma_{i_2} = 0$, which implies $(\beta_{i_2}')^T (x') = 0$. If $\beta_{i_2} \neq 0$ then we can further choose $x'$ outside a hyperplane such that $(\beta_{i_2}')^T (x') \neq 0$. Hence, in any case, we can argue so that $\beta'_{i_2} = \theta'_{i_2} = 0$. Put it back to \eqref{eq:ident-location-scale-1d}, we also have $\alpha'_{i_2} = 0$. Therefore, $\alpha_{i_2} = \beta_{i_2} = \gamma_{i_2} = 0$. Repeat the same argument, notice that the tail of $h_0$ is lighter than any Gaussian distribution, we have $\alpha_j = \beta_j = \gamma_j = 0$ for all $j=1,\dots, k$. It finally leads to $\lambda = 0$. Hence, we have the distinguishability of $h_0$ with family of location-scale Gaussians up to first order.
    \paragraph{Case 2.} If $-\log h_0(x) \lesssim \enorm{x}^{\beta_2}$ for some $\beta_2 <2$ and for all $\enorm{x} > x_0$. We have $p(x|\theta_j, \Sigma_j) / h_0(x) \to 0$ as $x\to \infty$ for all $j = 1,\dots, k$. Therefore, dividing both sides of~\eqref{eq:ident-location-scale} by $h_0(x)$ and let $x \to\infty$ by some direction, we have $\lambda = 0$. Now proceed to argue similar to Case 1, we also have the distinguishability of $h_0$ with family of location-scale Gaussians up to first order.
    
    Now we proceed to prove that $h_0$ is distinguishable with $(f, k)$ up to the any order, for $f$ being family of location Gaussian and any $k > 0$. Arguing similar to above, we only need to work on one-dimensional space. Suppose that there exists $\lambda, (c_{i, j})_{i=1, \dots, k, j = 1,\dots, r}$ such that
    \begin{equation}\label{eq:ident-location-raw}
        \lambda h_0(x) + \sum_{i=1}^{k} \sum_{j=0}^{r} c_{i, j} \dfrac{\partial^{j} f}{\partial \theta^{j}} (x|\theta_i, v_i) = 0,  
    \end{equation}
    where $f(\cdot|\theta, v)$ is the density function of normal distribution with mean $\theta$ and variance $v$, and $(\theta_1, v_1), \dots, (\theta_k, v_k)$ are distinct. We need to prove that $\lambda = c_{i,j} = 0$ for all $i=1, \dots, k, j = 1,\dots, r$. Calculating the partial derivatives of $f$, we have
    \begin{equation}\label{eq:ident-location}
        \lambda h_0(x) + \sum_{i=1}^{k} \left(\sum_{j=0}^{r} \gamma_{i, j}  (x-\theta_i)^{j} \right) \exp\left(-\dfrac{(x-\theta_i)^2}{2v_i} \right) = 0,
    \end{equation}
    such that $\gamma_{i,j}$ for odd j are linear combination of $(c_{i, l})$ with odd $l \leq j$, $\gamma_{i,j}$ for even j are linear combination of $(c_{i, l})$ with even $l \leq j$, and one can prove (for example, by induction) that $\gamma_{i,j}=0\forall j$ is equivalent to $c_{i,j} = 0\forall j$. Now we can argue similar to Case 1 and Case 2 above to get the contradiction, with the notice that polynomials grow slower than exponential functions.
    
(b) Let $T$ be a piecewise linear function with a positive finite number of breakpoints and $h_0$ is the density function of $N(0,I_d)$ being pushforwarded by $T$. Argue similar to above, we only need to prove the result in one-dimensional case. In order to prove the distinguishable of $h_0$ with mixtures of location Gaussians family or mixtures of location-scale Gaussians family, it all boils down to prove that if there exists $\lambda \in \mathbb{R}$ and polynomials $Q_1(x), Q_2(x), \dots, Q_k(x)$ such that  
    \begin{equation}\label{eq:ident-pushforwarded-normal}
        \lambda h_0(x) + \sum_{i=1}^{k} Q_i(x) f(x|\theta_i, v_i^2) = 0,
    \end{equation}
    where $(\theta_1, v_1^2), \dots, (\theta_k, v_k^2)$ are distinct, then $\lambda = Q_1(x) = \dots = Q_k(x) = 0$. We will prove this by induction in $k$. Consider the case $k=1$, we have
    \begin{equation}\label{eq:ident-pushforwarded-normal-k-1}
        \lambda h_0(x) +  Q_1(x) f(x|\theta_1, v_1^2) = 0.
    \end{equation}
    Because $T$ has finite number of break points, there exists some $x_0$ large enough so that for all $x > x_0$, $T$ is a linear one-to-one function between $[x_0, \infty)$ and its image. Denote by $T(x) = ax 
    + b$ when $x > x_0$. We can argue that $a \neq 0$, because otherwise the distribution of $h_0$ will has an atom, which directly leads to distinguishability between $h_0$ and mixtures of Gaussians. Then, 
    $h_0(x) = f(x|b, a^2)$ and we have
    \begin{equation*}
        \lambda f(x|b, a^2) +  Q_1(x) f(x|\theta_1, v_1^2) = 0.
    \end{equation*}
    Argue similar to part (a), if $(b, a^2) \neq (\theta_1, v_1^2)$, we have $\lambda = Q_1(x) = 0$, which implies the distinguishability. Otherwise, we have $b = \theta_1, a^2 = v_1^2$, and $Q_1(x) = - \lambda$ for all $x\in \mathbb{R}$. We can rewrite \eqref{eq:ident-pushforwarded-normal-k-1} as 
    \begin{equation*}
        h_0(x) - f(x|\theta_1, v_1^2) = 0.
    \end{equation*}
    Because $h_0$ is $N(0,1)$ being pushforwarded by a piecewise linear function, we can write $\mathbb{R}$ as a partition $(-\infty, c_1], (c_1, c_2], \dots, [c_m, \infty)$ such that each semi-open interval is image of some linear functions of $T$. Consider a semi-open interval $(c_{i}, c_{i+i}]$ being image of $T_j(z) = a_j z + b_j$ for $j=1,\dots, h$, by the change of variable formula for many-to-one map, we have
    \begin{equation}\label{eq:ident-pushforwarded-normal-k-1-rewrite}
        0 = h_0(x) - f(x|\theta_1, v_1^2)  = \sum_{j=1}^{h} f(x | b_j, a_j^2) -  f(x|\theta_1, v_1^2),
    \end{equation}
    for all $x\in (c_{i}, c_{i+i}]$. Applying Lemma \ref{lem:local-identifiable}, we have equation~\eqref{eq:ident-pushforwarded-normal-k-1-rewrite} is true for all $x\in \mathbb{R}$. Hence, by integrating both side, we get $h = 1$, and then $b_1 = \theta_1, a_1^2 = v_1^2$. Because this is true for all semi-open intervals $(c_{i}, c_{i+i}]$, we have $T(x) = a_1x + b_1$ for all $x\in \mathbb{R}$, which is contradict to our assumption that $T$ is non-linear.
    
    Suppose that our inductive hypothesis is correct for $k=n$, now we proceed to prove it is true for $k = n+1$. If there exists $\lambda \in \mathbb{R}$ and polynomials $Q_1(x), Q_2(x), \dots, Q_{n+1}(x)$ such that  
    \begin{equation}\label{eq:ident-pushforwarded-normal-k-plus-1}
        \lambda h_0(x) + \sum_{i=1}^{n+1} Q_i(x) f(x|\theta_i, v_i^2) = 0,
    \end{equation}
    where $(\theta_1, v_1), \dots, (\theta_{n+1}, v_{n+1}^2)$ are distinct. Without loss of generality, assume that $v^2_1 = \max_{1\leq i \leq n+1} v^2_k$ and $\theta_{1} = \max \{\theta_{j}: v_j^2 = v_1^2 \}$. Because $T$ has finite number of break points, there exists some $x_0$ large enough so that for all $x > x_0$, $T$ is a linear one-to-one function between $[x_0, \infty)$ and its image. Denote by $T(x) = ax 
    + b$ when $x > x_0$. We have
    \begin{equation}\label{eq:ident-pushforwarded-normal-k-plus-1-gtr-x0}
        \lambda f(x | b, a^2) + \sum_{i=1}^{n+1} Q_i(x) f(x|\theta_i, v_i^2) = 0, \quad \forall \, x > x_0.
    \end{equation}
    If $a^2 > v_1^2$ or $a^2 = v_1^2, b > \theta_1$, divide both sides of equation~\eqref{eq:ident-pushforwarded-normal-k-plus-1-gtr-x0} by $\exp((x-b)/2a^2)$ and let $x\to \infty$, we have $\lambda = 0$ and the conclusion follows from the identifiability of Gaussians family. 
    
    If $v_1^2 > a^2$ or $ v_1^2=a^2, \theta_1 > b$, divide both sides of equation~\eqref{eq:ident-pushforwarded-normal-k-plus-1-gtr-x0} by $\exp((x-\theta_1)/2v_1^2)$ and let $x\to \infty$, we have $Q_1(x) = 0$. The problem is back to the case $k=n$ and is proved using the inductive hypothesis. 
    
    If $a^2 = v_1^2, b = \theta_1$, divide both sides of equation~\eqref{eq:ident-pushforwarded-normal-k-plus-1-gtr-x0} by $\exp((x-b)/2a^2)$ and let $x\to \infty$, we have $Q_1(x) = -\lambda$ for all $x\in \mathbb{R}$. Hence for $x$ large enough, 
    $$\sum_{i=2}^{n+1} Q_i(x) f(x|\theta_i, v_i^2) = 0,$$
    which implies $Q_2(x) = \dots = Q_{n+1} (x) = 0$. The problem is back to the case $k=1$ and is proved using the inductive hypothesis. 
    
The following lemma presents the local identifiability of location-scale Gaussians mixtures.
\begin{lemma}\label{lem:local-identifiable}
    Denote by $f(\cdot | \theta, \sigma^2)$ the density function of Gaussian distribution with mean $\theta $ and variance $\sigma^2$. For all $a < b$ and pairs $\{(\theta_i, \sigma_i^2)\}_{i=1}^{k}$, if there exists $\alpha_1, \alpha_2, \dots, \alpha_n \in \mathbb{R}$ such that
    \begin{equation*}
        \alpha_1 f(x | \theta_1, \sigma_1^2) + \dots + \alpha_k f(x | \theta_k, \sigma_k^2) = 0
    \end{equation*}
    for all $x\in [a, b]$, then 
    \begin{equation}\label{eq:lem-ident-conclusion}
        \alpha_1 f(x | \theta_1, \sigma_1^2) + \dots + \alpha_k f(x | \theta_k, \sigma_k^2) = 0,
    \end{equation}
    for all $x\in \mathbb{R}$.
\end{lemma}
\begin{proof}
    \textit{Step 1. (Centralize and normalize coefficients). }
    Suppose that there exists $\alpha_1, \alpha_2, \dots, \alpha_n \in \mathbb{R}$ such that
    \begin{equation*}
        \alpha_1 f(x | \theta_1, \sigma_1^2) + \dots + \alpha_k f(x | \theta_k, \sigma_k^2) = 0
    \end{equation*}
    for all $x\in [a, b]$. Denote by $\theta'_i = \theta_i - \dfrac{a+b}{2}$ for all $i=1, \dots, k$, then 
    \begin{equation}\label{eq:lem-local-ident}
        \alpha_1 \dfrac{1}{\sqrt{2\pi} \sigma_1} \exp\left(-\dfrac{(x-\theta_1')^2}{2\sigma_1^2} \right) + \dots + \alpha_k \dfrac{1}{\sqrt{2\pi} \sigma_k} \exp\left(-\dfrac{(x-\theta_k')^2}{2\sigma_k^2} \right) = 0,
    \end{equation}
    for all $x\in [-\frac{b-a}{2}, \frac{b-a}{2}]$. Denote by $\sigma_{i_1} = \min \{\alpha_1, \dots, \alpha_k\}$. Multiple both sides of \eqref{eq:lem-local-ident} by $\exp(-\frac{x^2}{\sigma_{i_1}^2})$, and denote by $s_i^2 = \frac{1}{\sigma_{i_1}^2} - \frac{1}{2\sigma_{i}^2}, m_i = \theta_i'/\sigma_i^2, \beta_i = \dfrac{1}{\sqrt{2\pi}\sigma_i} \exp(-(\theta_i')^2/2\sigma_i^2)$ for all $i=1,\dots, k$, we have
    \begin{equation}\label{eq:lem-local-ident-rewrite}
        \beta_1 \exp\left(s_1^2 x^2 + m_1 x\right) + \dots + \beta_k \exp\left(s_k^2 x^2 + m_k x\right) = 0,
    \end{equation}
    for all $x \in [-\frac{b-a}{2}, \frac{b-a}{2}]$.
    
    \textit{Step 2. (Use properties of Laplace transformation).} The left-hand side of equation \eqref{eq:lem-local-ident-rewrite} is the Laplace transformation of $\sum_{i=1}^{k} \beta_i f(x | m_i, s_i^2)$ and is identical to 0 in an open set around 0. Hence $$\sum_{i=1}^{k} \beta_i f(x | m_i, s_i^2) = 0,$$
    for all $x\in \mathbb{R}$. It implies that 
    \begin{equation*}
        \beta_1 \exp\left(s_1^2 x^2 + m_1 x\right) + \dots + \beta_k \exp\left(s_k^2 x^2 + m_k x\right) = 0,
    \end{equation*}
    for all $x \in \mathbb{R}$, which is equivalent to equation~\eqref{eq:lem-ident-conclusion}.
\end{proof}
\subsection{Proof of Proposition~\ref{prop:KDE-distinguish}}
\label{subsec:proof:prop:KDE-distinguish}
If $k_{\sigma}$ is the Gaussian kernel with $m > K$, then we get the conclusions as  direct consequences of Example~\ref{example:distinguishable_condition}(a). If $k_{\sigma}$ is the multivariate Student kernel, then $h_0$ has a tail heavier than Gaussian tail, so that we get the conclusions as consequences of Proposition~\ref{theorem:example-distinguishable-Gaussian}(a).
\subsection{Proof of Proposition~\ref{prop:neural_networks}}
\label{subsec:proof:prop:neural_networks}
Because $T$ has a finite and postive number of layers, it is a piecewise linear and non-linear function. So the conclusions are direct consequences of Proposition~\ref{theorem:example-distinguishable-Gaussian}(b).

{ 
\subsection{Proof of Theorem~\ref{thm:density_estimation_rate}}
This result can be obtained by modifying the proof of Theorem 7.4 in \cite{geer2000empirical}. Recall that we defined the function class 
\begin{equation}
\overline{\mathcal{P}}^{1/2}_k(\Theta,\epsilon)  
 = \left\{\bar p_{\lambda G}^{1/2}  : G \in \mathcal{O}_k(\Theta), ~ h(\bar p_{\lambda G}, p_{\lambda^* G_*}) \leq \epsilon\right\},
\end{equation}
where for any $G \in \mathcal{O}_K(\Theta)$, 
we write $\bar{p}_{\lambda G} = (p_{\lambda G} + p_{\lambda^*G_*})/2$, 
and measure the complexity of this class
through the bracketing entropy integral
\begin{equation*}   
\mathcal{J}_B(\epsilon, \overline{\mathcal{P}}^{1/2}_k(\Theta,\epsilon), \nu)
= \int_{\epsilon^2/2^{13}}^{\epsilon} \sqrt{\log N_B(u, \overline{\mathcal{P}}_k^{1/2}(\Theta,u), \nu)}du \vee \epsilon,
 \end{equation*}
where $N_B(\epsilon, X, \eta)$ denotes
the $\epsilon$-bracketing number
of a metric space $(X,\eta)$ and $\nu$ is the Lebesgue measure. We denote by $P_{\lambda G}$ the distribution corresponding to the density $p_{\lambda G}$. The technique to prove this theorem is to bound the convergence rate by the increments of an empirical processes: 
$$\nu_n(\lambda G) = \sqrt{n} \int_{\{p_{\lambda^* G_*}\} > 0} \dfrac{1}{2} \log \dfrac{\overline{p}_{\lambda G}}{p_{\lambda^* G_*}} d(P_n - P_{\lambda^* G_*}),$$ 
where $P_n = \frac{1}{n}\sum_{i=1}^{n} \delta_{X_i}$ is the empirical measure ($X_1, \dots, X_n \overset{iid}{\sim} p_{\lambda^* G_*}$).
We first recall Theorem 5.11 in \cite{geer2000empirical} with the notations adapted from our setting:
\begin{theorem}\label{thm:bound-tail-uniform}
    Let $R >0$, $k\geq 1$, and $\mathcal{G}$ be a subset of $\mathcal{O}_k(\Theta)$, which contains $G_*$. Given $C_1 < \infty$, for all $C$ sufficiently large, and for $n\in \mathbb{N}$ and $t>0$ satisfying
    \begin{equation}\label{eq:cond-t-leq}
        t \leq \sqrt{n}((8R) \wedge (C_1 R^2)),
    \end{equation}
    and
    \begin{equation}\label{eq:cond-t-geq}
        t\geq C^2(C_1+1)\left(R \vee \int_{t/(2^6\sqrt{n})}^{R} H_{B}^{1/2}\left(\frac{u}{\sqrt{2}}, \overline{\mathcal{P}}_k^{1/2}(\Theta, R), \nu\right)du\right),
    \end{equation}
    we have
    \begin{equation}
        \mathbb{P}_{\lambda^* G_*} \left(\sup_{G\in \mathcal{G}, h(\overline{p}_{\lambda G}, p_{\lambda^* G_*})\leq R} |\nu_{n}(\lambda G)| \geq t\right) \leq C \exp\left(-\dfrac{t^2}{C^2(C_1+1)R^2}\right).
    \end{equation}
\end{theorem}
Now we proceed to prove Theorem~\ref{thm:density_estimation_rate}, the proof is divided into three parts: Bounding the tail probability of $h(p_{\hat{\lambda}_n \hat{G}_n}, p_{\lambda^* G_*})$ by sums of empirical processes increments using chaining technique, bounding the empirical processes increments using Theorem~\ref{thm:bound-tail-uniform}, and bounding the expectation of $h(p_{\hat{\lambda}_n \hat{G}_n}, p_{\lambda^* G_*})$ using its tail probability. 

\paragraph{Step 1 (Bounding the tail probability $h(p_{\hat{\lambda}_n \hat{G}_n}, p_{\lambda^* G_*})$ by sums of empirical processes increments):} Firstly, by Lemma 4.1 and 4.2 of \cite{geer2000empirical}, we have 
$$\dfrac{1}{16}h^2(p_{\hat{\lambda}_n \hat{G}_n}, p_{\lambda^* G_*}) \leq h^2(\overline{p}_{\hat{\lambda}_n \hat{G}_n}, p_{\lambda^* G_*})\leq \dfrac{1}{\sqrt{n}} \nu_n(\hat{\lambda}_n\hat{G}_n).$$
Hence, for any $\delta > \delta_n := (\log n / n)^{1/2}$, we have
\begin{align*}
    \Pbb_{\lambda^* G_*}(h(p_{\hat{\lambda}_n \hat{G}_n}, p_{\lambda^* G_*})\geq \delta) & \leq \Pbb_{\lambda^* G_*}\bigg( \nu_n(\hat{\lambda}_n \hat{G}_n) - \sqrt{n} h^2(\overline{p}_{\hat{\lambda}_n \hat{G}_n}, p_{\lambda^* G_*})\geq 0,\\
    & \hspace{5cm} h(\overline{p}_{\hat{\lambda}_n \hat{G}_n}, p_{\lambda^* G_*}) \geq \delta / 4 \bigg)\\
    & \leq \Pbb_{\lambda^*G_*}\left(\sup_{\lambda, G: h(\overline{p}_{\lambda G}, p_{\lambda^* G_*}) \geq \delta / 4} [ \nu_n(\lambda G) - \sqrt{n} h^2 (\overline{p}_{\lambda G}, p_{\lambda^* G_*})] \geq 0 \right)\\
    & \leq \sum_{s=0}^{S} \Pbb_{\lambda^*G_*}\left(\sup_{\lambda, G: 2^{s}\delta/4\leq h(\overline{p}_{\lambda G}, p_{\lambda^* G_*})\leq 2^{s+1}\delta/4} | \nu_n(\lambda G)| \geq \sqrt{n}2^{2s} (\delta/4)^{2} \right)\\
    &\leq \sum_{s=0}^{S} \Pbb_{\lambda^*G_*}\left(\sup_{\lambda, G:  h(\overline{p}_{\lambda G}, p_{\lambda^* G_*})\leq 2^{s+1}\delta/4} | \nu_n(\lambda G)| \geq \sqrt{n}2^{2s} (\delta/4)^{2} \right),
\end{align*}
where $S$ is a smallest number such that $2^S \delta /4 > 1$, as $h(\overline{p}_{\lambda G}, p_{\lambda^* G_*}) \leq 1$. Now we will bound the each term above using Theorem~\ref{thm:bound-tail-uniform}.
\paragraph{Step 2 (Bounding the empirical processes increments using Theorem~\ref{thm:bound-tail-uniform}):} In Theorem~\ref{thm:bound-tail-uniform}, choose $R = 2^{s+1}\delta, C_1=15$ and $t = \sqrt{n} 2^{2s} (\delta/4)^2$, we can readily check that Condition~\eqref{eq:cond-t-leq} satisfies (because $2^{s-1}\delta/4\leq 1$ for all $s=0, \dots, S$). Condition~\eqref{eq:cond-t-geq} satisfies thanks to Assumption A3:
\begin{align*}
    \int_{t/(2^6\sqrt{n})}^{R} H_B^{1/2}\left(\dfrac{u}{\sqrt{2}}, \mathcal{P}_{k}^{1/2}(\Theta, R), \nu\right) du \vee 2^{s+1} \delta  & = \sqrt{2}\int_{R^2/2^{13}}^{R/\sqrt{2}} H_B^{1/2}\left(u, \mathcal{P}_{k}^{1/2}(\Theta, R), \nu\right) du \vee 2^{s+1} \delta\\
    & \leq 2 \mathcal{J}_B(R, \mathcal{P}^{1/2}(\Theta, R), \nu)\\
    &\leq 2J \sqrt{n}2^{2s+1} \delta^2 = 2^{6} J t.
\end{align*}
So the conclusion of Theorem~\ref{thm:bound-tail-uniform} gives us 
\begin{equation}
    \Pbb_{\lambda^* G_*}(h(p_{\hat{\lambda}_n \hat{G}_n}, p_{\lambda^* G_*}) > \delta) \leq C\sum_{s=0}^{\infty} \exp\left(\dfrac{2^{2s}n\delta^2}{J^2 2^{14}}\right) \leq c\exp\left(\dfrac{n\delta^2}{c^2}\right),
\end{equation}
where $c$ is a large constants that does not depend on $\lambda^*, G_*$. 
\paragraph{Step 3 (Implying the bound on supremum of expectation):}
Thus, we have
\begin{equation*}
    \mathbb{E} h(p_{\hat{\lambda}_n \hat{G}_n}, p_{\lambda^* G_*}) = \int_0^{\infty} \mathbb{P}(h(p_{\hat{\lambda}_n \hat{G}_n}, p_{\lambda^* G_*}) >\delta) d\delta \leq \delta_n + c\int_{\delta_n}^{\infty} \exp\left(-\dfrac{n\delta^2}{c^2} \right)\leq \tilde{c} \delta_n,
\end{equation*}
for some $\tilde{c}$ does not depend on $\lambda^*, G_*$. Hence, we finally proved that
\begin{equation*}
\sup_{G_{*} \in \mathcal{O}_k(\Theta), \lambda^*\in [0, 1]} \mathbb{E}_{\lambda^*, G_*} h(p_{\widehat{\lambda}_n\widehat{G}_n}, p_{\lambda^*G_*})  
\leq  C\sqrt{\log n/n}.
\end{equation*}
As a consequence, we obtain the conclusion of the theorem.
}

\section{Proof of Section~\ref{sec:conv-rate-density-estimation}}\label{sec:section_3}
\subsection{Proof of Proposition~\ref{prop:entropy-number-calculation}}
We first need to denote some notations that are required for the proof. Those notations are well-known in Empirical Processes field \cite{geer2000empirical}. Denote by 
\begin{equation*}
    \mathcal{P}_k(\Theta) = \{p_{\lambda G}: \lambda\in [0,1], G\in\Ocal_k(\Theta) \},
\end{equation*}
and let $N(\epsilon, \Pcal_k(\Theta), \|\cdot\|_{\infty})$ be the $\epsilon-$covering number of $(\Pcal_k(\Theta, \|\cdot\|_{\infty})$ and $N_B(\epsilon, \mathcal{P}_k(\Theta), h)$ be the bracketing number of $\mathcal{P}_k(\Theta)$ measured by Hellinger metric $h$. $H_B(\epsilon, \mathcal{P}_k(\Theta), h) = \log  N_B(\epsilon, \mathcal{P}_k(\Theta), h)$ is called the bracketing entropy of $\mathcal{P}_k(\Theta)$ under metric $h$. Let $\overline{\mathcal{P}}_k(\Theta) = \{(p_{\lambda G} + p_{\lambda^* G_*})/2 : \lambda \in [0,1], G\in \Ocal_k(\Theta) \}$ and $\overline{\mathcal{P}}_k^{1/2}(\Theta) = \{p^{1/2} : p\in \overline{\mathcal{P}}_k(\Theta) \}$.
We want to show that
\begin{equation}\label{eq:entropy-integral-bound}
    \mathcal{J}_B(\epsilon, \overline{\Pcal_k}^{1/2}(\Theta, \epsilon), L^2(\mu)) = \left(\int_{\epsilon^2/2^{2^{13}}}^{\epsilon} H_B^{1/2}(\delta, \overline{\Pcal_k}^{1/2}(\Theta, \delta), \nu) d\delta \vee \delta\right) \lesssim \sqrt{n} \epsilon^2,
\end{equation}
for all $n > N$ large enough and $\epsilon > (\log n / n)^{1/2}$. We proceed to show that claim~\eqref{eq:entropy-integral-bound} will be proved if 
\begin{equation}\label{eq:covering-supnorm-bound}
    \log N (\epsilon, \mathcal{P}_{k}(\Theta), \norm{\cdot}_{\infty}) \lesssim \log(1/\epsilon),
\end{equation}
\begin{equation}\label{eq:entropy-hellinger-bound}
    H_B (\epsilon, \mathcal{P}_{k}(\Theta), h) \lesssim \log(1/\epsilon),
\end{equation}
and then prove claim~\eqref{eq:covering-supnorm-bound} and~\eqref{eq:entropy-hellinger-bound}. 

\paragraph{Proof of that claim~\eqref{eq:entropy-hellinger-bound} implies claim~\eqref{eq:entropy-integral-bound}} Because $\overline{\mathcal{P}}_k^{1/2}(\Theta, \delta) \subset \overline{\mathcal{P}}_k^{1/2}(\Theta)$ and from the definition of Hellinger distance, 
\begin{equation*}
    H_B(\delta, \overline{\mathcal{P}}_k^{1/2}(\Theta, \delta), \mu) \leq H_B(\delta, \overline{\mathcal{P}}_k^{1/2}(\Theta), \mu) = H_B(\frac{\delta}{\sqrt{2}}, \overline{\mathcal{P}}_k(\Theta), h).
\end{equation*}
Now use the fact that for densities $f_*, f_1, f_2$, we have $h^2((f_1 + f_*)/2, (f_2 + f_*)/2)\leq h^2(f_1, f_2)/2$, oen can readily check that $H_B(\frac{\delta}{\sqrt{2}}, \overline{\mathcal{P}}_k(\Theta), h)\leq H_B(\delta, \mathcal{P}_k(\Theta), h)$. Hence, if claim~\eqref{eq:entropy-hellinger-bound} holds true, then 
\begin{equation*}
    H_B(\delta, \overline{\mathcal{P}}_k^{1/2}(\Theta, \delta), \mu) \leq H_B(\delta, \mathcal{P}_k(\Theta), h) \lesssim \log (1/\delta),
\end{equation*}
which implies that
\begin{equation*}
    \mathcal{J}_B(\epsilon, \overline{\mathcal{P}}_k^{1/2}(\Theta, \delta), \mu)\lesssim \epsilon (\log(2^{13}/\epsilon^2 ))^{1/2} < n \epsilon^2,
\end{equation*}
for all $\epsilon > (\log n / n)^{1/2}$. Hence, claim~\eqref{eq:entropy-integral-bound} is proved. 
\paragraph{Proof of claim~\eqref{eq:covering-supnorm-bound}} By invoking the proof of Lemma 2.1. of \cite{Ho-Nguyen-Ann-16}, we have a $\epsilon$-net $\mathcal{S}$ for $(\{p_{G}: G\in \mathcal{O}_k(\Theta), h \})$ with the cardinality being bounded as follows
\begin{equation*}
    |\mathcal{S}|\leq \left(\dfrac{2d\overline{\lambda}}{\epsilon}\right)^{d(d+1)k/2} \times \left(\dfrac{2a}{\epsilon}\right)^{dk} \left(\dfrac{5}{\epsilon}\right)^{k}.
\end{equation*}
Denote by $\mathcal{G}$ the set of latent mixing measures $G$ in that net. Let $\mathcal{S}_0$ be an $\epsilon-$net in $[0,1]$ for $\lambda$, it is seen that $|\mathcal{S}_0|\leq 1/\epsilon$. Now we form a net for $\mathcal{P}_{k}(\Theta)$ by $\{p_{\lambda G} : \lambda \in \mathcal{S}_0, G\in \mathcal{G} \}$. Hence, for any $\lambda, G$, there exists $\tilde{\lambda}\in \mathcal{S}_0, G \in \mathcal{G}$ such that
\begin{equation*}
    |\lambda - \tilde{\lambda}| \leq \epsilon, \norm{p_{G} - p_{\tilde{G}}}_{\infty} \leq \epsilon.
\end{equation*}
This implies
\begin{align*}
    \norm{p_{\lambda G} - p_{\tilde{\lambda}\tilde{G}}}_{\infty}& \leq \norm{p_{\lambda G} - p_{\tilde{\lambda}{G}}}_{\infty} + \norm{p_{\tilde{\lambda}{G}} - p_{\tilde{\lambda}\tilde{G}}}_{\infty} \\
    &\leq |\lambda - \tilde{\lambda}|(\norm{h_0}_{\infty} + \norm{p_G}_{\infty}) + \tilde{\lambda} \norm{p_{G} - p_{\tilde{G}}}_{\infty}\\
    &\leq \epsilon \left(\norm{h_0}_{\infty} + \frac{1}{(\sqrt{2\pi}\underline{\lambda})^{d/2}}\right) + \epsilon\\
    & \lesssim \epsilon.
\end{align*}
Hence, we get an $\epsilon-$net for $\mathcal{P}_k(\Theta)$ with the cardinality less than or equal 
\begin{equation*}
    |\mathcal{S}_0|\times |\mathcal{S}| = \dfrac{1}{\epsilon}\times \left(\dfrac{2d\overline{\lambda}}{\epsilon}\right)^{d(d+1)k/2} \times \left(\dfrac{2a}{\epsilon}\right)^{dk} \left(\dfrac{5}{\epsilon}\right)^{k}.
\end{equation*}
Thus, 
\begin{equation*}
    \log N(\epsilon, \mathcal{P}_k(\Theta), \norm{\cdot}_{\infty}) \lesssim \log(1/\epsilon).
\end{equation*} 

\paragraph{Proof of claim~\eqref{eq:entropy-hellinger-bound}} Now, from the entropy number to get the bracketing number, we let $\eta \leq \epsilon$ which will be chosen later. Let $f_1, \dots, f_N$ be a $\eta$-net for $\Pcal_k(\Theta)$. We have
\begin{equation}\label{eq:bound-tail-construct-net}
    (x-\theta)^{T} \Sigma^{-1} (x-\theta) \geq \dfrac{\enorm{x-\theta}^2}{\bar{\lambda}}\geq \dfrac{\enorm{x}^2}{4\bar{\lambda}},\quad  \forall \, \norm{x} \geq 2\sqrt{d} a, (\theta, \Sigma) \in \Theta,
\end{equation}
Moreover, $h_0$ has an exponential tail $- \log h_0(x) \gtrsim \enorm{x}^{\beta}$ for some $\beta > 0$, and $\norm{h_0}_{\infty} < C$ for some constant $C$. 
Therefore, if we let $\beta' = \min\{\beta, 2\} > 0$ and $C'=\max\left\{C, \dfrac{1}{(2\pi)^{d/2} \underline{\lambda}^{d}}\right\}$, then 
\begin{equation}
    H(x) = \begin{cases}
    C_1\exp(-\enorm{x}^{\beta'}), & \quad \enorm{x} \geq B_1,\\
    C' , &\quad \text{otherwise}
    \end{cases}
\end{equation}
is an envelop for $\Pcal_k(\Theta)$, where $C_1$ depends only on $\underline{\lambda}$ and $h_0$, $B_1$ depends on $a, \overline{\lambda}, h_0$. We can construct brackets $[p_i^L, p_i^U]$ as follows. 
\begin{equation}
    p_i^L(x) = \max\{f_i(x) - \eta, 0 \}, p_i^U(x) = \min\{f_i(x) + \eta, H(x) \}.
\end{equation}
Because for each $f\in \Pcal_k(\Theta)$, there is $f_i$ such that $\norm{f - f_i}_{\infty} < \eta$, therefore $ p_i^{L}\leq f\leq p_i^{U}$. Moreover, for any $B\geq B_1$,
\begin{align}
    \int_{\mathbb{R}^{d}} (p_i^{U} - p_i^{L}) d\mu & \leq \int_{\enorm{x} \leq B} 2\eta dx + \int_{\enorm{x} \geq B} H(x) dx \nonumber \\
    &\lesssim \eta B^{d} + B^{d}\exp\left( - B^{\beta'}\right)\label{bracket-radi},
\end{align}
where we use spherical coordinate to have
\begin{equation*}
	 \int_{\norm{x} \leq B}  dx = \dfrac{\pi^{d/2}}{\Gamma(d/2 + 1)} B^d \lesssim B^d,
\end{equation*}
and
\begin{align*}
	\int_{\norm{x} \geq B}\exp\left(-\enorm{x}^{\beta'} \right) & \lesssim \int_{r \geq B} r^{d-1} \exp \left(-r^{\beta'} \right) dr \\
	& = \dfrac{1}{\beta}\int_{B^{\beta'}}^{\infty} u^{d/\beta' - 1} \exp(-u) du \quad (\text{change of variable }u = r^{\beta'})\\
	& \leq \dfrac{1}{\beta'} B^{d-\beta'} \exp(-B^{\beta'}),
\end{align*}
in which the last step we use the inequality (with change of variable formula)
\begin{equation}
	\int_{z}^{\infty} u^{d/\beta-1} e^{-u} du = z^{d/\beta} e^{-z}\int_{0}^{\infty} (1+s)^{d/\beta-1}e^{-zs}ds \leq z^{d/\beta} e^{-z} \dfrac{1}{z - d/\beta + 1} < z^{d/\beta} e^{-z},
\end{equation}
whenever $z > d/\beta'$, and we use $z = B^{\beta'}$. Hence, in \eqref{bracket-radi}, if we choose $B = B_1(\log(1/\eta))^{1/\beta'}$ then 
\begin{equation}\label{bracket-bound}
	\int_{\mathbb{R}^{d}} (p_i^{U} - p_i^{L}) d\mu\lesssim \eta \left(\log\left(\dfrac{1}{ \eta}\right)\right)^{d/\beta'}.
\end{equation}
Therefore, there exists a positive constant $c$ which does not depend on $\eta$ such that 
\begin{equation*}
    H_B(c\eta \log(1/\eta)^{d/\beta'}, \mathcal{P}_k(\Theta), \norm{\cdot}_1)
    \lesssim \log(1/\eta).
\end{equation*}
Let $\epsilon = c \eta (\log(1/\eta))^{d/\beta'}$, we have $\log(1/\epsilon) \asymp \log(1/\eta)$, which combines with inequality $\norm{\cdot}_1\leq h^2$ leads to 
\begin{equation*}
    H_B(\epsilon, \mathcal{P}_k(\Theta), h)\leq H_B(\epsilon^2, \mathcal{P}_k(\Theta), \norm{\cdot}_1)\lesssim \log(1/\epsilon^2) \lesssim \log(1/\epsilon).
\end{equation*}
Thus, we have proved claim~\eqref{eq:entropy-hellinger-bound}.

{  We put a remark here that the technique in this proof can be generalized for any family of $f(x|\theta)$ that have sub-exponential tails, i.e. $f(x|\theta) \lesssim \exp(-\|x\|^{\gamma})$ for all $x$ large enough and $\gamma>0$. We can substitute this condition into equation~\eqref{eq:bound-tail-construct-net}, then proceed to continue the proof similarly.}

Next, we provide proofs for inverse bounds in Section~\ref{sec:conv-rate-density-estimation} of the paper. Because there are several results with the same spirit in this section, to make it easy for reader, we recall each result before proving it.  
%%%%%%%%%%%%%%%%%%%%%%%%%%%%%%%%%%%%%%%%%%%%%%%%%%%%%%%%%%%%%%%%%%%%%%%%%%%%%%%%%%%%%
\subsection{Proof of Theorem~\ref{theorem:distinguish_exact_specified_point_wise}}
\label{subsec:proof:theorem:distinguish_exact_specified_point_wise}
%%%%%%%%%%%%%%%%%%%
%%%%%%%%%%%%%%%%%%%%%%%%%%%%%%%%%%%%%%%%%%%%%%%%%%%%%%%%%%%%%%%%%%
\begin{customthm}{\ref{theorem:distinguish_exact_specified_point_wise}}
Assume that $k_{*}$ is known, $f$ is first order identifiable and $(f, k_{*})$ is distinguishable from $h_{0}$. Then, for any $G \in \Ecal_{k_{*}}( \Theta)$, there exist positive constant $C_{1}$ and $C_{2}$ depending only on $\lambda^{*}, G_{*}, h_{0}, \Theta$ such that the following holds:
\begin{itemize}
\item[(a)] When $\lambda^{*} = 0$, then
	$V(p_{\lambda^{*} G_{*}}, p_{\lambda G}) \geq C_{1} \lambda$.

\item[(b)] When $\lambda^{*} \in (0, 1]$, then
\begin{align*}
	V(p_{\lambda^{*} G_{*}}, p_{\lambda G}) \geq C_{2} \underbrace{\left[ |\lambda-\lambda^{*}|+(\lambda+\lambda^{*}) W_{1}(G, G_{*}) \right]}_{\overline{W}_{1}(\lambda G, \lambda^{*} G_{*})}.
\end{align*}
\end{itemize}
\end{customthm}

We first provide the proof of the theorem for the setting $\lambda^{*} \in (0, 1]$ in Section~\ref{subsec:proof:theorem:distinguish_exact_specified_point_wise_nonzero}. Then, the proof for the setting $\lambda^{*} = 0$ is presented in Section~\ref{subsec:proof:theorem:distinguish_exact_specified_point_wise_zero}.
\subsubsection{Proof of setting $\lambda^{*} \in (0, 1]$}
\label{subsec:proof:theorem:distinguish_exact_specified_point_wise_nonzero}
Recall that, we define $\overline{W}_{1}(\lambda G, \lambda^{*} G_{*}) : = |\lambda-\lambda^{*}|+(\lambda+\lambda^{*}) W_{1}(G, G_{*})$. Besides that, $G_{*} = \sum_{i = 1}^{k_{*}} p_{i}^{*} \delta_{ \theta_{i}^{*}}$. In order to obtain the proof of the theorem for the setting $\lambda^{*} \in (0, 1]$, it is sufficient to verify the following two claims:
\begin{align}
	\lim \limits_{\epsilon \to 0} \inf \limits_{\lambda \in [0, 1], G \in \Ecal_{k_{*}}( \Theta)}{\left\{\dfrac{V(p_{\lambda G}, p_{\lambda^{*} G_{*}})}{\overline{W}_{1}(\lambda G, \lambda^{*} G_{*})}: \ \overline{W}_{1}(\lambda G, \lambda^{*} G_{*}) \leq \epsilon\right\}} > 0, \label{eq:claim_1_distin_exact__point_wise} \\
	\inf \limits_{\lambda \in [0, 1], G \in \Ecal_{k_{*} ( \Theta)}: \overline{W}_{1}(\lambda G, \lambda^{*} G_{*}) > \epsilon'} \dfrac{V(p_{\lambda G}, p_{\lambda^{*} G_{*}})}{\overline{W}_{1}(\lambda G, \lambda^{*} G_{*})} > 0, \label{eq:claim_2_distin_exact__point_wise}
\end{align}
for any $\epsilon' > 0$. 
\paragraph{Proof of claim~\eqref{eq:claim_1_distin_exact__point_wise}:}
Assume that claim~\eqref{eq:claim_1_distin_exact__point_wise} does not hold. It indicates that there exists a sequence of probability measures $G_{n} \in \Ecal_{k_{*}} ( \Theta)$ and a sequence of $\lambda_{n} \in [0, 1]$ such that $\overline{W}_{1}(\lambda_{n} G_{n}, \lambda^{*} G_{*}) \to 0$ and $V(p_{\lambda_{n} G_{n}}, p_{\lambda^{*} G_{*}})/ \overline{W}_{1}(\lambda_{n} G_{n}, \lambda^{*} G_{*}) \to 0$ as $n \to \infty$. Therefore, we have $\lambda_{n} \to \lambda^{*}$ and $W_{1}(G_{n}, G_{*}) \to 0$ as $n \to \infty$. We can relabel the atoms and weights of $G_{n}$ such that it admits the following form:
\begin{align}
	G_{n} = \sum_{i = 1}^{k_{*}} p_{i}^{n} \delta_{ \theta_{i}^{n}}, \label{eq:relabel_measure}
\end{align}
where $p_{i}^{n} \to p_{i}^{*}$ and $\theta_{i}^{n} \to \theta_{i}^{*}$ for all $i \in [k_{*}]$. To ease the ensuing presentation, we denote $\Delta \theta_{i}^{n} : = \theta_{i}^{n} - \theta_{i}^{*}$ and $\Delta p_{i}^{n} : = p_{i}^{n} - p_{i}^{*}$ for $i \in [k_{*}]$. Then, using the coupling between $G_n$ and $G_*$ such that it put mass $\min\{p_i^n, p_i^*\}$ on $\delta_{(\theta_i^n, \theta_i^*)}$, we can verify that 
\begin{align}
	W_{1}( G_{n}, G_{*}) \asymp \sum_{i = 1}^{k_{*}} \abss{ \Delta p_{i}^{n}} + p_{i}^{n} \enorm{ \Delta \theta_{i}^{n}}. \label{eq:wasserstein_equivalence}
\end{align} 
Our proof is divided into three steps.
\paragraph{Step 1 - Taylor expansion:} Invoking Taylor expansion up to the first order, we find that
\begin{align*}
	f(x| \theta_{i}^{n} ) = f(x| \theta_{i}^{*}) + (\Delta \theta_{i}^{n})^{ \top} \frac{\partial{f}}{\partial{ \theta}} (x| \theta_{i}^{*}) + R_{i}(x),
\end{align*}
where $R_{i}(x)$ is Taylor remainder such that $R_{i}(x) = o ( \enorm{ \Delta \theta_{i}^{n}})$ for $i \in [k_{*}]$. Given the above expressions, we obtain that
\begin{align}
	p_{\lambda_{n} G_{n}}(x) - p_{\lambda^{*} G_{*}}(x) = ( \lambda^{*} - \lambda_{n}) h_{0}(x) + \sum_{i = 1}^{k_{*}} \parenth{ \lambda_{n} p_{i}^{n} - \lambda^{*} p_{i}^{*}} f(x| \theta_{i}^{*}) & \nonumber \\
	& \hspace{- 5 em} + \lambda_{n} p_{i}^{n} \parenth{\Delta \theta_{i}^{n}}^{ \top} \frac{\partial{f}}{\partial{ \theta}} (x| \theta_{i}^{*}) + R(x), \label{eq:Taylor_distinguish_exact_specified_point_wise}
\end{align}
where $R(x) = \lambda_{n} \sum_{i = 1}^{n} p_{i}^{n} R_{i}(x) = o \parenth{\lambda_{n} \sum_{i = 1}^{k_{*}} p_{i}^{n} \enorm{ \Delta \theta_{i}^{n}}}$. From the expression of $W_{1}(G_{n}, G_{*})$ in~\eqref{eq:wasserstein_equivalence}, we have $R(x)/ \overline{W}_{1}(\lambda_{n} G_{n}, \lambda^{*} G_{*}) \to 0$ as $n \to \infty$ for all $x$.
\paragraph{Step 2 - Non-vanishing coefficients:} From equation~\eqref{eq:Taylor_distinguish_exact_specified_point_wise}, we can represent the ratio $\parenth{ p_{\lambda_{n} G_{n}}(x) - p_{\lambda^{*} G_{*}}(x)}/ \overline{W}_{1}(\lambda_{n} G_{n}, \lambda^{*} G_{*})$ as a linear combination of elements of $\sumfun(x)$, $f(x| \theta_{i}^{*})$, $\frac{\partial{f}}{\partial{ \theta}} (x| \theta_{i}^{*})$ for $i \in [k_{*}]$. Assume that all of the coefficients associated with these terms go to 0 as $n \to \infty$. As the coefficient with $\sumfun(x)$ goes to 0, we obtain that $( \lambda^{*} - \lambda_{n})/ \overline{W}_{1}(\lambda_{n} G_{n}, \lambda^{*} G_{*}) \to 0$ as $n \to \infty$. Furthermore, the coefficients of $f(x| \theta_{i}^{*})$, $\frac{\partial{f}}{\partial{ \theta}} (x| \theta_{i}^{*})$ vanish to 0 are equivalent to the following limits
\begin{align*}
	\parenth{ \lambda_{n} p_{i}^{n} - \lambda^{*} p_{i}^{*}}/ \overline{W}_{1}(\lambda_{n} G_{n}, \lambda^{*} G_{*}) \to 0, \ \ \
	p_{i}^{n} \enorm{ \Delta \theta_{i}^{n}}/  \overline{W}_{1}(\lambda_{n} G_{n}, \lambda^{*} G_{*}) \to 0. 
\end{align*}
As we have $( \lambda^{*} - \lambda_{n})/ \overline{W}_{1}(\lambda_{n} G_{n}, \lambda^{*} G_{*}) \to 0$, the above limits lead to
\begin{align*}
	\lambda^{*} \parenth{\Delta p_{i}^{n}}/ \overline{W}_{1}(\lambda_{n} G_{n}, \lambda^{*} G_{*}) \to 0.
\end{align*}
Putting the above results together, we obtain $1 = \overline{W}_{1}(\lambda_{n} G_{n}, \lambda^{*} G_{*})/ \overline{W}_{1}(\lambda_{n} G_{n}, \lambda^{*} G_{*}) \to 0$, which is a contraction. As a consequence, not all the coefficients of $\sumfun(x)$, $f(x| \theta_{i}^{*})$, $\frac{\partial{f}}{\partial{ \theta}} (x| \theta_{i}^{*})$ go to 0 for $i \in [k_{*}]$. 
\paragraph{Step 3: Show the contradiction using the distinguishability condition and Fatou's lemma:} Denote $m_{n}$ as the maximum of the absolute values of the coefficients of $\sumfun(x)$, $f(x| \theta_{i}^{*})$, $\frac{\partial{f}}{\partial{ \theta}} (x| \theta_{i}^{*})$ as $i \in [k_{*}]$. Since not all of these coefficients vanish to 0, we have $m_{n} \not \to 0$ as $n \to \infty$. Therefore, $d_{n} = 1/ m_{n} \not \to \infty$ as $n \to \infty$. Given the previous results, there exist $\alpha_{0}, \alpha_{1}, \ldots, \alpha_{k_{*}}$ and $\beta_{1}, \ldots, \beta_{k_{*}}$ such that not all of them are 0 and the following limit holds:
\begin{align*}
	d_{n} \cdot \frac{p_{\lambda_{n} G_{n}}(x) - p_{\lambda^{*} G_{*}}(x)}{\overline{W}_{1}(\lambda_{n} G_{n}, \lambda^{*} G_{*})} \to \alpha_{0} \sumfun(x) + \sum_{i = 1}^{k_{*}} \alpha_{i} f(x| \theta_{i}^{*}) + \beta_{i}^{\top} \frac{\partial{f}}{\partial{ \theta}} (x| \theta_{i}^{*}).
\end{align*}
By means of Fatou's lemma, we have
\begin{align}
	0 = \lim_{n \to \infty} d_{n} \cdot \frac{V( p_{\lambda_{n} G_{n}}, p_{\lambda^{*} G_{*}})}{\overline{W}_{1}(\lambda_{n} G_{n}, \lambda^{*} G_{*})} & \geq \int \mathop{ \lim \inf} \limits_{n \to \infty} d_{n} \cdot \frac{p_{\lambda_{n} G_{n}}(x) - p_{\lambda^{*} G_{*}}(x)}{\overline{W}_{1}(\lambda_{n} G_{n}, \lambda^{*} G_{*})} dx, \nonumber \\
	& = \int \parenth{ \alpha_{0} \sumfun(x) + \sum_{i = 1}^{k_{*}} \alpha_{i} f(x| \theta_{i}^{*}) + \beta_{i}^{\top} \frac{\partial{f}}{\partial{ \theta}} (x| \theta_{i}^{*})} dx.
\end{align}
The above equation indicates that 
\begin{align*}
	\alpha_{0} \sumfun(x) + \sum_{i = 1}^{k_{*}} \alpha_{i} f(x| \theta_{i}^{*}) + \beta_{i}^{\top} \frac{\partial{f}}{\partial{ \theta}} (x| \theta_{i}^{*}) = 0,
\end{align*}
for almost surely $x$. Since $(f, k_{*})$ is distinguishable from $h_{0}$ and $f$ is first order identifiable, the above equation suggests that $\alpha_{0} = \alpha_{1} = \ldots = \alpha_{k_{*}} = 0$ and $\beta_{1} = \ldots = \beta_{k_{*}} = \vec{0}$, which is a contradiction. 

As a consequence, we achieve the conclusion of claim~\eqref{eq:claim_1_distin_exact__point_wise}.
\paragraph{Proof of claim~\eqref{eq:claim_2_distin_exact__point_wise}}
Similar to the proof of claim~\eqref{eq:claim_1_distin_exact__point_wise}, we also prove claim~\eqref{eq:claim_2_distin_exact__point_wise} by contradiction. Assume that claim~\eqref{eq:claim_2_distin_exact__point_wise} does not hold. It implies that we can find sequences $\lambda_{n}' \in [0, 1]$ and $G_{n}' \in \Ecal_{k_{*}}( \Theta)$ such that $\overline{W}_{1}(\lambda_{n}' G_{n}', \lambda^{*} G_{*}) > \epsilon'$ and $V(p_{\lambda_{n}' G_{n}'}, p_{\lambda^{*} G_{*}})/ \overline{W}_{1}(\lambda_{n}' G_{n}', \lambda^{*} G_{*}) \to 0$ as $n \to \infty$. Since $[0, 1]$ and $\Theta$ are bounded sets, there exist $\lambda' \in [0, 1]$ and $G' \in \Ecal_{k_{*}}( \Theta)$ such that $\lambda_{n}' \to \lambda'$ and $W_{1}(G_{n}', G') \to 0$ as $n \to \infty$. Since $\overline{W}_{1}(\lambda_{n}' G_{n}', \lambda^{*} G_{*}) > \epsilon'$ for all $n$, the previous limits indicate that $\overline{W}_{1}(\lambda' G', \lambda^{*} G_{*}) \geq \epsilon'$. 

On the other hand, since $V(p_{\lambda_{n}' G_{n}'}, p_{\lambda^{*} G_{*}})/ \overline{W}_{1}(\lambda_{n}' G_{n}', \lambda^{*} G_{*}) \to 0$, we have $V(p_{\lambda_{n}' G_{n}'}, p_{\lambda^{*} G_{*}}) \to 0$ as $n \to \infty$. An application of Fatou's lemma leads to
\begin{align*}
	0 = \lim_{n \to \infty} V(p_{\lambda_{n}' G_{n}'}, p_{\lambda^{*} G_{*}}) & \geq \frac{1}{2} \int \mathop {\lim \inf} \limits_{n \to \infty} \abss{p_{\lambda_{n}' G_{n}'}(x) - p_{\lambda^{*} G_{*}}(x)} dx = V(p_{\lambda' G', \lambda^{*} G_{*}}).
\end{align*}
Due to the identifiability of model~\eqref{eq:true_model}, the above equation leads to $(\lambda', G') \equiv (\lambda^{*}, G_{*})$, which is a contradiction to the condition that $\overline{W}_{1}(\lambda' G', \lambda^{*} G_{*}) \geq \epsilon'$. As a consequence, we achieve the conclusion of claim~\eqref{eq:claim_2_distin_exact__point_wise}.
\subsubsection{Proof of setting $\lambda^{*} = 0$}
\label{subsec:proof:theorem:distinguish_exact_specified_point_wise_zero}
We want to show that
\begin{align}\label{eq:claim_1_distin_exact__point_wise_zero}
	\inf_{G\in \Ecal_{k_*}(\Theta)} \dfrac{V(p_{\lambda G}, p_{\lambda^*G_{*}})}{\lambda} > 0
\end{align}
\paragraph{Proof of claim~\eqref{eq:claim_1_distin_exact__point_wise_zero}:} Assume that claim~\eqref{eq:claim_1_distin_exact__point_wise_zero} does not hold. We can find two sequences $\bar{ \lambda}_{n} \in [0, 1]$ and $\bar{G}_{n} \in \Ecal_{k_{*}}( \Theta)$ such that $V(p_{ \bar{ \lambda}_{n} \bar{ G}_{n}}, p_{\lambda^{*} G_{*}})/ \bar{ \lambda}_{n} \to 0$ as $n \to \infty$. We denote $\bar{G}_{n} = \sum_{i = 1}^{k_{*}} \bar{p}_{i}^{n} \delta_{ \bar{ \theta}_{i}^{n}}$. Since $\Theta$ is a bounded set, there exists $\bar{G} = \sum_{i = 1}^{k_{*}} \bar{p}_{i} \delta_{ \bar{\theta}_{i}} \in \Ecal_{k_{*}}( \Theta)$ such that $W_{1}( \bar{G}_{n}, \bar{G}) \to 0$ as $n \to \infty$. Invoking Fatou's lemma, we obtain that
\begin{align*}
	0 = \lim_{n \to \infty} \frac{V( p_{ \bar{ \lambda}_{n} \bar{ G}_{n}}, p_{\lambda^{*} G_{*}})}{\lambda_{n}} & \geq \frac{1}{2} \int \mathop {\lim \inf} \limits_{n \to \infty} \abss{ \sum_{i = 1}^{k_{*}} \bar{p}_{i}^{n} f(x | \bar{ \theta}_{i}^{n}) - \sumfun(x)} dx \\
	& = V \parenth{ \sum_{i = 1}^{k_{*}} \bar{p}_{i} f(.| \bar{ \theta}_{i}), h_{0}(.) }.
\end{align*}
The above equation shows that $\sum_{i = 1}^{k_{*}} \bar{p}_{i} f(x| \bar{ \theta}_{i}) = h_{0}(x)$ for almost surely $x$, which is a contradiction to the hypothesis that $(f, k_{*})$ is distinguishable from $\sumfun$. Hence, we reach the conclusion of claim~\eqref{eq:claim_1_distin_exact__point_wise_zero}. 
%%%%%%%%%%%%%%%%%%%%%%%%%%%%%%%%%%%%%%%%%%%%%%%%%%%%%%%%%%%%%%%%%%%%%%%%%%%%%%%%%%%%%
\subsection{Proof of Theorem~\ref{theorem:distinguish_over_specified_point_wise}}
\label{subsec:proof:theorem:distinguish_over_specified_point_wise}
%%%%%%%%%%%%%%%%%%%
%%%%%%%%%%%%%%%%%%%%%%%%%%%%%%%%%%%%%%%%%%%%%%%%%%%%%%%%%%%%%%%%%%
\begin{customthm}{\ref{theorem:distinguish_over_specified_point_wise}}
Assume that $k_{*}$ is unknown and strictly upper bounded by a given $K$. Besides that, $f$ is second order identifiable and $(f, K)$ is distinguishable from $h_{0}$. Then, for any $G \in \Ocal_{K}( \Theta)$, there exist positive constant $C_{1}$ and $C_{2}$ depending only on $\lambda^{*}, G_{*}, h_{0}, \Theta$ such that the following holds:
\begin{itemize}
\item[(a)] When $\lambda^{*} = 0$, then
	$V(p_{\lambda^{*} G_{*}}, p_{\lambda G}) \geq C_{1} \lambda$.
\item[(b)] When $\lambda^{*} \in (0, 1]$, then
\begin{align*}
	V(p_{\lambda^{*} G_{*}}, p_{\lambda G}) \geq C_{2} \underbrace{\left[ |\lambda-\lambda^{*}|+(\lambda+\lambda^{*}) W_{2}^2(G, G_{*}) \right]}_{\overline{W}_{2}(\lambda G, \lambda^{*} G_{*})}.
\end{align*}
\end{itemize}
\end{customthm}

The proof argument for the setting $\lambda^{*} = 0$ is similar to  that in Section~\ref{subsec:proof:theorem:distinguish_exact_specified_point_wise_zero}; therefore, it is omitted. We focus only on the proof of the setting $\lambda^{*} \in (0, 1]$.

Similar to the proof of Theorem~\ref{theorem:distinguish_exact_specified_point_wise}, in order to reach the conclusion of Theorem~\ref{theorem:distinguish_over_specified_point_wise} for the setting $\lambda^{*} \in (0, 1]$, it is sufficient to demonstrate the following claims:
\begin{align}
	\lim \limits_{\epsilon \to 0} \inf \limits_{\lambda \in [0, 1], G \in \Ocal_{K}( \Theta)}{\left\{\dfrac{V(p_{\lambda G}, p_{\lambda^{*} G_{*}})}{\overline{W}_{2}(\lambda G, \lambda^{*} G_{*})}: \ \overline{W}_{2}(\lambda G, \lambda^{*} G_{*}) \leq \epsilon\right\}} > 0, \label{eq:claim_1_distin_over_point_wise} \\
	\inf \limits_{\lambda \in [0, 1], G \in \Ocal_{K ( \Theta)}: \overline{W}_{2}(\lambda G, \lambda^{*} G_{*}) > \epsilon'} \dfrac{V(p_{\lambda G}, p_{\lambda^{*} G_{*}})}{\overline{W}_{2}(\lambda G, \lambda^{*} G_{*})} > 0, \nonumber
\end{align}
for any $\epsilon' > 0$. Since the proof of the second claim is similar to that of claim~\eqref{eq:claim_2_distin_exact__point_wise} in Section~\ref{subsec:proof:theorem:distinguish_exact_specified_point_wise}; therefore, it is omitted. 
\paragraph{Proof of claim~\eqref{eq:claim_1_distin_over_point_wise}:}
Similar to the proof of claim~\eqref{eq:claim_1_distin_exact__point_wise}, we use proof by contradiction for claim~\eqref{eq:claim_1_distin_over_point_wise}. Assume that claim~\eqref{eq:claim_1_distin_over_point_wise} does not hold. Given that assumption, we can find sequences $G_{n} \in \Ocal_{K} ( \Theta)$ and $\lambda_{n} \in [0, 1]$ such that $\overline{W}_{2}(\lambda_{n} G_{n}, \lambda^{*} G_{*}) \to 0$ and $V(p_{\lambda_{n} G_{n}}, p_{\lambda^{*} G_{*}})/ \overline{W}_{2}(\lambda_{n} G_{n}, \lambda^{*} G_{*}) \to 0$ as $n \to \infty$. As $W_{2}(G_{n}, G_{*}) \to 0$ as $n \to \infty$, using the similar argument as that in Section 3.2 in Ho et al.~\cite{Ho-Nguyen-Ann-17}, we can find a subsequence of $G_{n}$ (without loss of generality, we replace that subsequence by the whole sequence of $G_{n}$ with $k' \in [k_{*}, K]$ supports such that 
\begin{align}
	G_{n} = \sum_{i = 1}^{k_{*} + \bar{l}} \sum_{j = 1}^{s_{i}} p_{ij}^{n} \delta_{ \theta_{ij}^{n}}, \label{eq:relabel_measure}
\end{align}
where $\sum_{j = 1}^{s_{i}} p_{ij}^{n} \to p_{i}^{*}$ and $\theta_{ij}^{n} \to \theta_{i}^{*}$ for all $i \in [k_{*}+ \bar{l}]$. Here, $p_{i}^{*} = 0$ for $k_{*} + 1 \leq i \leq k_{*} + \bar{l}$. In addition, $s_{1}, \ldots, s_{k_{*} + \bar{l}} \geq 1$ are such that $\sum_{i = 1}^{k_{*} + \bar{l}} s_{i} = k'$.  To ease the ensuing presentation, we denote $\Delta \theta_{ij}^{n} : = \theta_{ij}^{n} - \theta_{i}^{*}$ and $\Delta p_{i.}^{n} : = \sum_{j = 1}^{s_{i}} p_{ij}^{n} - p_{i}^{*}$ for $i \in [k_{*} + \bar{l}]$. Then, based on Lemma 3.1 in Ho et al.~\cite{Ho-Nguyen-Ann-17}, we have
\begin{align}
	W_{2}^2( G_{n}, G_{*}) \asymp \sum_{i = 1}^{k_{*} + \bar{l}} \abss{ \Delta p_{i.}^{n}} + \sum_{i = 1}^{k_{*} + \bar{l}} \sum_{j = 1}^{s_{i}}  p_{ij}^{n} \enorm{ \Delta \theta_{ij}^{n}}^2. \label{eq:wasserstein_equivalence_over}
\end{align}  
We divide our proof of claim~\eqref{eq:claim_1_distin_over_point_wise} into three steps. 
\paragraph{Step 1 - Taylor expansion:}  An application of Taylor expansion up to the second order leads to
\begin{align*}
	f(x| \theta_{ij}^{n}) = f(x| \theta_{i}^{*}) + (\Delta \theta_{ij})^{\top} \frac{\partial{f}}{\partial{\theta}}(x| \theta_{i}^{*}) + (\Delta \theta_{ij})^{\top} \frac{\partial^2{f}}{\partial{\theta^2}}(x| \theta_{i}^{*}) (\Delta \theta_{ij}) + R_{ij}(x),
\end{align*} 
where $R_{ij}(x)$ is Taylor remainder such that $R_{ij}(x) = \smallO (\enorm{\Delta \theta_{ij}}^2)$ for all $i \in [k_{*} + \bar{l}]$ and $j \in [s_{i}]$. Collecting the above equations, we obtain that
\begin{align}
	p_{\lambda_{n} G_{n}}(x) - p_{\lambda^{*} G_{*}}(x) = ( \lambda^{*} - \lambda_{n}) h_{0}(x) + \sum_{i = 1}^{k_{*} + \bar{l}} \parenth{ \sum_{j = 1}^{s_{i}} \lambda_{n} p_{ij}^{n} - \lambda^{*} p_{i}^{*}} f(x| \theta_{i}^{*}) & \nonumber \\
	& \hspace{- 30 em} + \lambda_{n} \parenth{ \sum_{j = 1}^{s_{i}} p_{ij}^{n} \Delta \theta_{ij}^{n}}^{ \top} \frac{\partial{f}}{\partial{ \theta}} (x| \theta_{i}^{*}) + \lambda_{n} \parenth{ \sum_{j = 1}^{s_{i}} p_{ij}^{n} \parenth{ \Delta \theta_{ij}^{n}}^{\top} \frac{\partial^2{f}}{\partial{\theta^2}}(x| \theta_{i}^{*}) (\Delta \theta_{ij}^{n})} + R(x), \label{eq:Taylor_distinguish_over_specified_point_wise}
\end{align}
where $R(x) = \lambda_{n} \sum_{i = 1}^{k_{*} + \bar{l}} \sum_{j = 1}^{s_{i}} p_{ij}^{n} R_{ij}(x) = \smallO \parenth{ \lambda_{n} \sum_{i = 1}^{k_{*} + \bar{l}} \sum_{j = 1}^{s_{i}} p_{ij}^{n} \enorm{ \Delta \theta_{ij}^{n}}^2}$. Given the expression of $W_{2}^2(G_{n}, G_{*})$ in equation~\eqref{eq:wasserstein_equivalence_over}, we can verify that $R(x)/ \overline{W}_{2}(\lambda_{n} G_{n}, \lambda^{*} G_{*}) \to 0$ as $n \to \infty$. 
\paragraph{Step 2 - Non-vanishing coefficients:} Given the expression in equation~\eqref{eq:Taylor_distinguish_over_specified_point_wise}, we can view $(p_{\lambda_{n} G_{n}}(x) - p_{\lambda^{*} G_{*}}(x))/ \overline{W}_{2}(\lambda_{n} G_{n}, \lambda^{*} G_{*})$ as a linear combination of elements of the forms $h_{0}(x), f(x| \theta_{i}^{*}), \frac{\partial{f}}{\partial{ \theta}} (x| \theta_{i}^{*})$, and $\frac{\partial^2{f}}{\partial{\theta^2}}(x| \theta_{i}^{*})$ for all $i \in [k_{*} + \bar{l}]$. Assume that their coefficients go to 0 as $n$ tends to infinity. As the coefficient of $h_{0}(x)$ goes to 0, we have $(\lambda_{n} - \lambda^{*})/ \overline{W}_{2}(\lambda_{n} G_{n}, \lambda^{*} G_{*}) \to 0$. 

Similarly, by learning the coefficients of $f(x| \theta_{i}^{*})$ and $\brackets{ \frac{\partial^2{f}}{\partial{\theta^2}}(x| \theta_{i}^{*})}_{jj}$ for $j \in [d]$, we obtain the following limits:
\begin{align*}
	\parenth{ \sum_{j = 1}^{s_{i}} \lambda_{n} p_{ij}^{n} - \lambda^{*} p_{i}^{*}} / \overline{W}_{2}(\lambda_{n} G_{n}, \lambda^{*} G_{*}) \to 0, \quad \lambda_{n} \parenth{ \sum_{j = 1}^{s_{i}} p_{ij}^{n} \enorm{ \Delta \theta_{ij}^{n} }^2}/ \overline{W}_{2}(\lambda_{n} G_{n}, \lambda^{*} G_{*}) \to 0.
\end{align*}
Collecting the above limits, we find that
\begin{align*}
	\frac{\lambda^{*} \Delta p_{i.}^{n}}{\overline{W}_{2}(\lambda_{n} G_{n}, \lambda^{*} G_{*})} = \frac{ ( \lambda^{*} - \lambda_{n}) \parenth{ \sum_{j = 1}^{s_{i}} p_{ij}^{n}} + \parenth{ \sum_{j = 1}^{s_{i}} \lambda_{n} p_{ij}^{n} - \lambda^{*} p_{i}^{*}}}{\overline{W}_{2}(\lambda_{n} G_{n}, \lambda^{*} G_{*})} \to 0. 
\end{align*}
Putting the above results together, we achieve that $1 = \overline{W}_{2}(\lambda_{n} G_{n}, \lambda^{*} G_{*})/ \overline{W}_{2}(\lambda_{n} G_{n}, \lambda^{*} G_{*}) \to 0$, which is a contraction. Therefore, not all the coefficients associated with $h_{0}(x), f(x| \theta_{i}^{*}), \frac{\partial{f}}{\partial{ \theta}} (x| \theta_{i}^{*})$, and $\frac{\partial^2{f}}{\partial{\theta^2}}(x| \theta_{i}^{*})$ for $i \in [k_{*} + \bar{l}]$ go to 0 as $n$ tends to infinity. 
\paragraph{Step 3: Show the contradiction using the distinguishability condition and Fatou's lemma:} Similar to Step 3 in Section~\ref{subsec:proof:theorem:distinguish_exact_specified_point_wise_nonzero}, by denoting $d_{n} = 1/ m_{n}$ where $m_{n}$ is the maximum values of the absolute values of the coefficients of $h_{0}(x), f(x| \theta_{i}^{*}), \frac{\partial{f}}{\partial{ \theta}} (x| \theta_{i}^{*})$, and $\frac{\partial^2{f}}{\partial{\theta^2}}(x| \theta_{i}^{*})$, we have
\begin{align*}
	d_{n} \cdot \frac{p_{\lambda_{n} G_{n}}(x) - p_{\lambda^{*} G_{*}}(x)}{\overline{W}_{1}(\lambda_{n} G_{n}, \lambda^{*} G_{*})} \to \alpha_{0} \sumfun(x) + \sum_{i = 1}^{k_{*} + \bar{l}} \alpha_{i} f(x| \theta_{i}^{*}) + \beta_{i}^{\top} \frac{\partial{f}}{\partial{ \theta}} (x| \theta_{i}^{*}) + \gamma_{i}^{\top} \frac{\partial^2{f}}{\partial{ \theta^2}} (x| \theta_{i}^{*}) \gamma_{i},
\end{align*}
where $\alpha_{i}, \beta_{i}, \gamma_{i}$ are some coefficients such that not all of them are 0. However, the Fatou's lemma suggests that the RHS of the above equation is 0 for almost surely $x$. Since $(f, K)$ is distinguishable from $h_{0}$, it shows that $\alpha_{i} = 0$, $\beta_{i} = \vec{0} \in \Rspace^{d}$, and $\gamma_{i} = \vec{0} \in \Rspace^{d \times d}$ for all $i \in [k_{*} + \bar{l}]$--- a contradiction. As a consequence, we obtain the conclusion of claim~\eqref{eq:claim_1_distin_over_point_wise}.
%%%%%%%%%%%%%%%%%%%%%%%%%%%%%%%%%%%%%%%%%%%%%%%%%%%%%%%%%%%%%%%%%%%%%%%%%%%%%%%%%%%%%
\subsection{Proof of Theorem~\ref{theorem:distinguish_over_specified_weakly_ident_point_wise}}\label{subsec:proof:theorem:distinguish_over_specified_weakly_ident_point_wise}
\begin{customthm}{\ref{theorem:distinguish_over_specified_weakly_ident_point_wise}}
Assume that $k_{*}$ is unknown and strictly upper bounded by a given $K$. Besides that, $f$ is location-scale Gaussian distribution and $(f, K)$ with fixed variance is distinguishable in any order from $h_{0}$. Then, for any $G \in \Ocal_{K}( \Theta)$, there exist positive constant $C_{1}$ and $C_{2}$ depending only on $\lambda^{*}, G_{*}, h_{0}, \Theta$ such that the following holds:
\begin{itemize}
\item[(a)] When $\lambda^{*} = 0$, then
	$V(p_{\lambda^{*} G_{*}}, p_{\lambda G}) \geq C_{1} \lambda$.

\item[(b)] When $\lambda^{*} \in (0, 1]$, then
\begin{align*}
	V(p_{\lambda^{*} G_{*}}, p_{\lambda G}) \geq C_{2} \overline{W}_{\overline{r}(K-k_*)}(\lambda G, \lambda^{*} G_{*}).
\end{align*}
\end{itemize}
\end{customthm}
The proof argument for the setting $\lambda^{*} = 0$ is similar to  that in Section~\ref{subsec:proof:theorem:distinguish_exact_specified_point_wise_zero}; therefore, it is omitted. We focus only on the proof of the setting $\lambda^{*} \in (0, 1]$.

Denote by $\overline{r}_1 = \overline{r}(K-k_*) $. Similar to the proof of Theorem~\ref{theorem:distinguish_exact_specified_point_wise}, in order to reach the conclusion of Theorem~\ref{theorem:distinguish_over_specified_weakly_ident_point_wise} for the setting $\lambda^{*} \in (0, 1]$, it is sufficient to demonstrate the following claims:
\begin{align}
	\lim \limits_{\epsilon \to 0} \inf \limits_{\lambda \in [0, 1], G \in \Ocal_{K}( \Theta)}{\left\{\dfrac{V(p_{\lambda G}, p_{\lambda^{*} G_{*}})}{\overline{W}_{\overline{r}_1}(\lambda G, \lambda^{*} G_{*})}: \ \overline{W}_{\overline{r}_1}(\lambda G, \lambda^{*} G_{*}) \leq \epsilon\right\}} > 0, \label{eq:claim_distin_weak_point_wise} \\
	\inf \limits_{\lambda \in [0, 1], G \in \Ocal_{K ( \Theta)}: \overline{W}_{\overline{r}_1}(\lambda G, \lambda^{*} G_{*}) > \epsilon'} \dfrac{V(p_{\lambda G}, p_{\lambda^{*} G_{*}})}{\overline{W}_{\overline{r}_1}(\lambda G, \lambda^{*} G_{*})} > 0, \nonumber
\end{align}
for any $\epsilon' > 0$. Since the proof of the second claim is similar to that of claim~\eqref{eq:claim_2_distin_exact__point_wise} in Section~\ref{subsec:proof:theorem:distinguish_exact_specified_point_wise}; therefore, it is omitted. We now proceed to prove claim~\eqref{eq:claim_distin_weak_point_wise}. Suppose that it is not correct, that is, there exist sequences $\lambda_n$ and $G_n = \sum_{i=1}^{k_n} p_i^n \delta_{\theta_i^n} \in \Ocal_K(\Theta)$ such that $\overline{W}_{\overline{r}_1}(\lambda_n G_n, \lambda^{*} G_{*}) \to 0$ and $V(p_{\lambda_n G_n}, p_{\lambda^* G_{*}})/\overline{W}_{\overline{r}_1}(\lambda_n G_n, \lambda^{*} G_{*}) \to 0$. For the ease of presentation, we consider the one dimension Gaussian case where $(\mu, \Sigma) = (\theta, v)$, the higher dimension cases are treated similar. 

We can use the subsequence argument to have $\lambda^* \geq \lambda_n$ for all $n$ and $G_n$ can be assumed to have a fixed number of atoms $k'$ (less than or equals $K$) and have a representation as in \eqref{eq:relabel_measure}, that is,
\begin{align}
	G_{n} = \sum_{i = 1}^{k_{*} + \bar{l}} \sum_{j = 1}^{s_{i}} p_{ij}^{n} \delta_{ (\theta_{ij}^{n}, v_{ij}^{n})}, \label{eq:relabel_measure}
\end{align}
where $\sum_{j = 1}^{s_{i}} p_{ij}^{n} \to p_{i}^{*}$ and $\theta_{ij}^{n} \to \theta_{i}^{*}, v_{ij}^{n} \to v_{i}^{*}$ for all $i \in [k_{*}+ \bar{l}]$. Here, $p_{i}^{*} = 0$ for $k_{*} + 1 \leq i \leq k_{*} + \bar{l}$. In addition, $s_{1}, \ldots, s_{k_{*} + \bar{l}} \geq 1$ are such that $\sum_{i = 1}^{k_{*} + \bar{l}} s_{i} = k'$.

\paragraph{Step 1 - Taylor expansion:} Using Taylor expansion of $f$ around $\{(\theta_i^*, v_i^*)\}_{i=1}^{k_*}$ to the $\overline{r}_1-$th order we have
\begin{align*}
    p_{\lambda_n G_n}(x) - p_{\lambda^* G_*}(x) & = (\lambda^* - \lambda_n)   h_0(x) + \lambda_n (\sum_{i=1}^{k_* + \underline{l}} \sum_{j=1}^{s_i} p_{ij}^{n} f(x|\theta_{ij}^{n}, v_{ij}^{n})) - \sum_{i=1}^{k_*} {p}_i^* f(x|\theta_i^*, v_i^*)\\
    & = (\lambda^* - \lambda_n)   h_0(x) + \sum_{i=1}^{k_* + \underline{l}} \sum_{j=1}^{s_i} \lambda_n p_{ij}^{n} \sum_{|\boldsymbol{\alpha}|=1}^{\overline{r}_1} (\Delta\theta_{ij}^{n})^{\alpha_1} (\Delta v_{ij}^{n})^{\alpha_2} \dfrac{1}{\boldsymbol{\alpha}!} \dfrac{\partial^{|\boldsymbol{\alpha}| f(\theta_i^*, v_i^*)}}{\partial^{\alpha_1}\theta \partial^{\alpha_2} v}\\
    & + \sum_{i=1}^{k_*+\underline{l}} (\Delta{p}_{i \cdot}^n) f(x|\theta_i^*, v_i^*) + R(x),
\end{align*}
where $\boldsymbol{\alpha} = (\alpha_1, \alpha_2)$, $|\boldsymbol{\alpha}| = \alpha_1 + \alpha_2, \boldsymbol{\alpha}! = \alpha_1!\alpha_2!$, $\Delta \overline{p}^{n}_{i\cdot} = \lambda_n \sum_{j}p_{ij}^n - {p}_i^*$, $\Delta \theta_{ij}^{n} = \theta_{ij}^n - \theta_i^*, \Delta v_{ij}^{n} = v_{ij}^n - v_i^*$ and $R(x) = o(\sum_{i=1}^{k_*+\underline{l}} \sum_{j=1}^{s_i} p_{ij}^{n} (|\Delta \theta_{ij}^n|^{\overline{r}_1}+ |\Delta v_{ij}^n|^{\overline{r}_1}))$. Now we can use the character equation $\dfrac{\partial^2 f}{\partial \theta^2} = 2 \dfrac{\partial f}{\partial v}$ to rewrite the formula above as
\begin{eqnarray}
    (\lambda^* - \lambda_n)h_0(x) + \sum_{\alpha=1}^{2\overline{r}_1} \sum_{i=1}^{k_* + \underline{l}} \left( \sum_{j=1}^{s_i} \lambda_n p_{ij}^{n}  \sum_{n_1, n_2}  \dfrac{(\Delta\theta_{ij}^{n})^{n_1} (\Delta v_{ij}^{n})^{n_2}}{2^{n_2}n_1!n_2!}\right) \dfrac{\partial^{\alpha} f(\theta_i^*, v_i^*)}{\partial \theta^{\alpha}} \nonumber \\ 
    + \sum_{i=1}^{k_*+\underline{l}} (\Delta{p}_{i \cdot}^n) f(x|\theta_i^*, v_i^*) + R(x),\label{eq:location-scale-Gauss-taylor-disting}
\end{eqnarray}
where we sum over $n_1, n_2$ such that $n_1+2n_2=\alpha, n_1+n_2\leq \overline{r}_1$.

\paragraph{Step 2 - Non-vanishing coefficients:} Assume that all coefficients in the formula above vanish when dividing by $W_{\overline{r}_1}^{\overline{r}_1}(\lambda_n G_n, \lambda^* G_*)$  when $n\to \infty$. Because 
\begin{equation}
    	W_{\overline{r}_1}^{\overline{r}_1}( \lambda_nG_{n}, \lambda^*G_{*}) \asymp |\lambda_n - \lambda^*| + (\lambda_n + \lambda^*) \left(\sum_{i = 1}^{k_{*} + \bar{l}} \abss{ \Delta p_{i.}^{n}} + \sum_{i = 1}^{k_{*} + \bar{l}} \sum_{j = 1}^{s_{i}}  p_{ij}^{n} (\enorm{ \Delta \theta_{ij}^{n}}^{\overline{r}_1} + \enorm{ \Delta v_{ij}^{n}}^{\overline{r}_1}) \right):=D_{\overline{r}_1}(G_n, G_*),
\end{equation}
we have
\begin{equation}\label{eq:thm-3-3-contradiction-cond}
    \dfrac{\lambda^* -\lambda_n}{D_{\overline{r}_1}(G_n, G_*)} \to 0, \,\,\, \dfrac{\Delta p_{i \cdot}^n}{D_{\overline{r}_1}(G_n, G_*)} \to 0.
\end{equation}
These limits together imply
\begin{equation*}
    \dfrac{(\lambda^* + \lambda_n) \Delta p_{i \cdot}^n}{D_{\overline{r}_1}(G_n, G_*)} \to 0, \quad \forall i = 1, \dots, k_* + \overline{l}.
\end{equation*}
From the definition of $D_{\overline{r}_1}$, it can be deduced that there exists at least an index $i^*$ such that
\begin{equation*}
    \sum_{j=1}^{s_{i*}} \dfrac{(\lambda_n + \lambda^*)  p_{i^* j}^{n} ((\theta_{ij}^{n})^{\overline{r}_1}+ (v_{ij}^{n})^{\overline{r}_1})}{D_{\overline{r}_1}(G_n, G_*)} \not \to 0.
\end{equation*}
Without loss of generality, assign $i^*=1$. But as we assume all the coefficients in equation \eqref{eq:location-scale-Gauss-taylor-disting} go to 0 for all $\alpha$ and $i$, we have
\begin{equation*}
    \dfrac{\sum\limits_{j=1}^{s_1} \lambda_n p_{1j}^{n} \sum\limits_{\substack{n_1+2 n_2=\alpha\\ n_1+n_2\leq \overline{r}_1}}  \dfrac{(\theta_{1j}^{n})^{n_1} (v_{1j}^{n})^{n_2}}{2^{n_2}n_1!n_2!}}{D_{\overline{r}_1}(G_n, G_*)} \to 0,
\end{equation*}
for all $\alpha = 1, \dots, 2\overline{r}_1$. From two expressions above combining with equation \eqref{eq:thm-3-3-contradiction-cond}, we have for all $\alpha= 1, \dots, 2\overline{r}_1$, 
\begin{equation}\label{eq:weak-ident-F-alpha-disting}
    F_{\alpha} := \dfrac{\sum\limits_{j=1}^{s_1} p_{1j}^{n} \sum\limits_{\substack{n_1+2 n_2=\alpha\\ n_1+n_2\leq \overline{r}_1}} \dfrac{(\Delta\theta_{1j}^{n})^{n_1} (\Delta v_{1j}^{n})^{n_2}}{2^{n_2}n_1!n_2!}}{\sum_{j=1}^{s_1} p_{1 j}^{n} ((\Delta\theta_{ij}^{n})^{\overline{r}_1} + (\Delta v_{ij}^{n})^{\overline{r}_1})}\to 0. 
\end{equation}
If $s_1 = 1$ then substituting $\alpha = 1$ and $\alpha = 2\overline{r}_1$ gives 
\begin{equation*}
    \dfrac{|\Delta \theta_{11}^{n}|^{\overline{r}_1}}{|\Delta \theta_{11}^{n}|^{\overline{r}_1}+ |\Delta v_{11}^{n}|^{\overline{r}_1}}, \dfrac{|\Delta v_{11}^{n}|^{\overline{r}_1}}{|\Delta \theta_{11}^{n}|^{\overline{r}_1}+ |\Delta v_{11}^{n}|^{\overline{r}_1}} \to 0,
\end{equation*}
which is impossible as they are sum up to 1 for all $n$. Hence $s_1 \geq 2$. Now we proceed to show the contradiction using the system of equations \eqref{eqn:system_polynomial_Gaussian_first}. Denote by $\overline{p}_n = \max_{1\leq j \leq s_1} \{p_{1j}^{n} \}, \overline{M}_n = \max_{1\leq j \leq s_1} \{|\Delta \theta_{1j}^{n}|, |\Delta v_{1j}^{n}|^{1/2} \}$. By the subsequence argument in compact sets, without loss of generality, we can denote $c_j^2 := \lim_{n\to \infty} p_{1j}^{n}/\overline{p}_n$, $a_j = \lim \Delta \theta_{1j}^{n} / \overline{M}_n$, and $b_j = \lim \Delta v_{1j}^{n} / \overline{M}_n$ for all $j = 1,\dots, k_* + \overline{l}$. Because of the definition of $\Ocal_{K, c_0}$, we have $p_{ij}^{n} \geq c_0$ for all $j$, which implies all $c_j$ are different from 0 and at least one of them is $1$. Similarly, in $(a_j, b_j)_{j}$, there is at least one of them equals to $1$ or $-1$. Dividing both numerators and denominators of equation \eqref{eq:weak-ident-F-alpha-disting} by $\overline{p}_n \overline{M}_n^{\alpha}$, we have
\begin{equation*}
    \sum\limits_{j=1}^{s_1}\sum\limits_{\substack{n_1+2 n_2=\alpha}} \dfrac{c_j^2 a_j^{n_1} b_j^{n_2}}{n_1! n_2!} = 0,  
\end{equation*}
for all $\alpha = 1, \dots, \overline{r}_1$. Hence, we get the contradiction, where we use the fact that $s_1\leq K - k_* + 1$ (as $s_i\geq 1$ for all $i \geq 2$) and  $\overline{r}_1 = \overline{r}(K-k_*)$ is the smallest number such that equation \eqref{eqn:system_polynomial_Gaussian_first}, where $k = K-k_*,$ has the trivial solution only. Hence, when dividing by $W_{\overline{r}_1}^{\overline{r}_1}(\lambda_n G_n, \lambda^* G_*)$, not all coefficients of equation~\eqref{eq:location-scale-Gauss-taylor-disting} vanish as $n\to \infty$. 

\paragraph{Step 3: Show the contradiction using the distinguishability condition and Fatou's lemma:}
Denote by 
\begin{equation*}
    E_{i, \alpha} = \sum_{j=1}^{s_i} \lambda_n p_{ij}^{n}  \sum_{n_1, n_2}  \dfrac{(\Delta\theta_{ij}^{n})^{n_1} (\Delta v_{ij}^{n})^{n_2}}{2^{n_2}n_1!n_2!} \bigg/ W_{\overline{r}_1}^{\overline{r}_1}(\lambda_n G_n, \lambda^* G_*), \quad \forall i, \alpha\geq 1.
\end{equation*} 
\begin{equation*}
E_{i, 0} = \Delta p_{i \cdot}^{n} \bigg/ W_{\overline{r}_1}^{\overline{r}_1}(\lambda_n G_n, \lambda^* G_*), \forall i\geq 1, E_{0, 0} = (\lambda^* - \lambda_n) \bigg/ W_{\overline{r}_1}^{\overline{r}_1}(\lambda_n G_n, \lambda^* G_*).
\end{equation*}
We have proved that not all $E_{i,\alpha}$ go to 0. Let $d_n = \max_{0\leq \alpha\leq 2\overline{r}_1, 0\leq i\leq k'} |E_{i,\alpha}|$. Because $E_{i,\alpha} / d_n \in [-1,1]$ for all $n$, by the subsequence argument if needed, we have $E_{i,\alpha} / m_n\to \beta_{i,\alpha}$ as $n\to \infty$, where at least one of the limits are different from 0. But Fatou's argument implies that 
\begin{equation*}
    \beta_{0, 0} h_0(x) + \sum_{i=1}^{k_*} \sum_{\alpha=0}^{2\overline{r}_1} \beta_{i,\alpha} \dfrac{\partial^{\alpha} f}{\partial \theta^{\alpha}} (x | \theta_i^{*}, v_i^{*}) = 0,
\end{equation*}
which contradicts our assumption. Hence, claim~\eqref{eq:claim_distin_weak_point_wise} is proved.
\subsection{Proof Theorem~\ref{theorem:pardistinguish_point_wise_zerolambda}}\label{subsec:proof:theorem:pardistinguish_point_wise_zerolambda}
\begin{customthm}
{\ref{theorem:pardistinguish_point_wise_zerolambda}}
Assume that $\sumfun$ takes the form~\eqref{eq:par_distin_h0} and $\lambda^{*} = 0$. Then, there exist positive constants $C_{1}$ and $C_{2}$ depending only on $\sumfun, \Theta$ such that the following holds:
\begin{itemize}
\item[(a)] (exact-fitted) If $f$ is first order identifiable, then for any $G\in \Ecal_{k_0}(\Theta)$
\begin{align*}
V(p_{\lambda^{*}, G_{*}}, p_{\lambda, G}) \geq C_{1} \lambda W_{1}( G, G_{0}),
\end{align*}
\item[(b)] (over-fitted) If $f$ is second order identifiable, then for any $G\in \Ocal_{K}(\Theta)$ that $K > k_0$
\begin{align*}
V(p_{\lambda^{*}, G_{*}}, p_{\lambda, G}) \geq C_{2} \lambda W_{2}^2(G, G_{0}),
\end{align*}
\item[(c)] (over-fitted and weakly identifiable) If $f$ is location-scale Gaussian distribution and we further assume that $G_* \in \Ecal_{k_*, c_0}(\Theta)$, then for any $G\in \Ocal_{K, c_0}(\Theta)$ that $K > k_0$, there exists $C_3$ depends on $h_0, \Theta_0, c_0$ such that
\begin{align*}
V(p_{\lambda^{*}, G_{*}}, p_{\lambda, G}) \geq C_{3} \lambda W_{\overline{r}(K-k_*)}^{\overline{r}(K-k_*)}(G, G_{0})
\end{align*}
\end{itemize}
\end{customthm}

\begin{enumerate}
\item[(a)] We can write 
\begin{align*}
    \dfrac{V(p_0, p_{\lambda G})}{\lambda W_1(G, G_0)} & = \int \dfrac{|\sum_{i=1}^{k_0} p_i^0 f(x|\theta_i^0) - \sum_{i=1}^{k_0} p_i f(x|\theta_i)|}{W_1(G, G_0)} dx \\
    & = \dfrac{V(p_0, p_{G})}{W_1(G, G_0)},
\end{align*}
because this is the exact-fitted and first-order identifiable, we can apply Theorem 3.1. in Ho et al.~\cite{Ho-Nguyen-EJS-16}
\item[(b)] Similar to the last part, we can write 
\begin{align*}
    \dfrac{V(p_0, p_{\lambda G})}{\lambda W_2^2(G, G_0)}& = \int \dfrac{|\sum_{i=1}^{k_0} p_i^0 f(x|\theta_i^0) - \sum_{i=1}^{K} p_i f(x|\theta_i)|}{W_2^2(G, G_0)} dx \\
    & = \dfrac{V(p_0, p_{G})}{W_2^2(G, G_0)},
\end{align*}
as this is the over-fitted and second-order identifiable, we can apply Theorem 3.2. in Ho et al.~\cite{Ho-Nguyen-EJS-16}.
\item[(c)] Similar to last two cases, we can write
\begin{align*}
    \dfrac{V(p_0, p_{\lambda G})}{\lambda W_{\overline{r}(K-k_*)}^{\overline{r}(K-k_*)}(G, G_0)}= \dfrac{V(p_0, p_{G})}{W_{\overline{r}(K-k_*)}^{\overline{r}(K-k_*)}(G, G_0)},
\end{align*}
and apply Proposition 2.2. in~\cite{Ho-Nguyen-Ann-16}.
\end{enumerate}

\subsection{Proof of Theorem~\ref{theorem:pardistinguish_point_wise_overspec_strong_iden}}
\label{subsec:proof:theorem:pardistinguish_point_wise_overspec_strong_iden}
%%%%%%%%%%%%%%%%%%%
%%%%%%%%%%%%%%%%%%%%%%%%%%%%%%%%%%%%%%%%%%%%%%%%%%%%%%%%%%%%%%%%%%
\begin{customthm}
{\ref{theorem:pardistinguish_point_wise_overspec_strong_iden}}
Assume that $\sumfun$ takes the form~\eqref{eq:par_distin_h0}. Besides that, $K \geq k_{0}$ and $f$ is location-scale Gaussian distribution. Then, for any $\lambda \in [0, 1]$ and $G \in \Ocal_{K, c_{0}}( \Theta)$ for some $c_{0} > 0$, there exist positive constants $C_{1}, C_{2}, C_3, C_4$ depending only on $\lambda^{*}, G_{*}, G_{0}, \Theta$ ($C_3$ and $C_4$ also depends on $\delta$) such that the following holds:
\begin{itemize}
\item[(a)] When $K \leq k_{*} + k_{0} - \bar{k} - 1$, then
\begin{align*}
V(p_{\lambda^{*}, G_{*}}, p_{\lambda, G}) & \geq  C_1 \overline{W}_{\overline{r}(K - k_{*} )}(\lambda G, \lambda^{*} G_{*}).
\end{align*}
\item[(b)] When $K \geq k_{*} + k_{0} - \bar{k} $, then
\begin{align*}
V(p_{\lambda^{*}, G_{*}}, p_{\lambda, G})  \geq C_{2} \biggr( 1_{\{\lambda \leq \lambda^{*}\}} \overline{W}_{\overline{r}(K - k_{*} )}(\lambda G, \lambda^{*} G_{*}) \\
+ 1_{\{\lambda > \lambda^{*}\}}  W_{\overline{r}(K - k_{*} )}^{\overline{r}(K - k_{*})}(G, \overline{G}_{*}(\lambda))\biggr)
\end{align*}
\item[(c)] For $\delta > 0$, when $K = k_{*} + k_{0} - \bar{k}$, we have 
\begin{align*}
 V(p_{\lambda^{*}, G_{*}}, p_{\lambda, G}) & \geq C_{3}  1_{\{\lambda > \lambda^{*} + \delta\}}  W_{1}(G, \overline{G}_{*}(\lambda)),
\end{align*}
and when $K > k_{*} + k_{0} - \bar{k}$, we have 
\begin{align*}
 V(p_{\lambda^{*}, G_{*}}, p_{\lambda, G}) \geq & C_{4}  1_{\{\lambda > \lambda^{*} + \delta\}} \\ 
 & \times W_{\overline{r}(K - k_{0} - k_{*} + \bar{k} )}^{\overline{r}(K - k_{0}\ - k_{*} + \bar{k} )}(G, \overline{G}_{*}(\lambda)).
\end{align*}
\end{itemize}
\end{customthm}

To facilitate the proof argument, we denote $\diff : = k_{*} + k_{0} - \bar{k}$. In addition, we assume without loss of generality that $\theta_{i}^{*} = \theta_{i}^{0}$ for $i \in [\bar{k}]$. Moreover, we introduce the following shorthand:
\begin{align*}
	D(\lambda G, \lambda^{*} G_{*}) = \begin{cases} \overline{W}_{2}(\lambda G, \lambda^{*} G_{*}), \ \text{when} \ K \leq \diff - 1 \\ 1_{\{\lambda \leq \lambda^{*}\}} \overline{W}_{2}(\lambda G, \lambda^{*} G_{*}) + 1_{\{\lambda > \lambda^{*}\}}  (\lambda + \lambda^{*}) W_{2}^2(G, \overline{G}_{*} (\lambda)), \ \text{when} \ K \geq \diff \end{cases}.
\end{align*}
Similar to the previous proofs, in order to obtain the conclusion of the theorem, we need to prove the following claims:
\begin{align}
	\lim \limits_{\epsilon \to 0} \inf \limits_{\lambda \in [0, 1], G \in \Ocal_{K}( \Theta)}{\left\{\dfrac{V(p_{\lambda G}, p_{\lambda^{*} G_{*}})}{D(\lambda G, \lambda^{*} G_{*})}: \ D(\lambda G, \lambda^{*} G_{*}) \leq \epsilon\right\}} > 0. \label{eq:claim_pardistin_strong_ind_over_point_wise}
\end{align}
% and
% \begin{align}
% 	\inf \limits_{\lambda \in [0, 1], G \in \Ocal_{K}( \Theta)}{\left\{\dfrac{V(p_{\lambda G}, p_{\lambda^{*} G_{*}})}{D(\lambda G, \lambda^{*} G_{*})}: \ D(\lambda G, \lambda^{*} G_{*}) \geq \epsilon' \right\}} > 0. \label{eq:claim_pardistin_strong_ind_over_point_wise_grt_0}
% \end{align}
% for all $\epsilon'>0$.
\paragraph{Proof of claim~\eqref{eq:claim_pardistin_strong_ind_over_point_wise}:} Assume that the above claim is not true. It indicates that we can find sequences $G_{n} = \sum_{i = 1}^{k_{n}} p_{i}^{n} \delta_{\theta_{i}^{n}} \in \mathcal{O}_{K}( \Theta)$ and $\lambda_{n} \in [0, 1]$ such that $D(\lambda_{n} G_{n}, \lambda^{*} G_{*})$ and $V(p_{\lambda_{n} G_{n}}, p_{\lambda^{*} G_{*}})/ D(\lambda_{n} G_{n}, \lambda^{*} G_{*})$ go to 0 as $n$ approaches to infinity. Given the assumption that $\theta_{i}^{*} = \theta_{i}^{0}$ for $i \in [\bar{k}]$, we obtain that
\begin{align}
p_{\lambda_{n} G_{n}}(x) - p_{\lambda^{*} G_{*}}(x) = (\lambda^{*} - \lambda_{n}) \sum_{i = \bar{k} + 1}^{k_{0}} p_{i}^{0} f(x| \theta_{i}^{0}) + \lambda_{n} \parenth{ \sum_{i = 1}^{k_{n}} p_{i}^{n} f(x| \theta_{i}^{n})} - \sum_{i = 1}^{k_{*}} \bar{p}_{i}^{*} f(x| \theta_{i}^{*}), \label{eq:key_express_pardistin_strong_ind_over_point_wise}
\end{align}
where $\bar{p}_{i}^{*} = \lambda^{*} p_{i}^{*} + (\lambda_{n} - \lambda^{*}) p_{i}^{0}$ when $1 \leq i \leq \bar{k}$ and $\bar{p}_{i}^{*} = \lambda^{*} p_{i}^{*}$ otherwise. 
Now, we prove the contradiction of our assumption under two separate settings of $\lambda_{n}$.
\paragraph{Case 1:} $\lambda^{*} \geq \lambda_{n}$ for infinitely many $n$. Without loss of generality, we assume that $\lambda^{*} \geq \lambda_{n}$ for all $n \geq 1$. Under this case, $D(\lambda_{n} G_{n}, \lambda^{*} G_{*}) = \overline{W}_{2}(\lambda_{n} G_{n}, \lambda^{*} G_{*})$. As $D(\lambda_{n} G_{n}, \lambda^{*} G_{*}) \to 0$, we have $\lambda_{n} \to \lambda^{*}$ and $W_{2}(G_{n}, G_{*}) \to 0$ as $n \to \infty$. Therefore, we can rewrite $G_{n}$ like equation~\eqref{eq:relabel_measure}. 

In light of equation~\eqref{eq:key_express_pardistin_strong_ind_over_point_wise} and the assumption $\lambda^{*} \geq \lambda_{n}$, by means of Taylor expansion up to the second order around $\theta_{1}^{*}, \ldots, \theta_{k_{*}}^{*}$ as that in the proof of Theorem~\ref{subsec:proof:theorem:distinguish_over_specified_point_wise}, we can view $(p_{\lambda_{n} G_{n}}(x) - p_{\lambda^{*} G_{*}}(x))/ D(\lambda_{n} G_{n}, \lambda^{*} G_{*})$ as a linear combination of elements of the forms $f(x| \theta_{i}^{0})$, $f(x| \theta_{j}^{*}), \frac{\partial{f}}{\partial{ \theta}} (x| \theta_{j}^{*})$, and $\frac{\partial^2{f}}{\partial{\theta^2}}(x| \theta_{j}^{*})$ for $\bar{k} + 1 \leq i \leq k_{0}$ and $j \in [k_{*}]$. 

It is sufficient to argue that not all the coefficients of these elements go 0 as the remaining Fatou's argument is similar to Step 3 of the proof of Theorem~\ref{subsec:proof:theorem:distinguish_over_specified_point_wise}. Indeed, assume that all of these coefficients go to 0 as $n$ tends to infinity. Since $\bar{k} < k_{0}$, we always have at least one index $I \in [\bar{k} + 1,  k_{0}]$. Studying the coefficient of $f(x| \theta_{I}^{0})$ proves that $(\lambda^{*} - \lambda_{n})/ D(\lambda_{n} G_{n}, \lambda^{*} G_{*}) \to 0$ as $n \to \infty$. From here, with similar argument as in Step 2 of claim~\eqref{eq:claim_1_distin_over_point_wise}, we can show that $1 = D(\lambda_{n} G_{n}, \lambda^{*} G_{*})/ D(\lambda_{n} G_{n}, \lambda^{*} G_{*}) \to 0$, which is a contradiction. Therefore, we obtain the conclusion of claim~\eqref{eq:claim_pardistin_strong_ind_over_point_wise}. 

\paragraph{Case 2:} $\lambda^{*} < \lambda_{n}$ for infinitely many $n$. Without loss of generality, we assume that $\lambda^{*} < \lambda_{n}$ for all $n \geq 1$. Under this case, we can rewrite   equation~\eqref{eq:key_express_pardistin_strong_ind_over_point_wise} as follows:
\begin{align*}
	p_{\lambda_{n} G_{n}}(x) - p_{\lambda^{*} G_{*}}(x) = \lambda_{n} \biggr( \underbrace{\sum_{i = 1}^{k_{n}} p_{i}^{n} f(x| \theta_{i}^{n})}_{: = f(x; G_{n})} - \biggr[ \underbrace{ \parenth{ 1 - \frac{\lambda^{*}}{\lambda_{n}}} \sum_{i = \bar{k} + 1}^{k_{0}} p_{i}^{0} f(x| \theta_{i}^{0}) + \sum_{i = 1}^{k_{*}} \frac{\bar{p}_{i}^{*}}{\lambda_{n}} f(x| \theta_{i}^{*})}_{: = f \parenth{ x; \overline{G}_{*}( \lambda_{n})}} \biggr] \biggr),
\end{align*}                                                                                                                                                     
where $\overline{G}_{*}(\lambda_{n}) : = \biggr( 1 - \frac{\lambda^{*}}{\lambda_{n}} \biggr) G_{0} + \frac{\lambda^{*}}{\lambda_{n}} G_{*}$. Under Case 2, $\bar{p}_{i}^{*} > \lambda^{*} p_{i}^{*} > 0$ for $i \in [k_{*}]$. Therefore, we can treat $f(x; G_{n})$ and $f \parenth{ x; \overline{G}_{*}( \lambda_{n})}$ respectively as mixtures with $k_{n}$ and $k_{0} + k_{*} - \bar{k}$ elements. 

Without loss of generality, we assume $k_{n} = K$ for all $n$, namely, the setting where $G_{n}$ have full $K$ supports. We consider three separate settings of $K$.
\paragraph{Case 2.1:} $K \leq k_{*} + k_{0} - \bar{k} - 1$. Under this case, $G_{n}$ has fewer supports than $\overline{G}_{*} ( \lambda_{n})$. Hence, there always exists one element in the set $\{\theta_{i}^{0}: \ \bar{k} + 1 \leq i \leq k_{0} \} \cup \{\theta_{j}^{*}: \ 1 \leq j \leq k_{*} \}$ such that no supports of $G_{n}$ converge to. We first show that this element cannot belong to the set $\{\theta_{j}^{*}: \ 1 \leq j \leq k_{*} \}$. Assume by contrary that this element is in that set. Without loss of generality, we assume this element is $\theta_{1}^{*}$. Since $V(p_{\lambda_{n} G_{n}}, p_{\lambda^{*} G_{*}})/ D(\lambda_{n} G_{n}, \lambda^{*} G_{*}) \to 0$, we have $f(x; G_{n}) - f(x; \overline{G}_{*}( \lambda_{n})) \to 0$ for almost surely $x$. Since $\theta_{i}^{n}$ do not converge to $\theta_{1}^{*}$, the identifiability of $f$ and the previous limit imply that $\bar{p}_{1}^{*}/ \lambda_{n}$ goes to 0 as $n \to \infty$, which is a contradiction as $\bar{p}_{1}^{*}/ \lambda_{n} > \lambda^{*} p_{1}^{*}$.

Therefore, there exists an element in the set $\{\theta_{i}^{0}: \ \bar{k} + 1 \leq i \leq k_{0} \}$ such that no elements of $G_{n}$ converge to. We assume without loss of generality that this element is $\theta_{1}^{0}$. In addition, all the elements in the set $\{\theta_{j}^{*}: \ 1 \leq j \leq k_{*} \}$ have at least one support of $G_{n}$ converge to. By performing Taylor expansion up to the second order around the limit points of the supports of $G_{n}$, we can view $(p_{\lambda_{n} G_{n}}(x) - p_{\lambda^{*} G_{*}}(x))/ D(\lambda_{n} G_{n}, \lambda^{*} G_{*})$ as a linear combination of elements of the forms $f(x| \theta_{i}^{0}), f(x| \theta_{j}^{*}), \frac{\partial{f}}{\partial{ \theta}} (x| \theta_{i}^{0})$, $\frac{\partial{f}}{\partial{ \theta}} (x| \theta_{j}^{*})$, $\frac{\partial^2{f}}{\partial{\theta^2}}(x| \theta_{i}^{0})$, and $\frac{\partial^2{f}}{\partial{\theta^2}}(x| \theta_{j}^{*})$ for some but not all $\bar{k} + 1 \leq i \leq k_{0}$ and for all $j \in [k_{*}]$. Assume that all of the coefficients associated with these elements go to 0 as $n$ goes to infinity. Since no support of $G_{n}$ converges to $\theta_{1}^{0}$, the previous assumptions mean that $(\lambda_{n} - \lambda^{*})/ D(\lambda_{n} G_{n}, \lambda^{*} G_{*}) \to 0$. Given that result, we have
\begin{align*}
0 = \lim_{n \to \infty} \frac{V(p_{\lambda_{n} G_{n}}, p_{\lambda^{*} G_{*}})}{D(\lambda_{n} G_{n}, \lambda^{*} G_{*})} = \lim_{n \to \infty} \frac{\lambda_{n} V(f(.; G_{n}), f(.;G_{*}))}{(\lambda_{n} + \lambda^{*}) W_{2}^2(G_{n}, G_{*})},
\end{align*}
which is a contradiction as $V(f(.; G_{n}), f(.;G_{*}))/ W_{2}^2(G_{n}, G_{*}) \not \to 0$ based on the result of Theorem 3.2 in~\cite{Ho-Nguyen-EJS-16}. Hence, not all the coefficients with  $f(x| \theta_{i}^{0}), f(x| \theta_{j}^{*}), \frac{\partial{f}}{\partial{ \theta}} (x| \theta_{i}^{0})$, $\frac{\partial{f}}{\partial{ \theta}} (x| \theta_{j}^{*})$, $\frac{\partial^2{f}}{\partial{\theta^2}}(x| \theta_{i}^{0})$, and $\frac{\partial^2{f}}{\partial{\theta^2}}(x| \theta_{j}^{*})$ go to 0 as $n \to \infty$. From here, invoking the Fatou's argument and the identifiability of $f$, we conclude the claim~\eqref{eq:claim_pardistin_strong_ind_over_point_wise} under Case 2.1.
\paragraph{Case 2.2:} $K \geq  k_{*} + k_{0} - \bar{k}$. We see that the number of support points of $\bar{G}_{*}(\lambda_n)$ decreases to $k_*$ if $\lambda_n\to \lambda^*$ as $n\to \infty$ or keeps being $k_* + k_0 - \bar{k}$ for any subsequence of $\lambda_n$ does not converge to $\lambda^*$. In both cases, we are in the over-fitted setting as $K \geq  k_{*} + k_{0} - \bar{k}$. If $\lambda_n\to \lambda^*$, our assumption $W_2(G_n, \bar{G}_{*}(\lambda_n))\to 0$ indicates that we can write $G_n$ as in equation \eqref{eq:relabel_measure} so that the atoms of $G_n$ converge to $\theta_{i}^*$ for $i\in [k_*]$ or 0. The proof of claim~\eqref{eq:claim_pardistin_strong_ind_over_point_wise} goes through similar to what of Theorem \ref{theorem:distinguish_over_specified_point_wise} (or Theorem 3.2. in Ho et al.~\cite{Ho-Nguyen-EJS-16}). 

If $\lambda_n\not\to \lambda^*$ as $n\to \infty$ then $\bar{G}_{*}(\lambda_n)$ has $k_0 + k_* - \bar{k}$ in any of its limits. Hence this is over-fitted setting when $K \geq  k_{*} + k_{0} - \bar{k}$ and we can proceed similar to above to have  claim~\eqref{eq:claim_pardistin_strong_ind_over_point_wise}.
% This setting indicates that $G_{n}$ has the same number of supports at $\overline{G}_{*} ( \lambda_{n})$. In addition, $D(\lambda_{n} G_{n}, \lambda^{*} G_{*}) = (\lambda_{n} + \lambda^{*}) W_{1}(G_{n}, \overline{G}_{*}( \lambda_{n}))$. According to Theorem 3.1 in Ho et al.~\cite{Ho-Nguyen-EJS-16}, we have 
% \begin{align*}
% 	V(f(.; G_{n}), f(.; \overline{G}_{*}( \lambda_{n}))/ W_{1}(G_{n}, \overline{G}_{*} ( \lambda_{n})) \not \to 0,
% \end{align*}
% as $n \to \infty$. Since $V(p_{\lambda_{n} G_{n}}, p_{\lambda^{*} G_{*}}) = \lambda_{n} V(f(.; G_{n}), f(.; \overline{G}_{*}( \lambda_{n}))$, the previous result proves that $V(p_{\lambda_{n} G_{n}}, p_{\lambda^{*} G_{*}})/ D(\lambda_{n} G_{n}, \lambda^{*}G_{*}) \not \to 0$, which is a contradiction. As a consequence, claim~\eqref{eq:claim_pardistin_strong_ind_over_point_wise} is shown under Case 2.2.

\paragraph{Case 2.3:} $K = k_{*} + k_{0} - \bar{k}$ and $\lambda_n > \lambda^* + \delta > \lambda^*$ for all $n$. In this case, $\lambda_n \not \to \lambda^*$, so that $\bar{G}_{*}(\lambda_n)$ has $k_0 + k_* - \bar{k}$ in any of its limits. Hence, this is an exact-fitted setting and we can apply Theorem 3.1. in Ho et al.~\cite{Ho-Nguyen-EJS-16}. As a consequence, claim~\eqref{eq:claim_pardistin_strong_ind_over_point_wise} is shown under Case 2.3.
%%%%%%%%%%%%%%%%%%%%%%%%%%%%%%%%%%%%%%%%%%%%%%%%%%%%%%%%%%%%%%%%%%%%%%%%%%%%%%%%%%%%%
\subsection{Proof of Theorem~\ref{theorem:pardistinguish_point_wise_overspec_weak_iden}}
\label{subsec:proof:theorem:pardistinguish_point_wise_overspec_weak_iden}
\begin{customthm}
{\ref{theorem:pardistinguish_point_wise_overspec_weak_iden}}
Assume that $\sumfun$ takes the form~\eqref{eq:par_distin_h0}. Besides that, $K \geq k_{0}$ and $f$ is location-scale Gaussian distribution. Then, for any $\lambda \in [0, 1]$ and $G \in \Ocal_{K, c_{0}}( \Theta)$ for some $c_{0} > 0$, there exist positive constants $C_{1}, C_{2}, C_3, C_4$ depending only on $\lambda^{*}, G_{*}, G_{0}, \Theta$ ($C_3$ and $C_4$ also depends on $\delta$) such that the following holds:
\begin{itemize}
\item[(a)] When $K \leq k_{*} + k_{0} - \bar{k} - 1$, then
\begin{align*}
V(p_{\lambda^{*}, G_{*}}, p_{\lambda, G}) & \geq  C_1 \overline{W}_{\overline{r}(K - k_{*} )}(\lambda G, \lambda^{*} G_{*}).
\end{align*}
\item[(b)] When $K \geq k_{*} + k_{0} - \bar{k} $, then
\begin{align*}
V(p_{\lambda^{*}, G_{*}}, p_{\lambda, G})  \geq C_{2} \biggr( 1_{\{\lambda \leq \lambda^{*}\}} \overline{W}_{\overline{r}(K - k_{*} )}(\lambda G, \lambda^{*} G_{*}) \\
+ 1_{\{\lambda > \lambda^{*}\}}  W_{\overline{r}(K - k_{*} )}^{\overline{r}(K - k_{*})}(G, \overline{G}_{*}(\lambda))\biggr)
\end{align*}
\item[(c)] For $\delta > 0$, when $K = k_{*} + k_{0} - \bar{k}$, we have 
\begin{align*}
 V(p_{\lambda^{*}, G_{*}}, p_{\lambda, G}) & \geq C_{3}  1_{\{\lambda > \lambda^{*} + \delta\}}  W_{1}(G, \overline{G}_{*}(\lambda)),
\end{align*}
and when $K > k_{*} + k_{0} - \bar{k}$, we have 
\begin{align*}
 V(p_{\lambda^{*}, G_{*}}, p_{\lambda, G}) \geq & C_{4}  1_{\{\lambda > \lambda^{*} + \delta\}} \\ 
 & \times W_{\overline{r}(K - k_{0} - k_{*} + \bar{k} )}^{\overline{r}(K - k_{0}\ - k_{*} + \bar{k} )}(G, \overline{G}_{*}(\lambda)).
\end{align*}
\end{itemize}
\end{customthm}

We still denote $\diff = k_* + k_0 - \overline{k}$ and follow the path of Theorem~\ref{theorem:pardistinguish_point_wise_overspec_strong_iden} to prove by contradiction. We denote by $\overline{r}_1 = \overline{r}(K - k_*)$, $\overline{r}_2 = \overline{r}(K - k_0 - k_* + \overline{k})$, and 
\begin{align*}
	D(\lambda G, \lambda^{*} G_{*}) = \begin{cases} \overline{W}_{\overline{r}_1}(\lambda G, \lambda^{*} G_{*}), \ \text{when} \ K \leq \diff - 1 \\ 1_{\{\lambda \leq \lambda^{*}\}} \overline{W}_{\overline{r}_1}(\lambda G, \lambda^{*} G_{*}) + 1_{\{\lambda > \lambda^{*}\}}  (\lambda + \lambda^{*}) W_{\overline{r}_2}^{\overline{r}_2}(G, \overline{G}_{*} (\lambda)), \ \text{when} \ K \geq \diff \end{cases}.
\end{align*}
We need to show the following claim:
\begin{align}
	\lim \limits_{\epsilon \to 0} \inf \limits_{\lambda \in [0, 1], G \in \Ocal_{K}( \Theta)}{\left\{\dfrac{V(p_{\lambda G}, p_{\lambda^{*} G_{*}})}{D(\lambda G, \lambda^{*} G_{*})}: \ D(\lambda G, \lambda^{*} G_{*}) \leq \epsilon\right\}} > 0. \label{eq:claim_pardistin_weak_ind_over_point_wise}
\end{align}
There exists sequences $\lambda_n$ and $G_n = \sum_{i=1}^{k_n} p_i^n \delta_{\theta_i^n} \in \Ocal_K(\Theta)$ such that $D(\lambda_n G_n, \lambda^* G_{*}) \to 0$ and $V(p_{\lambda_n G_n}, p_{\lambda^* G_{*}})/D(\lambda_n G_n, \lambda^* G_{*}) \to 0$, where $D$ is the lower bound in the theorem statement. For the ease of presentation, we consider the one dimension Gaussian case where $(\mu, \Sigma) = (\theta, v)$, the higher dimension cases are treated similar. 

\paragraph{Case 1:} $\lambda^* \geq \lambda_n$ for infinitely many $n$. We can use the subsequence argument to have $\lambda^* \geq \lambda_n$ for all $n$ and $G_n$ can be assumed to have a fixed number of atoms (less than or equals $K$) and have a representation as in \eqref{eq:relabel_measure}. In this case,
\begin{equation}
    D(\lambda_n G_n, \lambda^* G_{*}) = \abss{\lambda_n - \lambda^*} + (\lambda_n + \lambda^*) \overline{W}_{\overline{r}_1}^{\overline{r}_1}(G_n, G_*) \to 0, \quad \dfrac{V(p_{\lambda^* G_*}, p_{\lambda_n G_n})}{D(\lambda_n G_n, \lambda^* G_{*})} \to 0.
\end{equation}
Using Taylor expansion of $f$ around $\{(\theta_i^*, v_i^*)\}_{i=1}^{k_*}$ to the $\overline{r}_1-$th order we have
\begin{align*}
    p_{\lambda_n G_n}(x) - p_{\lambda^* G_*}(x) & = (\lambda^* - \lambda_n) \sum_{i=\bar{k}+1}^{k_0}  p_i^0 f(x|\theta_i^0, v_i^0) + \lambda_n (\sum_{i=1}^{k_* + \underline{l}} \sum_{j=1}^{s_i} p_{ij}^{n} f(x|\theta_{ij}^{n}, v_{ij}^{n})) - \sum_{i=1}^{k_*} \overline{p}_i^* f(x|\theta_i^*, v_i^*)\\
    & = (\lambda^* - \lambda_n) \sum_{i=\bar{k}+1}^{k_0}  p_i^0 f(x|\theta_i^0, v_i^0) + \sum_{i=1}^{k_* + \underline{l}} \sum_{j=1}^{s_i} \lambda_n p_{ij}^{n} \sum_{|\boldsymbol{\alpha}|=1}^{\overline{r}_1} (\Delta\theta_{ij}^{n})^{\alpha_1} (\Delta v_{ij}^{n})^{\alpha_2} \dfrac{1}{\boldsymbol{\alpha}!} \dfrac{\partial^{|\boldsymbol{\alpha}| f(\theta_i^*, v_i^*)}}{\partial^{\alpha_1}\theta \partial^{\alpha_2} v}\\
    & + \sum_{i=1}^{k_*+\underline{l}} (\Delta\overline{p}_{i \cdot}^n) f(x|\theta_i^*, v_i^*) + R(x),
\end{align*}
where $\boldsymbol{\alpha} = (\alpha_1, \alpha_2)$, $|\boldsymbol{\alpha}| = \alpha_1 + \alpha_2, \boldsymbol{\alpha}! = \alpha_1!\alpha_2!$, $\Delta \overline{p}^{n}_{i\cdot} = \lambda_n \sum_{j}p_{ij}^n - \overline{p}_i^*$, $\Delta \theta_{ij}^{n} = \theta_{ij}^n - \theta_i^*, \Delta v_{ij}^{n} = v_{ij}^n - v_i^*$ and $R(x) = O(\sum_{i=1}^{k_*+\underline{l}} \sum_{j=1}^{s_i} p_{ij}^{n} (|\Delta \theta_{ij}^n|^{\overline{r}_1}+ |\Delta v_{ij}^n|^{\overline{r}_1}))$. Now we can use the character equation $\dfrac{\partial^2 f}{\partial \theta^2} = 2 \dfrac{\partial f}{\partial v}$ to rewrite the formula above as
\begin{eqnarray}
    (\lambda^* - \lambda_n) \sum_{i=\bar{k}+1}^{k_0}  p_i^0 f(x|\theta_i^0, v_i^0) + \sum_{\alpha=1}^{2\overline{r}_1} \sum_{i=1}^{k_* + \underline{l}} \left( \sum_{j=1}^{s_i} \lambda_n p_{ij}^{n}  \sum_{n_1, n_2}  \dfrac{(\Delta\theta_{ij}^{n})^{n_1} (\Delta v_{ij}^{n})^{n_2}}{2^{n_2}n_1!n_2!}\right) \dfrac{\partial^{\alpha} f(\theta_i^*, v_i^*)}{\partial \theta^{\alpha}} \nonumber \\ 
    + \sum_{i=1}^{k_*+\underline{l}} (\Delta\overline{p}_{i \cdot}^n) f(x|\theta_i^*, v_i^*) + R(x),\label{eq:location-scale-Gauss-taylor}
\end{eqnarray}
where we sum over $n_1, n_2$ such that $n_1+2n_2=\alpha, n_1+n_2\leq \overline{r}_1$.
Now we turn into proving the non-vanishing coefficients. Assume that all coefficients in the formula above vanish when dividing by $D(\lambda_n G_n, \lambda^* G_*)$  when $n\to \infty$. Because 
\begin{equation}
    	D( \lambda_nG_{n}, \lambda^*G_{*}) \asymp |\lambda_n - \lambda^*| + (\lambda_n + \lambda^*) \left(\sum_{i = 1}^{k_{*} + \bar{l}} \abss{ \Delta p_{i.}^{n}} + \sum_{i = 1}^{k_{*} + \bar{l}} \sum_{j = 1}^{s_{i}}  p_{ij}^{n} (\enorm{ \Delta \theta_{ij}^{n}}^{\overline{r}_1} + \enorm{ \Delta v_{ij}^{n}}^{\overline{r}_1}) \right):=D_{\overline{r}_1}(G_n, G_*),
\end{equation}
we have
\begin{equation}\label{eq:thm-3-5-contradiction-cond}
    \dfrac{\lambda^* -\lambda_n}{D_{\overline{r}_1}(G_n, G_*)} \to 0, \,\,\, \dfrac{\Delta\overline{p}_{i \cdot}^n}{D_{\overline{r}_1}(G_n, G_*)} \to 0.
\end{equation}
These limits together imply
\begin{equation}
    \dfrac{(\lambda^* +\lambda_n) \Delta\overline{p}_{i \cdot}^n}{D_{\overline{r}_1}(G_n, G_*)} \to 0, \quad \forall i = 1, \dots, k_* + \overline{l}.
\end{equation}
From the definition of $D_{\overline{r}_1}$, it can be deduced that there exists at least an index $i^*$ such that
\begin{equation}
    \sum_{j=1}^{s_{i*}} \dfrac{(\lambda_n + \lambda^*)  p_{i^* j}^{n} ((\theta_{ij}^{n})^{\overline{r}_1}+ (v_{ij}^{n})^{\overline{r}_1})}{D_{\overline{r}_1}(G_n, G_*)} \not \to 0.
\end{equation}
Without loss of generality, assign $i^*=1$. But as we assume all the coefficients in equation \eqref{eq:location-scale-Gauss-taylor} go to 0 for all $\alpha$ and $i$, we have
\begin{equation}
    \dfrac{\sum\limits_{j=1}^{s_1} \lambda_n p_{1j}^{n} \sum\limits_{\substack{n_1+2 n_2=\alpha\\ n_1+n_2\leq \overline{r}_1}}  \dfrac{(\theta_{1j}^{n})^{n_1} (v_{1j}^{n})^{n_2}}{2^{n_2}n_1!n_2!}}{D_{\overline{r}_1}(G_n, G_*)} \to 0,
\end{equation}
for all $\alpha = 1, \dots, 2\overline{r}_1$. From two expressions above combining with equation \eqref{eq:thm-3-5-contradiction-cond}, we have for all $\alpha= 1, \dots, 2\overline{r}_1$, 
\begin{equation}\label{eq:weak-ident-F-alpha}
    F_{\alpha} := \dfrac{\sum\limits_{j=1}^{s_1} p_{1j}^{n} \sum\limits_{\substack{n_1+2 n_2=\alpha\\ n_1+n_2\leq \overline{r}_1}} \dfrac{(\Delta\theta_{1j}^{n})^{n_1} (\Delta v_{1j}^{n})^{n_2}}{2^{n_2}n_1!n_2!}}{\sum_{j=1}^{s_1} p_{1 j}^{n} ((\Delta\theta_{ij}^{n})^{\overline{r}_1} + (\Delta v_{ij}^{n})^{\overline{r}_1})}\to 0. 
\end{equation}
If $s_1 = 1$ then substituting $\alpha = 1$ and $\alpha = 2\overline{r}_1$ gives 
\begin{equation*}
    \dfrac{|\Delta \theta_{11}^{n}|^{\overline{r}_1}}{|\Delta \theta_{11}^{n}|^{\overline{r}_1}+ |\Delta v_{11}^{n}|^{\overline{r}_1}}, \dfrac{|\Delta v_{11}^{n}|^{\overline{r}_1}}{|\Delta \theta_{11}^{n}|^{\overline{r}_1}+ |\Delta v_{11}^{n}|^{\overline{r}_1}} \to 0,
\end{equation*}
which is impossible as they are sum up to 1 for all $n$. Hence $s_1 \geq 2$. Now we proceed to show the contradiction using the system of equations \eqref{eqn:system_polynomial_Gaussian_first}. Denote by $\overline{p}_n = \max_{1\leq j \leq s_1} \{p_{1j}^{n} \}, \overline{M}_n = \max_{1\leq j \leq s_1} \{|\Delta \theta_{1j}^{n}|, |\Delta v_{1j}^{n}|^{1/2} \}$. By the subsequence argument in compact sets, without loss of generality, we can denote $c_j^2 := \lim_{n\to \infty} p_{1j}^{n}/\overline{p}_n$, $a_j = \lim \Delta \theta_{1j}^{n} / \overline{M}_n$, and $b_j = \lim \Delta v_{1j}^{n} / \overline{M}_n$ for all $j = 1,\dots, k_* + \overline{l}$. Because of the definition of $\Ocal_{K, c_0}$, we have $p_j \geq c_0$ for all $j$, which implies all $c_j$ are different from 0 and at least one of them is $1$. Similarly, in $(a_j, b_j)_{j}$, there is at least one of them equals to $1$ or $-1$. Dividing both numerators and denominators of equation \eqref{eq:weak-ident-F-alpha} by $\overline{p}_n \overline{M}_n^{\alpha}$, we have
\begin{equation*}
    \sum\limits_{j=1}^{s_1}\sum\limits_{\substack{n_1+2 n_2=\alpha}} \dfrac{c_j^2 a_j^{n_1} b_j^{n_2}}{n_1! n_2!} = 0,  
\end{equation*}
for all $\alpha = 1, \dots, \overline{r}_1$. Hence, we get the contradiction, where we use the fact that $s_1\leq K - k_* + 1$ (as $s_i\geq 1$ for all $i \geq 2$) and  $\overline{r}_1 = \overline{r}(K-k_*)$ is the smallest number such that equation \eqref{eqn:system_polynomial_Gaussian_first}, where $k = K-k_*,$ has the trivial solution only. After that, we can argue as in the Step 9 of Proposition 2.2. in ~\cite{Ho-Nguyen-Ann-16} to get the contradiction to the assumption proposed in the beginning, where we use the fact that Gaussian family is identifiable up to any order with respect to the location parameters.

\paragraph{Case 2:} $\lambda^* \leq \lambda_n$ for all $n$. We rewrite 
\begin{equation}\label{Gaussian-family-rewrite}
    	p_{\lambda_{n} G_{n}}(x) - p_{\lambda^{*} G_{*}}(x) = \lambda_{n} \biggr( \underbrace{\sum_{i = 1}^{k_{n}} p_{i}^{n} f(x| \theta_{i}^{n})}_{: = f(x; G_{n})} - \biggr[ \underbrace{ \parenth{ 1 - \frac{\lambda^{*}}{\lambda_{n}}} \sum_{i = \bar{k} + 1}^{k_{0}} p_{i}^{0} f(x| \theta_{i}^{0}) + \sum_{i = 1}^{k_{*}} \frac{\bar{p}_{i}^{*}}{\lambda_{n}} f(x| \theta_{i}^{*})}_{: = f \parenth{ x; \overline{G}_{*}( \lambda_{n})}} \biggr] \biggr),
\end{equation}
\paragraph{Cases 2.1.} $K\leq \diff - 1$, argue similarly to Case 2.1. of the proof of Theorem \ref{theorem:pardistinguish_point_wise_overspec_strong_iden}, we have $\dfrac{\lambda_n - \lambda^*}{D(\lambda_n G_n, \lambda^* G_*)}\to 0$ as $n\to \infty$. Now we arrive at the equation \eqref{eq:thm-3-5-contradiction-cond} of Case 1. Follow the argument above, we can prove claim \eqref{eq:claim_pardistin_weak_ind_over_point_wise}.
\paragraph{Case 2.2.} $K \geq \diff$, we can see equation  \eqref{Gaussian-family-rewrite} as an over-fitted mixture of location-scale Gaussian setting where the number of over-fitted atoms is at most $K - k_*$. Hence we can argue similar to Case 1 or the Proposition 2.2. in~\cite{Ho-Nguyen-Ann-16} to obtain the conclusion.
\paragraph{Cases 2.3.} $K = \diff$ and $\lambda_n > \lambda^* + \delta$ for all $n$. From the presentation as in equation \eqref{Gaussian-family-rewrite}, we can see that $1 - \dfrac{\lambda^*}{\lambda_n}$ does not vanish in any of it limits. Therefore $\overline{G}_{*}(\lambda_n)$ has $k_* + k_0 - \overline{k} = \Tcal$ number of components in its limits. Because this is an exact-fitted setting, we can apply Theorem 3.1. in Ho et al.~\cite{Ho-Nguyen-EJS-16} to get the result of claim \eqref{eq:claim_pardistin_weak_ind_over_point_wise}
\paragraph{Cases 2.4.} $K > \diff$ and $\lambda_n > \lambda^* + \delta$ for all $n$, we can also see that $\overline{G}_{*}(\lambda_n)$ has $k_* + k_0 - \overline{k} = \Tcal$ number of components in its limits. We can apply Proposition 2.2. in Ho et al.~\cite{Ho-Nguyen-Ann-16} to get the result of claim \eqref{eq:claim_pardistin_weak_ind_over_point_wise}.

%%%%%%%%%%%%%%%%%%%
%%%%%%%%%%%%%%%%%%%%%%%%%%%%%%%%%%%%%%%%%%%%%%%%%%%%%%%%%%%%%%%%%

%%%%%%%%%%%%%%%%%%%%%%%%%%%%%%%%%%%%%%%%%%%%%%%%%%%%%%%%%%%%%%%%%%%%%%%%%%%%%%%%%%%%%
\subsection{Proof of Theorem~\ref{theorem:pardistinguish_point_wise_equal_exactspec}}
\label{subsec:proof:theorem:pardistinguish_point_wise_equal_exactspec}
%%%%%%%%%%%%%%%%%%%
%%%%%%%%%%%%%%%%%%%%%%%%%%%%%%%%%%%%%%%%%%%%%%%%%%%%%%%%%%%%%%%%%
\begin{customthm}
{\ref{theorem:pardistinguish_point_wise_equal_exactspec}}
Assume that $\sumfun$ takes the form~\eqref{eq:par_distin_h0} and $\bar{k} = k_{0}$. Besides that, $f$ is second order identifiable. Then, for any $\lambda \in [0, 1]$ and $G \in \Ocal_{K}( \Theta)$ that $K \geq k_*$, there exist positive constants $C_{1}$ and $C_{2}$ depending only on $\lambda^{*}, G_{*}, G_{0}, \Theta$ such that the following holds:
\begin{itemize}
\item[(a)] If $\extraind( \lambda)$ is not ratio-independent, then
\begin{align}
V(p_{\lambda^{*} G_{*}}, p_{\lambda G}) & \geq C_{1} \biggr[ 1_{\{\lambda \in\extraset^{c}\}} + 1_{\{\lambda \in\extraset\}}  W_{2}^2(G, \bar{G}_{*} ( \lambda))\biggr]. 
\end{align}
\item[(b)] If $\extraind( \lambda)$ is ratio-independent, then
\begin{align}
V(p_{\lambda^{*} G_{*}}, p_{\lambda G}) & \geq C_{2} \biggr[ 1_{\{\lambda \in\extraset^{c}\}} \biggr(  \sum_{i \in \extraind( \lambda)} \biggr[ (\lambda^{*} - \lambda) p_{i}^{0} - \lambda^{*}p_{i}^{*} \biggr] \nonumber \\
& + \sumind(\mathcal{I} (\lambda)) W_{2}^2(G, \widetilde{G}_{*} ( \lambda)) \biggr) \nonumber \\ 
& + 1_{\{\lambda \in\extraset\}}  W_{2}^2(G, \bar{G}_{*} ( \lambda))\biggr]. 
\end{align}
\end{itemize}
\end{customthm}

To ease the ensuing presentation, we denote $D(\lambda G, \lambda^{*} G_{*}) = 
1_{\{\lambda \in\extraset^{c}\}}  \biggr( \sum_{i \in \extraind( \lambda)} \biggr[ (\lambda^{*} - \lambda) p_{i}^{0} - \lambda^{*}p_{i}^{*} \biggr] + \sumind(\mathcal{I} (\lambda)) W_{2}^2(G, \widetilde{G}_{*} ( \lambda)) \biggr) + 1_{\{\lambda \in\extraset\}}   W_{2}^2(G, \bar{G}_{*} ( \lambda))$ when $\extraind( \lambda)$ is ratio-independent or $D(\lambda G, \lambda^{*} G_{*}) = 
1_{\{\lambda \in\extraset^{c}\}} + 1_{\{\lambda \in\extraset\}}  W_{2}^2(G, \bar{G}_{*} ( \lambda))$ when $\extraind( \lambda)$ is not ratio-independent.

In order to prove the theorem, it is sufficient to verify the following inequality:
\begin{align}
	\lim \limits_{\epsilon \to 0} \inf \limits_{\lambda \in [0,1], G \in \Ecal_{k_{*}}( \Theta)}{\left\{\dfrac{V(p_{\lambda G}, p_{\lambda^{*} G_{*}})}{D(\lambda G, \lambda^{*} G_{*})}: \ D(\lambda G, \lambda^{*} G_{*}) \leq \epsilon\right\}} > 0. \label{eq:claim_pardistinguish_point_wise_equal_exactspec}
\end{align}
\paragraph{Proof of claim~\eqref{eq:claim_pardistinguish_point_wise_equal_exactspec}:} Assume that the above claim is not true. It implies that there exist sequences $G_{n} = \sum_{i = 1}^{k_{n}} p_{i}^{n} \delta_{\theta_{i}^{n}} \in \Ocal_{K}( \Theta)$ and $\lambda_{n} \in [0, 1]$ such that $D(\lambda_{n} G_{n}, \lambda^{*} G_{*})$ and $V(p_{\lambda_{n} G_{n}}, p_{\lambda^{*} G_{*}})/ D(\lambda_{n} G_{n}, \lambda^{*} G_{*})$ go to 0 as $n$ approaches to infinity. Since $\bar{k} = k_{0}$ and $G_{*}$ admits the form~\eqref{eq:refor_G*}, we find that
\begin{align}
p_{\lambda_{n} G_{n}}(x) - p_{\lambda^{*} G_{*}}(x) = \lambda_{n} \parenth{ \sum_{i = 1}^{k_{n}} p_{i}^{n} f(x| \theta_{i}^{n})} - \sum_{i = 1}^{k_{*}} \bar{p}_{i}^{*} f(x| \theta_{i}^{*}), \label{eq:key_express_pardistinguish_point_wise_equal_exactspec}
\end{align}
where $\bar{p}_{i}^{*} = \lambda^{*} p_{i}^{*} + (\lambda_{n} - \lambda^{*}) p_{i}^{0}$ when $1 \leq i \leq k_{0}$ and $\bar{p}_{i}^{*} = \lambda^{*} p_{i}^{*}$ otherwise. In addition, $\theta_{i}^{*} = \theta_{i}^{0}$ for $i \in [k_{0}]$. From this presentation, we see that there must exists a constant $C$ depending on $\lambda^*, G_*, G_0$ such that $\liminf \lambda_n >C$. Indeed, suppose it is not the case, then by the subsequence argument, we can assume that $\lambda_n\to 0$. Besides, $V(\lambda_n G_n, \lambda^* G_*) \to 0$, we have $\overline{p}_i^* \to 0$ for all $i\in [k_*]$. These conditions lead to $p_i^* = 0$ for all $i > k_0$ and $p_{i}^0 = p_i^*$ for all $i\in [k_0]$, which mean that $G_* = G_0$ (a contradiction to our assumption). Hence, limits of $(\lambda_n)$ is bounded below.
We have two settings with $\lambda_{n}$.
\paragraph{Case 1:} $\lambda_{n} \in \extraset$ for infinitely many $n$. Without loss of generality, we assume that $\lambda_{n} \in \extraset$ for all $n \geq 1$. If $k_* = k_0$ then we see that $\bar{p}_{i}^{*}$ can not vanish simultaneously when $n\to\infty$ for all $i$, otherwise we have $G^* = G_0$, which contradicts to the assumption in this section. Otherwise, $k_* > k_0$, and $\overline{p}^*_{i}$ does not vanish for all $i > k_0$. Therefore, every limit of $\sum_{i = 1}^{k_{*}} \bar{p}_{i}^{*} f(x| \theta_{i}^{*})$
has a number of atoms ranging from $\max\{1, k_* - k_0\}$ to $k_*$, which is less than or equal to $K$. So that this is an over-fitted scenario. 
In addition, $D(\lambda_{n} G_{n}, \lambda^{*} G_{*}) = W_{2}^2(G_{n}, \bar{G}_{*} ( \lambda_{n}))$. We can further rewrite equation~\eqref{eq:key_express_pardistinguish_point_wise_equal_exactspec} as:
\begin{align*}
	p_{\lambda_{n} G_{n}}(x) - p_{\lambda^{*} G_{*}}(x) = \lambda_{n} ( f(x; G_{n}) - f(x; \bar{G}_{*}( \lambda_{n})).
\end{align*}
From Theorem 3.2 in Ho et al.~\cite{Ho-Nguyen-EJS-16}, we have $V(f(.; G_{n}), f(.; \overline{G}_{*}( \lambda_{n}))/ W_{2}^2(G_{n}, \bar{G}_{*} ( \lambda_{n})) \not \to 0$ as $n \to \infty$. Putting the above results together, we obtain that $V(p_{\lambda_{n} G_{n}}, p_{\lambda^{*} G_{*}})/ D(\lambda_{n} G_{n}, \lambda^{*} G_{*}) \not \to 0$, which is a contradiction. Hence, we reach the conclusion of claim~\eqref{eq:key_express_pardistinguish_point_wise_equal_exactspec}. 
\paragraph{Case 2:} $\lambda_{n} \not \in \extraset$ for infinitely many $n$. Without loss of generality, we assume that $\lambda_{n} \not \in \extraset$ for all $n \geq 1$. Under this setting, $\extraind( \lambda_{n}) \neq \emptyset$. In addition, for any $i \in \extraind( \lambda_{n})$, $\bar{p}_{i}^{*} < 0$. Given these conditions, we can rewrite equation~\eqref{eq:key_express_pardistinguish_point_wise_equal_exactspec} as follows:
\begin{eqnarray}\label{eq:thm36-rewrite}
	p_{\lambda_{n} G_{n}}(x) - p_{\lambda^{*} G_{*}}(x) = \sum_{i \in \extraind ( \lambda_{n})} (- \bar{p}_{i}^{*}) f(x| \theta_{i}^{0}) + \biggr[ \lambda_{n} \parenth{ \sum_{i = 1}^{k_{n}} p_{i}^{n} f(x| \theta_{i}^{n})} - \sum_{i \in \extraind ( \lambda_{n})^{c}} \bar{p}_{i}^{*} f(x| \theta_{i}^{0}) & \nonumber\\
	& \hspace{- 8 em} - \sum_{i = k_{0} + 1}^{k_{*}} \bar{p}_{i}^{*} f(x| \theta_{i}^{*}) \biggr].
\end{eqnarray}
We have two separate settings with $\extraind( \lambda_{n})$.
\paragraph{Case 2.1:} $\extraind( \lambda_{n})$ is not ratio-independent. Under this case, $D( \lambda_{n} G_{n}, \lambda^{*} G_{*}) = 1$. Since $V(p_{\lambda_{n} G_{n}}, p_{\lambda^{*} G_{*}})/ D( \lambda_{n} G_{n}, \lambda^{*} G_{*}) \to 0$, we have $V(p_{\lambda_{n} G_{n}}, p_{\lambda^{*} G_{*}}) \to 0$. It indicates that $p_{\lambda_{n} G_{n}}(x) - p_{\lambda^{*} G_{*}}(x) \to 0$ almost surely $x$. Since $- \bar{p}_{i}^{*} > 0$ for all $i \in \extraind( \lambda_{n})$, the previous limit demonstrates that $\bar{p}_{i}^{*} \to 0$ for all $i \in \extraind( \lambda_{n})$, which leads to $p_{i}^{*}/ p_{i}^{0} = p_{j}^{*}/ p_{j}^{0}$ for all $i, j \in \extraind( \lambda_{n})$. It contradicts the assumption that $\extraind( \lambda_{n})$ is not ratio-independent. Hence, we achieve the conclusion of claim~\eqref{eq:key_express_pardistinguish_point_wise_equal_exactspec} under Case 2.1. 
\paragraph{Case 2.2:} $\extraind( \lambda_{n})$ is ratio-independent. Under this case, $D( \lambda_{n} G_{n}, \lambda^{*} G_{*}) = \sum_{i \in \extraind( \lambda_{n})} \biggr[ (\lambda^{*} - \lambda_{n}) p_{i}^{0} - \lambda^{*}p_{i}^{*} \biggr] + \sumind(\mathcal{I} (\lambda_{n})) W_{2}^2(G_{n}, \widetilde{G}_{*} ( \lambda_{n}))\to 0$ and $V(p_{\lambda_n G_n, \lambda^*G_*})/D( \lambda_{n} G_{n}, \lambda^{*} G_{*})\to 0$, which imply $V(p_{\lambda_n G_n, \lambda^*G_*})\to 0$. We first prove that $\sumind(\mathcal{I} (\lambda_{n}))\not \to 0$. Indeed, suppose it is not the case, then $p_i^* = 0$ for all $i > k_0$ and $( \lambda^* p_i^* + (\lambda_n - \lambda^*)p_i^0)\to 0$ for all $i\in \mathcal{I}(\lambda_n)$. From equation \eqref{eq:thm36-rewrite} and the fact that $V(p_{\lambda_n G_n, \lambda^*G_*})\to 0$, we also see that $\overline{p}_i^* \to 0$ for all $i\in \mathcal{I}(\lambda_n)$ and $\lambda_n\to 0$. But that means
\begin{equation*}
    \lambda_n\to 0, \lambda^* p_i^* + (\lambda_n - \lambda^*) p_i^0 \to 0, \quad \forall i \in [k_0].
\end{equation*}
Those limits together imply that $\lambda^*(p_i^0 - p_i^*) = 0$ for all $i\in [k_0]$, which is contradictory with our assumption that $G^* \neq G_0$. Hence $\sumind(\mathcal{I}(\lambda_n))\not \to 0$.
As $D(\lambda_{n} G_{n}, \lambda^{*} G_{*}) \to 0$, we have $W_{2}^2 (G_{n}, \widetilde{G}_{*} ( \lambda_{n})) \to 0$ as $n \to \infty$. It implies that we can rewrite $G_{n}$ as follows:  
\begin{align}
	G_{n} = \sum_{i \in \extraind( \lambda_{n})^c \cup \{k_{0} + 1, \ldots, k_{*} + \bar{l}\}} \sum_{j = 1}^{s_{i}} p_{ij}^{n} \delta_{ \theta_{ij}^{n}}, \label{eq:relabel_measure_pardistinguish_equal_exactspec}
\end{align}
where $\sum_{j = 1}^{s_{i}} p_{ij}^{n} \to \bar{p}_{i}^{*}/ \sumind(\mathcal{I} (\lambda_{n}))$ and $\theta_{ij}^{n} \to \theta_{i}^{*}$ for all $i \in \mathcal{J} : = \extraind( \lambda_{n})^c \cup \{k_{0} + 1, \ldots, k_{*} + \bar{l}\}$. Here, $\bar{p}_{i}^{*} = 0$ for $k_{*} + 1 \leq i \leq k_{*} + \bar{l}$. In addition, $\sum_{i \in \mathcal{J}} s_{i} = k'$ for some $k'$ such that $k_{*} - k_{0} + |\extraind( \lambda_{n})^c| \leq k' \leq k_{*}$.  To faciliate the proof argument, we denote $\Delta \theta_{ij}^{n} : = \theta_{ij}^{n} - \theta_{i}^{*}$ and $\Delta p_{i.}^{n} : = \sum_{j = 1}^{s_{i}} p_{ij}^{n} - \bar{p}_{i}^{*}/ \sumind(\mathcal{I} (\lambda_{n}))$ for $i \in \mathcal{J}$. The result of Lemma 3.1 in Ho et al.~\cite{Ho-Nguyen-Ann-17} leads to
\begin{align}
	W_{2}^2( G_{n}, \tilde{G}_{*} ( \lambda_{n})) \asymp \sum_{i \in \mathcal{J}} \abss{ \Delta p_{i.}^{n}} + \sum_{i \in \mathcal{J}} \sum_{j = 1}^{s_{i}}  p_{ij}^{n} \enorm{ \Delta \theta_{ij}^{n}}^2. \label{eq:wasserstein_equivalence_over}
\end{align}  
Invoking Taylor's expansion up to the second order, we have
\begin{align}
	p_{\lambda_{n} G_{n}}(x) - p_{\lambda^{*} G_{*}}(x) & = \sum_{i \in \extraind ( \lambda_{n})} (- \bar{p}_{i}^{*}) f(x| \theta_{i}^{0}) + \sum_{i \in \mathcal{J}} (\lambda_{n} \sum_{j = 1}^{s_{i}} p_{ij}^{n} - \bar{p}_{i}^{*}) f(x| \theta_{i}^{*}) \nonumber \\
	& \hspace{- 6 em} + \lambda_{n} \parenth{ \sum_{j = 1}^{s_{i}} p_{ij}^{n} \Delta \theta_{ij}^{n} }^{ \top} \frac{\partial{f}}{\partial{ \theta}}(x| \theta_{i}^{*}) + \lambda_{n} \parenth{ \sum_{j = 1}^{s_{i}} p_{ij}^{n} \parenth{ \Delta \theta_{ij}^{n}}^{\top} \frac{\partial^2{f}}{\partial{\theta^2}}(x| \theta_{i}^{*}) (\Delta \theta_{ij}^{n})} + R(x), \label{eq:Taylor_pardistinguish_exact_specified_equal}
\end{align}
where $R(x)$ is Taylor remainder such that $R(x) = \smallO \parenth{ \lambda_{n} \sum_{i \in \mathcal{J}} \sum_{j = 1}^{s_{i}} p_{ij}^{n} \enorm{ \Delta \theta_{ij}^{n}}^2}$. Therefore, we have $R(x)/ D(\lambda_{n} G_{n}, \lambda^{*} G_{*}) \to 0$ as $n \to \infty$. 

The expression in equation~\eqref{eq:Taylor_pardistinguish_exact_specified_equal} indicates that we can view $(p_{\lambda_{n} G_{n}}(x) - p_{\lambda^{*} G_{*}}(x))/ D(\lambda_{n} G_{n}, \lambda^{*} G_{*})$ as a linear combination of elements of the forms $f(x| \theta_{i}^{0})$, $f(x| \theta_{j}^{*})$, $\frac{\partial{f}}{\partial{ \theta}}(x| \theta_{j}^{*})$, $\frac{\partial^2{f}}{\partial{\theta^2}}(x| \theta_{j}^{*})$ for $i \in \extraind( \lambda_{n})$ and $j \in \mathcal{J}$. Assume that the coefficients of these terms go to 0 as $n$ approaches infinity. By studying the coefficients of $f(x| \theta_{i}^{0})$ when $i \in \extraind( \lambda_{n})$, we find that
\begin{align*}
	( \sum_{i \in \extraind( \lambda_{n})} (- \bar{p}_{i}^{*}))/ D(\lambda_{n} G_{n}, \lambda^{*} G_{*}) \to 0. 
\end{align*}
Given the above result, as the coefficients of $f(x| \theta_{i}^{*})$ and $\frac{\partial^2{f}}{\partial{\theta^2}}(x| \theta_{i}^{*})$ go to 0 when $i \in \mathcal{J}$, we obtain
\begin{align*}
	\frac{\sumind(\mathcal{I} (\lambda_{n}))\sum_{j = 1}^{s_{i}} p_{ij}^{n} - \bar{p}_{i}^{*}}{D(\lambda_{n} G_{n}, \lambda^{*} G_{*})} & = \frac{ [ \lambda_{n} - ( \sum_{l \in \extraind( \lambda_{n})} \bar{p}_{l}^{*}))] \sum_{j = 1}^{s_{i}} p_{ij}^{n} - \bar{p}_{i}^{*}}{D(\lambda_{n} G_{n}, \lambda^{*} G_{*})} \to 0, \\
	 \frac{\sumind(\mathcal{I} (\lambda_{n}))\sum_{j = 1}^{s_{i}} p_{ij}^{n} \enorm{ \Delta \theta_{ij}^{n}}^2}{D(\lambda_{n} G_{n}, \lambda^{*} G_{*})} & = \frac{ [ \lambda_{n} - ( \sum_{l \in \extraind( \lambda_{n})} \bar{p}_{l}^{*}))] \sum_{j = 1}^{s_{i}} p_{ij}^{n} \enorm{ \Delta \theta_{ij}^{n}}^2}{D(\lambda_{n} G_{n}, \lambda^{*} G_{*})} \to 0
\end{align*}
Putting the above results together, given the expression in equation~\eqref{eq:wasserstein_equivalence_over}, we obtain $1 = D(\lambda_{n} G_{n}, \lambda^{*} G_{*})/ D(\lambda_{n} G_{n}, \lambda^{*} G_{*}) \to 0$ as $n \to \infty$, which is a contradiction. Therefore, not all the coefficients of $f(x| \theta_{i}^{0})$, $f(x| \theta_{j}^{*})$, $\frac{\partial{f}}{\partial{ \theta}}(x| \theta_{j}^{*})$, $\frac{\partial^2{f}}{\partial{\theta^2}}(x| \theta_{j}^{*})$ when $i \in \extraind( \lambda_{n})$ and $j \in \mathcal{J}$. From here, we utilize the Fatou's argument from the previous proofs to obtain the conclusion of claim~\eqref{eq:claim_pardistinguish_point_wise_equal_exactspec} under Case 2.2.
\subsection{Proof of Theorem~\ref{theorem:pardistinguish_point_wise_equal_overspec_weakind}}
\label{subsec:proof:theorem:pardistinguish_point_wise_equal_overspec_weakind}

\begin{customthm}
{\ref{theorem:pardistinguish_point_wise_equal_overspec_weakind}}
Assume that $\sumfun$ takes the form~\eqref{eq:par_distin_h0} and $\bar{k} = k_{0}$. Besides that, $f$ is location-scale Gaussian distribution. Then, for $\tilde{k}:=\max\{k_* - k_0, 1\}$, and for any $\lambda \in [0, 1]$ and $G \in \Ocal_{K, c_{0}}( \Theta)$ for some $K \geq k_*$ and $c_{0} > 0$, there exist positive constants $C_{1}$ and $C_{2}$ depending only on $\lambda^{*}, G_{*}, G_{0}, \Theta$ such that
on $\lambda^{*}, G_{*}, G_{0}, \Theta$ such that
\begin{itemize}
\item[(a)] If $\extraind( \lambda)$ is not ratio-independent, then
\begin{align}
V(p_{\lambda^{*} G_{*}}, p_{\lambda G}) & \geq C_{1} \biggr[ 1_{\{\lambda \in\extraset^{c}\}} \nonumber \\
& + 1_{\{\lambda \in\extraset\}}   W_{\overline{r}(K - \tilde{k} )}^{\overline{r}(K - \tilde{k})}(G, \bar{G}_{*}( \lambda) ) \biggr]. 
\end{align}
\item[(b)] If $\extraind( \lambda)$ is ratio-independent, then
\begin{align}
V(p_{\lambda^{*}, G_{*}}, p_{\lambda, G}) & \geq C_{2} \biggr[ 1_{\{\lambda \in\extraset^{c}\}} \biggr(  \sum_{i \in \extraind( \lambda)} \biggr[ (\lambda^{*} - \lambda) p_{i}^{0} - \lambda^{*}p_{i}^{*} \biggr] \nonumber \\
& + \sumind(\mathcal{I} (\lambda)) W_{\overline{r}(K - \tilde{k})}^{\overline{r}(K - \tilde{k})}(G, \widetilde{G}_{*} ( \lambda)) \biggr) \nonumber \\ 
& + 1_{\{\lambda \in\extraset\}}   W_{\overline{r}(K - \tilde{k} )}^{\overline{r}(K - \tilde{k})}(G, \bar{G}_{*}( \lambda)) \biggr]. 
\end{align}
\end{itemize}
\end{customthm}
The proof of Theorem \ref{theorem:pardistinguish_point_wise_equal_overspec_weakind} is similar to what of Theorem \ref{theorem:pardistinguish_point_wise_equal_exactspec} and with the technical details borrowed from Theorem \ref{theorem:pardistinguish_point_wise_overspec_weak_iden}. Therefore we only highlight the main differences. Denote by $D(\lambda G, \lambda^{*} G_{*}) = 
 1_{\{\lambda \in\extraset^{c}\}}  \biggr(  \sum_{i \in \extraind( \lambda)} \biggr[ (\lambda^{*} - \lambda) p_{i}^{0} - \lambda^{*}p_{i}^{*} \biggr] 
 + \sumind(\mathcal{I} (\lambda)) W_{\overline{r}(K - \tilde{k})}^{\overline{r}(K - \tilde{k})}(G, \widetilde{G}_{*} ( \lambda)) \biggr) 
+ 1_{\{\lambda \in\extraset\}} W_{\overline{r}(K - \tilde{k} )}^{\overline{r}(K - \tilde{k})}(G, \bar{G}_{*}( \lambda))$ when $\extraind(\lambda)$ is ratio-independent or $D(\lambda G, \lambda^{*} G_{*}) = 
 1_{\{\lambda \in\extraset^{c}\}} 
 + 1_{\{\lambda \in\extraset\}} W_{\overline{r}(K - \tilde{k})}^{\overline{r}(K - \tilde{k})}(G, \bar{G}_{*}( \lambda) ) $ when $\extraind( \lambda)$ is not ratio-independent.

In order to prove the theorem, it is sufficient to verify the following inequality:
\begin{align}
	\lim \limits_{\epsilon \to 0} \inf \limits_{\lambda \in [0,1], G \in \Ecal_{k_{*}}( \Theta)}{\left\{\dfrac{V(p_{\lambda G}, p_{\lambda^{*} G_{*}})}{D(\lambda G, \lambda^{*} G_{*})}: \ D(\lambda G, \lambda^{*} G_{*}) \leq \epsilon\right\}} > 0. \label{eq:claim_pardistinguish_point_wise_equal_weakind_ratindep}
\end{align}
\paragraph{Proof of claim~\eqref{eq:claim_pardistinguish_point_wise_equal_weakind_ratindep}:} Assume that the above claim is not true. It implies that there exist sequences $G_{n} = \sum_{i = 1}^{k_{n}} p_{i}^{n} \delta_{\theta_{i}^{n}} \in \Ocal_{K}( \Theta)$ and $\lambda_{n} \in [0, 1]$ such that $D(\lambda_{n} G_{n}, \lambda^{*} G_{*})$ and $V(p_{\lambda_{n} G_{n}}, p_{\lambda^{*} G_{*}})/ D(\lambda_{n} G_{n}, \lambda^{*} G_{*})$ go to 0 as $n$ approaches to infinity. Since $\bar{k} = k_{0}$ and $G_{*}$ admits the form~\eqref{eq:refor_G*}, we find that
\begin{align}
p_{\lambda_{n} G_{n}}(x) - p_{\lambda^{*} G_{*}}(x) = \lambda_{n} \parenth{ \sum_{i = 1}^{k_{n}} p_{i}^{n} f(x| \theta_{i}^{n})} - \sum_{i = 1}^{k_{*}} \bar{p}_{i}^{*} f(x| \theta_{i}^{*}), \label{eq:key_express_pardistinguish_point_wise_equal_weak}
\end{align}
where $\bar{p}_{i}^{*} = \lambda^{*} p_{i}^{*} + (\lambda_{n} - \lambda^{*}) p_{i}^{0}$ when $1 \leq i \leq k_{0}$ and $\bar{p}_{i}^{*} = \lambda^{*} p_{i}^{*}$ otherwise. In addition, $\theta_{i}^{*} = \theta_{i}^{0}$ for $i \in [k_{0}]$. One could argue as in Theorem \ref{theorem:pardistinguish_point_wise_equal_exactspec} to get $(\lambda_n)$ being bounded below.
\paragraph{Case 1:} $\lambda_{n} \in \extraset$ for infinitely many $n$. Without loss of generality, we assume that $\lambda_{n} \in \extraset$ for all $n \geq 1$. Under this case, every limit of $\sum_{i = 1}^{k_{*}} \bar{p}_{i}^{*} f(x| \theta_{i}^{*})$
has a number of atoms ranging from $\tilde{k}$ to $k_*$, which is less than or equal to $K$. So that this is an over-fitted scenario where the number of over-fitted atoms is at most $K - \tilde{k}$. 
In addition, $D(\lambda_{n} G_{n}, \lambda^{*} G_{*}) = W_{\overline{r}(K - \tilde{k})}^{\overline{r}(K - \tilde{k})}(G_{n}, \bar{G}_{*} ( \lambda_{n}))$. We can further rewrite equation~\eqref{eq:key_express_pardistinguish_point_wise_equal_weak} as:
\begin{align*}
	p_{\lambda_{n} G_{n}}(x) - p_{\lambda^{*} G_{*}}(x) = \lambda_{n} ( f(x; G_{n}) - f(x; \bar{G}_{*}( \lambda_{n})).
\end{align*}
Now we can argue similarly to the proof Theorem \ref{theorem:pardistinguish_point_wise_overspec_weak_iden} or Proposition 2.2. in \cite{Ho-Nguyen-Ann-16} to get $V(p_{\lambda_{n} G_{n}}, p_{\lambda^{*} G_{*}})/ D(\lambda_{n} G_{n}, \lambda^{*} G_{*}) \not \to 0$, which combines with the fact that $\lambda_n\not \to 0$ gives us a contradiction. Hence, we reach the conclusion of claim~\eqref{eq:claim_pardistinguish_point_wise_equal_weakind_ratindep}
\paragraph{Case 2:} $\lambda_{n} \not \in \extraset$ for infinitely many $n$. Without loss of generality, we assume that $\lambda_{n} \not \in \extraset$ for all $n \geq 1$. Under this setting, $\extraind( \lambda_{n}) \neq \emptyset$. In addition, for any $i \in \extraind( \lambda_{n})$, $\bar{p}_{i}^{*} < 0$. Given these conditions, we can rewrite equation~\eqref{eq:key_express_pardistinguish_point_wise_equal_exactspec} as follows:
\begin{eqnarray}\label{eq:thm37-rewrite}
	p_{\lambda_{n} G_{n}}(x) - p_{\lambda^{*} G_{*}}(x) = \sum_{i \in \extraind ( \lambda_{n})} (- \bar{p}_{i}^{*}) f(x| \theta_{i}^{0}) + \biggr[ \lambda_{n} \parenth{ \sum_{i = 1}^{k_{n}} p_{i}^{n} f(x| \theta_{i}^{n})} - \sum_{i \in \extraind ( \lambda_{n})^{c}} \bar{p}_{i}^{*} f(x| \theta_{i}^{0}) & \nonumber\\
	& \hspace{- 8 em} - \sum_{i = k_{0} + 1}^{k_{*}} \bar{p}_{i}^{*} f(x| \theta_{i}^{*}) \biggr].
\end{eqnarray}
We have two separate settings with $\extraind( \lambda_{n})$.
\paragraph{Case 2.1:} $\extraind( \lambda_{n})$ is not ratio-independent. This is the same as Case 2.1. of Theorem \ref{theorem:pardistinguish_point_wise_equal_exactspec}. With a similar argument, we can show that $\extraind( \lambda_{n})$ must be ratio-independent, which is a contradiction. Hence, we get claim \eqref{eq:claim_pardistinguish_point_wise_equal_weakind_ratindep} under this case.
\paragraph{Case 2.2:} $\extraind( \lambda_{n})$ is ratio-independent. We can see that the second term of equation \eqref{eq:thm37-rewrite} is in an over-fitted setting with the number of extra components being at most $K - \tilde{k}$. Arguing similar to Case 2.2. of Theorem \ref{theorem:pardistinguish_point_wise_overspec_weak_iden} gives us the conclusion of claim \eqref{eq:claim_pardistinguish_point_wise_equal_weakind_ratindep}.

\end{document}